\documentclass{article}
\usepackage{titletoc}
\usepackage[page,header]{appendix}
\usepackage{tcolorbox}
\usepackage{geometry}
\usepackage{microtype}
\usepackage{graphicx}
\usepackage{subfigure}
\usepackage{booktabs} 
\usepackage{algorithm}
\usepackage{algorithmic}
\usepackage{hyperref}

\usepackage{amsmath}
\usepackage{amssymb}
\usepackage{mathtools}
\usepackage{amsthm}

\theoremstyle{plain}
\newtheorem{theorem}{Theorem}[section]

\newtheorem{lemma}[theorem]{Lemma}
\newtheorem{corollary}[theorem]{Corollary}
\theoremstyle{definition}
\newtheorem{definition}[theorem]{Definition}

\theoremstyle{remark}
\newtheorem{remark}[theorem]{Remark}
\usepackage{comment}
\newtheorem{observation}[theorem]{Observation}
\newtheorem{fact}[theorem]{Fact}

\usepackage[textsize=tiny]{todonotes}
\usepackage{pdfcomment}

\usepackage{moreenum}
\usepackage{mathtools}
\usepackage{url}
\usepackage{bm}
\usepackage{enumitem}
\usepackage{graphicx}
\usepackage{listings}
\usepackage{color}
\usepackage{subcaption}
\usepackage{booktabs} 
\usepackage[mathcal]{eucal}
\usepackage{tikz}
\usepackage{dsfont}
\usepackage{bbm}
\usepackage{cancel}
\usepackage{physics}
\usepackage{authblk}
\usepackage{mathrsfs}
\usepackage[justification=centering]{caption} 
\usepackage{graphicx}
\usepackage{float}
\usepackage[export]{adjustbox} 
\usepackage{caption}
\usepackage{subcaption}
\usepackage{float}
\usepackage[export]{adjustbox} 
\usepackage{caption}
\usepackage{subcaption}

\DeclareMathAlphabet{\mathcal}{OMS}{cmsy}{m}{n}
\SetMathAlphabet{\mathcal}{bold}{OMS}{cmsy}{b}{n}

\newcommand{\T}{\mathbb{T}}

\renewcommand{\O}{\mathbb{O}}
\newcommand{\Z}{\mathbb{Z}}
\newcommand{\R}{\mathbb{R}}
\newcommand{\lt}{\ensuremath <}
\newcommand{\gt}{\ensuremath >}

\renewcommand{\r}{\boldsymbol{\rho}}

\renewcommand{\|}{\Bigg|}

\newcommand{\x}{\mathbf{x}}

\newcommand{\e}{\mathbf{e}}

\newcommand{\sS}{\mathscr{S}}
\newcommand{\sA}{\mathscr{A}} 
\newcommand{\sO}{\mathscr{O}} 

\newcommand{\sF}[1]{\mathscr{F}_{#1}}
\newcommand{\sT}[1]{\mathcal{T}_{#1}} 
\renewcommand{\S}{\mathbf{S}}
\newcommand{\A}{\mathbf{A}}
\renewcommand{\O}{\mathbf{O}} 
\newcommand{\F}{{\boldsymbol{F}}}
\newcommand{\f}{{f}}
\newcommand{\btau}{\boldsymbol{\tau}}

\newcommand{\bT}{\mathbb{T}}
\newcommand{\bO}{\mathbb{O}}

\DeclareMathOperator*{\argmax}{argmax}

\newcommand{\Max}[1]{\underset{ #1 }{\max \ }}

\newcommand{\Argmax}[1]{\underset{ #1 }{\argmax\ }}

\newcommand{\sgn}{\text{sgn}}
\newcommand{\Pm}[1]{\mathbb{P}_{\mathcal{P}}\left\{ #1 \right\}}
\renewcommand{\Pr}[1]{\mathbb{P}_{\mathcal{P}}^\pi\left\{ #1 \right\}} 
\newcommand{\diag}[1]{\text{diag}\left(#1\right)}

\newcommand{\muu}[1]{\vec{\mu}_{#1}}

\newcommand{\sign}[1]{\sgn{#1}}

\definecolor{ForestGreen}{RGB}{34,139,34}

\newcommand{\algcomshort}[1]{\textcolor{blue}{{\hfill $\triangleright$ { #1}}}}
\newcommand{\algcom}[1]{\textcolor{blue}{ #1}}

\usepackage[marginal]{footmisc}

\title{
Provably Efficient 
Partially Observable 
Risk-Sensitive
Reinforcement Learning
with Hindsight Observation
}

\author{
    Tonghe Zhang$^{1, *}$,\quad
    Yu Chen$^{2, *}$,\quad
    \text{Longbo Huang}$^{2, \dagger}$
    \\
    $^{1}$ Department of Electronic Engineering, Tsinghua University\\
    $^{2}$ Institute for Interdisciplinary Information Sciences, Tsinghua University\\
    \texttt{\{zhang-th21, chenyu23\}@mails.tsinghua.edu.cn},\\
    \texttt{longbohuang@tsinghua.edu.cn}
}

\date{}

\begin{document}

\maketitle

\renewcommand{\thefootnote}{\fnsymbol{footnote}}
\footnotetext[1]{These authors contributed equally.}
\footnotetext[2]{Corresponding author.}
\renewcommand*{\thefootnote}{\arabic{footnote}}

\begin{abstract}
This work pioneers regret analysis of risk-sensitive reinforcement learning in partially observable environments with hindsight observation, addressing a gap in theoretical exploration. We introduce a novel formulation that integrates hindsight observations into a Partially Observable Markov Decision Process (POMDP) framework, where the goal is to optimize accumulated reward under the entropic risk measure. We develop the first provably efficient RL algorithm tailored for this setting. We also prove by rigorous analysis that our algorithm achieves polynomial regret $\tilde{O}\left(\frac{e^{|{\gamma}|H}-1}{|{\gamma}|H}H^2\sqrt{KHS^2OA}\right)$, which outperforms or matches existing upper bounds when the model degenerates to risk-neutral or fully observable settings. We adopt the method of change-of-measure and develop a novel analytical tool of beta vectors to streamline mathematical derivations. These techniques are of particular interest to the theoretical study of reinforcement learning.
\end{abstract}

\section{Introduction}\label{section_intro_main}
Reinforcement learning (RL) is a sequential decision-making problem in which an agent learns to maximize accumulated rewards through interactions with an unknown environment \cite{sutton2018reinforcement}.
In many practical scenarios such as derivative hedging \cite{cao2020DRL_derivative_hedging} and actuarial science \cite{richman2021ai_actuarial}, decision-makers have to consider the associated risks, leading to the study of risk-sensitive RL \cite{Fei20}.

It is also a common practice to make costly decisions based on unreliable or incomplete information, such as in autonomous driving \cite{2019riskpomdp_autonomous_drive},
stock market prediction \cite{2022DRL_stock}, and cybersecurity \cite{vassilev2022risk}.
The Partially Observable Markov Decision Process (POMDP)\cite{MonahanSurvey1982} is widely employed as the mathematical framework for these problems. 

Empirical studies have been conducted on risk-sensitive POMDPs to address the planning or learning problem in various application scenarios \cite{2021autodrive-POMDP-risk,choi2021risk,qiu2021rmix, shen2023riskq}. However, these studies often lack a performance guarantee. On the other hand, theoretical studies primarily focus on demonstrating the existence of an optimal policy \cite{di1999risk} or addressing the planning problem with full knowledge of the transition probabilities \cite{Baras&Elliott94,baauerle2017partially}, but have yet to develop the sample complexity analysis. Consequently, an open theoretical question remains from prior research:

\textbf{Can we devise a sample-efficient and theoretically grounded risk-sensitive RL algorithm {in partially observable environments?}}

Obtaining a conclusive answer to this question is challenging due to several technical obstacles. Firstly, the complex structure of the POMDP becomes even more intricate when incorporating a non-linear risk measure, raising doubts about whether mathematical analysis alone can effectively simplify the problem. Secondly, partial observations pose challenges in learning the model with limited sample complexity and designing the exploration bonus with incomplete information.

In this work, we devise a novel algorithm that addresses these problems in a POMDP model equipped with hindsight observations. Due to the fact that reinforcement learning in general POMDPs is intractable \cite{Tsitsiklis1987,Jin2020}, we involve hindsight observations\cite{Lee2023hindsight}
in the learning protocol(an introduction is deferred to Section \ref{section_formulate_main}), so that efficient learning becomes possible. We also excavate the dynamic programming structure implicit in the risk-sensitive POMDP model and subsequently derive a fresh set of Bellman equations that matched our problem setting. Moreover, we uncover a simple representation of the value functions with explicit analytical forms, leading to the creation of a new exploration bonus that exploits the partial information gleaned from the environment while considering the agent's risk sensitivity. Our algorithm efficiently estimates the accumulated risks in all the possible hidden states and selects the optimal actions.

In addition, we provide theoretical guarantee for the algorithm design: the upper bound of the regret presented in Eq.~\eqref{regret_full_expression} is not only polynomial in all the parameters but also explains how the risk measure, partial observations and empirical estimator affects the learning efficiency of each component of the POMDP model. 
Moreover, when the model degenerates to risk-neutral or fully observable settings, our result improves or matches existing upper bounds and nearly reaches the lower bounds in these scenarios.

The contributions of this work are summarized as follows:\label{section_contribution_main}
\begin{itemize}
\item \emph{Formulation.}
We propose a novel theoretical formulation for risk-sensitive reinforcement learning in a partially observable environment with hindsight observations. We adopt the entropic risk measure in the framework, which accommodates general underlying POMDP models with non-stationary decision process. We also generalize certain results to arbitrary utility risk measures.
\item
\emph{Algorithm.}
We develop a new algorithm that incorporates a belief propagation process before the value iteration begins. Our design allows the agent to preemptively estimate the accumulated risks within the hidden states, before she optimizes the value functions through a greatly simplified Bellman equation. We also introduce a new bonus function that exploits partial information to encourage risk-sensitive exploration.
\item \emph{Regret.}
We provide the first regret analysis of the problem. By disregarding lower-order terms, our algorithm successfully attains the regret
$
{O}
\left(
\frac{
e^{\abs{\gamma}H}-1}{\abs{\gamma }H}
\cdot 
H^{2}\sqrt{KS^2OA }
\cdot \sqrt{H\ln {KHSOA}/{\delta}}
\right),
$
which demonstrates the risk awareness of the agent and the history-dependency. 
When the model degenerates to risk-neutral or fully observable settings, our regret improves or matches existing upper bounds and nearly reaches the lower bound of risk-sensitive RL.

\item
\emph{Techniques.}
We introduce a novel {analytical tool} called the beta vector, which plays a pivotal role in designing our bonus function, resulting in simplified value iteration and regret analysis.
We also adopt the change-of-measure technique, which decouples the state and observations to streamline analytical derivations. 
\end{itemize}

\section{Related Work}\label{section_relate_main}
Due to space limits, we only discuss the most relevant works below. A thorough overview is provided in Appendix~\ref{section_relate}.

\textbf{Risk-sensitive RL.} 
Our analysis draws inspiration from studies about risk-sensitive RL. For instance, \cite{Fei21improve} introduces a new bonus to improve the regret bound of learning an MDP using entropic risk. Additionally, \cite{Lipschitz23} uses the concept of Lipschitz continuity to linearize various risk measures. These works contribute to our technical toolkit.

\textbf{POMDP.} It is well-known that planning or learning a general POMDP is intractable \cite{Tsitsiklis1987,Jin2020}. Consequently, a body of research in partially observable RL restricts their attention to sub-classes of POMDP with structural assumptions\cite{Liu2022,Sun2022PSR}. These studies assume the emission process reveals enough information for the agent to decode the hidden states, such as \cite{Liu2022,Yang2022POMDPfa,Golowich2022-b}. However, their regret bounds become vacuous when the assumptions are not satisfied. To make our formulation more pragmatic, we do not follow this direction. Our algorithm accommodates POMDPs with general underlying models.

\paragraph{Risk-Sensitive POMDP.}
Our theoretical framework builds on prior research such as \cite{Baras&Elliott94,baauerle2017partially}.
However, our study significantly diverges from the predecessors
because we do not presuppose the transition and emission matrices are time-invariant and fully known. Our agent learns a non-stationary model through online interactions, focusing on the exploration-exploitation trade-off typical of reinforcement learning.

\paragraph{Notations} In this study, we denote the set $\{1, \ldots, n\}$ by $[n]$ for any positive integer $n$, and we use
$\Delta(\mathcal{X})$ to represent the probability simplex over the finite space $\mathcal{X}$.
We use the notations $\gtrsim$ and $\tilde{O}$ to hide constants and logarithmic terms in the expressions, respectively.
$\gamma^-$ represents $\min\{\gamma, 0\}$ and $\gamma^+$ stands for $\max\{\gamma, 0\}$.
The symbol $\iota$ is shorthand for $\ln \left(\frac{KHSOA}{\delta}\right)$, and $N \vee 1$ denotes $\max\{N, 1\}$. 
Random variables are presented in bold, while their realizations are in roman.
For random vectors $\boldsymbol{v_A}$ and $\boldsymbol{v_B}$, $\boldsymbol{v_A} \setminus \boldsymbol{v_B}$ signifies the subvector of $\boldsymbol{v_A}$ where the components from $\boldsymbol{v_B}$ are excluded.

\section{Problem Formulation}\label{section_formulate_main}

\paragraph{The POMDP Model}\label{section_pomdp_formulate}
We consider a tabular, episodic and finite horizon POMDP with non-stationary transition matrices\cite{MonahanSurvey1982,Liu2022}. The underlying model of the POMDP can be specified by a tuple $\mathcal{P}=(\mathscr{S,O,A}; \mu_1, \mathbb{T}, \mathbb{O};K, H, r)$, where $\mathscr{S}$, $\mathscr{O}$ and $\mathscr{A}$ are the spaces of the hidden states, observations, and actions with cardinality $S, O$, and $A$ respectively. 
The agent plays with the model in $K$ episodes and each episode contains $H$ steps. $\mu_1\in \Delta(\mathscr{S})$ is the {prior distribution of the hidden states} which can be represented as an $S$-dimensional vector $\vec{\mu}_1$.
$\mathbb{T}=\{\mathbb{T}_{h,a} \in \R^{S\times S} \mid {(h,a) \in[H] \times \mathscr{A}}\}$ and $\mathbb{O}=\{\mathbb{O}_h \in \R^{O\times S }\mid h \in[H] \}$ are the transition and emission matrices respectively. 
If the environment is in state $s$ at step $h$, then $\T_{h,a}(\cdot|s)$ represents the distribution of the next hidden state when the agent takes action $a$, while  $\mathbb{O}_h(\cdot|s) $ is the distribution of the observations generated by the current hidden state. $r=\{\ r_h(\cdot,\cdot):\  \mathscr{S} \times \mathscr{A} \to [0,1] \mid   h\in [H]\ \}$ is the collection of reward functions that measure the performance of actions in each hidden state.

In each episode, the initial state $\boldsymbol{S}_1$ is sampled from  $\mu_1(\cdot)$. For all steps $h \in [H]$, the agent decides an action $\boldsymbol{A_h}$ based on previous observations and receives a reward $r_h(s_h,a_h)$. The environment then transits to a new hidden state $\boldsymbol{S}_{h+1}\sim \mathbb{T}_{h,a_h}(\cdot|s_h)$ and emits an observation $\boldsymbol{O}_{h+1} \sim \mathbb{O}_{h+1}(\cdot|s_{h+1})$, after which the a new action will be taken. In a POMDP, the agent never detects the hidden states, so she makes decisions according to the observable history
$\boldsymbol{F}_h:=(\boldsymbol{A_1,O_2,\cdots, O_{h-1},A_{h-1},O_h})\in \sF{h}$, from which her policy $\pi_h(\cdot)$ maps to the action space.
\footnote{We use the notation $\boldsymbol{F}$ because the histories constitute a filter process.}
The POMDP evolves until the last state $\boldsymbol{S}_{H+1}$, before a new episode begins.

\paragraph{Reinforcement Learning with Hindsight Observation}\label{condition_hind_sight_main}
In RL, the agent plans under empirical models $\{ \mathcal{\widehat{P}}^k \}_{k=1}^K$ learned from data samples of the history. The corresponding empirical distributions will be denoted as $\widehat{\mu}_1^k, \mathbb{\widehat{T}}^k$, and $\mathbb{\widehat{O}}^k$. 
In this work, we incorporate \emph{hindsight observation} \cite{Lee2023hindsight} in the interaction protocol of the POMDP: in test time, the agent is allowed to review the hidden states that occurred in the last $H$ steps at the end of each episode.
The concept of hindsight observability in RL is proposed by \cite{Lee2023hindsight} and echoed by \cite{sinclair2023hindsight,shi2023theoretical,guo2023sample}.
According to \cite{Lee2023hindsight}, hindsight observations are common in real-world applications of the POMDP, such as sim-to-real robotics\cite{pinto2017asymmetric,chen2020learning}, data center scheduling\cite{sinclair2023hindsight} and online imitation learning\cite{ross2011reduction}.\footnote{{Please refer to Section 3.2 and 5 in \cite{Lee2023hindsight} for more examples.}}
This setting also makes efficient learning possible.

\paragraph{Reinforcement Learning using Entropic Risk Measure}
In risk-sensitive RL, the agent seeks an \emph{optimal policy} $\pi^\star$ that maximizes the following \emph{optimization objective}, which is the entropic risk measure of accumulated reward:
\begin{equation}\label{optim_obj_main}
   \begin{aligned}
   J(\pi;\mathcal{P}, \gamma) :=
   \frac{1}{\gamma} \ln 
   \mathbb{E}_{\mathcal{P}}^{\pi}
   \left[e^{\gamma \sum_{t=1}^H r_t(\boldsymbol{S_t,A_t})}\right]
   \end{aligned}
\end{equation}
where $\mathcal{P}$ is the POMDP model and $\gamma \ne 0$ is 
the given parameter of risk-sensitivity.
\begin{remark}
    The agent is risk-seeking when $\gamma \gt 0$ while risk-averse when $\gamma \lt 0$ \cite{baauerle_Rieder14}. Other studies such as \cite{Fei20, Lipschitz23} also 
    adopt the entropic risk in the objective function.
\end{remark}

\emph{Learning Objective} 
In this study, we aim to devise an algorithm whose output policies $\{\widehat{\pi}^k\}_{k=1}^K$ minimizes the difference with $\pi^\star$, which will be measured by the \emph{``regret''} defined below:
$$\operatorname{Regret}(K;\mathcal{P},\gamma)
   :=
   \sum_{k=1}^K
   J(\pi^\star;\mathcal{P},\gamma)
   -
   J(\widehat{\pi}^k;\mathcal{P},\gamma)
$$
We are also concerned with the \emph{sample complexity} of the algorithm, which is the smallest episode number $K$ that ensures $\frac{1}{K}\sum_{k=1}^K J(\widehat{\pi}^k,\mathcal{P}, \gamma) \geq J(\pi^\star,\mathcal{P}, \gamma) -\epsilon$ with probability at least $1-\delta$.

\section{Value Function and the Bellman Equations}\label{section_Bellman_eqs_main}
In the much more complex setting of POMDP, we need to design special value functions to simplify the intensive computation caused by the history dependency in the policies.
However, as is shown in Remark~\ref{remark_alpha-beta-comparison} in the appendix, a naive adaptation of the value functions in the POMDP literature \cite{MonahanSurvey1982} fails to capture the nonlinear structure of a risk-POMDP. 
In this work, we introduce a new definition for value functions based on the studies of \cite{baauerle2017partially,Baras&Elliott94}, which not only simplifies the analysis but also helps us derive a new set of Bellman equations tailored to our problem. We will use these concepts in our algorithm presented in Section~\ref{section_algo_design_main}.
\paragraph{Change of Measure}\label{change_of_measure_main}
To simplify the mathematical analysis, we adopt the technique called ``change of measure'' and investigate the risk-POMDP problem in a simpler model $\mathcal{P}^\prime$ (obtained via transformation), in which the observations $\boldsymbol{O}_t$ and hidden states $\boldsymbol{S_t}$ are independent.

The technique of change-of-measure originates from stochastic calculus \cite{oksendal2013StochasticCalculus} and is vastly adopted in derivative pricing \cite{baxter1996finCalculus} and filtering theory \cite{jazwinski2007StochasticFilteringTheory}. In the study of POMDP, this method is used to decouple the transition and emission processes\cite{Baras97,structural97}, which will also facilitate the statistical complexity analysis using hindsight observations.

In this work, we refer to $\mathcal{P}^\prime$ as the ``reference model'', whose rigorous definition is presented in Appendix~\ref{section_change_of_measure}.
The following analysis comes from the studies of \cite{Baras&Elliott94,Cavazos2005}.
The relationship between $\mathcal{P}^\prime$ and the original POMDP can be 
described by their Radon-Nikodym derivative:
\footnote{Please refer to Appendix~\ref{rn_derivative} for a formal definition.}
\begin{equation*}
\begin{aligned}
\boldsymbol{D_h}
:=\frac{d\mathbb{P}_{\mathcal{P}}^\pi}{d\mathbb{P}_{\mathcal{P}^\prime}^\pi}
=\prod_{t=2}^h \frac{\bO_t(\O_t|\S_t)}{\bO^\prime(\O_t)}
\end{aligned}
\end{equation*}
In the reference model, the observations are irrelevant to the hidden states. As a result, they are separate from the underlying process and independent of the history. 
To further simplify the model, we can specify $\mathbb{O}^\prime(\cdot)$ as the uniform distribution, so that $\boldsymbol{O}_h\overset{i.i.d.}{\sim} \text{Unif}\sO$ in the model $\mathcal{P}^\prime$.\footnote{Other distributions are also suitable for $\mathbb{O}^\prime$. Please refer to Appendix~\ref{rn_derivative}.}
Consequently, we can significantly simplify the posterior distribution of $\mathbf{O}_t$ in model $\mathcal{P}^\prime$, which not only reduces space consumption dramatically but also reduces a series of analytical computations. After planning in $\mathcal{P}^\prime$, we will use $\boldsymbol{D_h}$ as a bridge to convert the results back to the environment $\mathcal{P}$, according to the following rule derived from the Lebesgue-Radon-Nikodyn theorem
\cite{rudin1987real_analysis}: 
$$
\mathbb{E}_{\mathcal{P}}^\pi = \mathbb{E}_{\mathcal{P}^\prime}^\pi \boldsymbol{D_h}
$$
We define our value functions in the model $\mathcal{P}^\prime$.
\begin{definition}(Value functions)\label{def_value_short}
\begin{equation}\label{eq_value_short}
\begin{aligned}
\mathsf{V}_{h}^\pi(\boldsymbol{F}_h):=&
\frac{1}{\gamma}\ln\mathbb{E}_{\mathcal{P}^\prime}^\pi\left[
e^{\gamma \sum_{t=1}^H r_t(\boldsymbol{S}_t, \boldsymbol{A}_t)}\bigg| \boldsymbol{F}_h\right]
\\
\mathsf{Q}_h^\pi(\boldsymbol{F}_h,\boldsymbol{A}_h):=&
\frac{1}{\gamma}\ln\mathbb{E}_{\mathcal{P}^\prime}^\pi\left[e^{\gamma \sum_{t=1}^H r_t(\boldsymbol{S}_t, \boldsymbol{A}_t) }\bigg| \boldsymbol{F}_h, \boldsymbol{A}_h\right]
\end{aligned}
\end{equation}
\end{definition}
In a POMDP, the policy passes the history-dependency down to the value functions, whose variables contain $\boldsymbol{F_h}$ but not $\boldsymbol{S_h}$, since the latter is not even observable. We can also derive the Bellman equation for our problem:
\begin{definition}(Bellman equations)\label{Bell_eqs_short}
    \begin{equation}
        \begin{aligned}\label{bellman_Eq_short}
        &
        \left\{
            \begin{aligned}
                \mathsf{V}_{H+1}^\pi(f_{H+1})=&
                \frac{1}{\gamma}\ln\mathbb{E}_{\mathcal{P}^\prime}^\pi\left[
                e^{\gamma \sum_{t=1}^H r_t(\boldsymbol{S}_t, \boldsymbol{A}_t)}\bigg| f_{H+1}\right]
                \\
                \mathsf{Q}_h^\pi(f_h,a_h)
                =&
                \frac{1}{\gamma}\ln
                \mathbb{E}
                _{\O_{h+1}\sim \text{Unif}\sO 
                }
                \left[
                e^{\gamma \mathsf{V}_{h+1}^\pi
                (f_h,a_h,\mathbf{O}_{h+1})}
                \right]
                \\
                \mathsf{V}_h^{\pi}(f_h)
                =&
                \mathbb{E}_{\mathbf{A_h}\sim \pi_h(\cdot|f_h)}
                {
                \mathsf{Q}_{h}^\pi(f_h,\mathbf{A}_h)}
            \end{aligned}
        \right.\end{aligned}
    \end{equation}
\end{definition}
The Bellman equations in Eq.~\eqref{bellman_Eq_short} are novel in the literature. Specifically, the value function $\mathsf{V}_{H+1}^{\pi}$ is not zero, which necessitates a belief propagation in the algorithm to initiate the value iteration. Moreover, the computation of Q function is greatly simplified since it is defined in the reference model. For the derivation of Eq.~\eqref{bellman_Eq_short} and a generalization to arbitrary utility functions, please refer to Appendix\ref{section_Value functions, Q-functions and Bellman equations}.
\section{Algorithm Design}\label{section_algo_design_main}
\begin{algorithm}[tb]
   \caption{\text{Beta Vector Value Iteration (BVVI)}}
    \label{alg:BVVI_short}
    \begin{algorithmic}[1]
        \STATE {\bfseries Input}
        K, H, risk level $\gamma \ne 0$, confidence $\delta\in(0,1)$
        \STATE {\bfseries Initialize}
        $\widehat{\mu}^1_1(\cdot), {\widehat{\bT}_{h,a}^1}(\cdot|s)  \gets  \text{Unif}({\sS})$
        \STATE {\bfseries Initialize}
        $\hphantom{\widehat{\mu}^1_1(\cdot), {\widehat{\bT}_{h,a}^1}(\cdot|s)}
         \mathllap{\widehat{\bO}_{h}^1(\cdot|s)}    \gets   \text{Unif}({\sO})$
        \FOR {$k = 1:K$}
                \STATE $\widehat{\sigma}^k_1   \gets  \widehat{\mu}^k_1$
                \algcomshort{Belief propagation}
            \FOR{$h =1:H$}
                \STATE
                Update risk belief $\widehat{\sigma}_{h+1, f_{h+1}}^k$ by Eq.~\eqref{belief_update_short}
                \label{alg_prop}
                \STATE Update bonus function $b_h^k(s_h,a_h)$ by Eq.~\eqref{bonus_short}
                \label{alg_update_bonus}
            \ENDFOR
            \STATE $\widehat{\beta}^k_{H+1,f_{H+1}}    \gets     \vec{1}_{S}$
            \algcomshort{ Value iteration}
            \FOR{$h = H:1$}
                \STATE
                    $\widehat{\mathsf{Q}}_h^k(f_h,a_h)    \gets   \frac{1}{\gamma} \ln\mathbb{E}_{o_{h+1}\sim \text{Unif}\sO}
                    \langle 
                    \widehat{\sigma}^k_{h+1}
                    , 
                    \widehat{\beta}^{k, \widehat{\pi}^k}_{h+1} 
                    \rangle $
                    \label{alg_uniform_short}
                \STATE
                    $\hphantom{\widehat{\mathsf{Q}}_h^k(f_h,a_h)} \mathllap{\mathsf{\widehat{V}}^k_{h}(f_h)}    \gets   \Max{a\in \sA} \ \mathsf{\widehat{Q}}^k_h(f_h,a)$
                \STATE
                    $\hphantom{\widehat{\mathsf{Q}}_h^k(f_h,a_h)} \mathllap{\widehat{\pi}^k_h(f_h)}  \gets    \Argmax{a\in \sA} \mathsf{\widehat{Q}}^k_h(f_h,a) $
                \STATE
                    Update beta vector $\widehat{\beta}^{k,\widehat{\pi}^k}_{h,f_h}$ by Eq.~\eqref{update_empirical_beta_short}
                    \label{alg_update_beta_short}
                 \STATE
                    Restrict $\widehat{\beta}_{h,f_h}^{k,\widehat{\pi}^k}$ in $\left[e^{\gamma^-(H-h+1)}, e^{\gamma^+(H-h+1)}\right]$
                    \label{alg_clip}
            \ENDFOR
            \algcomshort{Statistical Learning}
            \STATE Play with $\mathcal{P}$
                   under $\widehat{\pi}^k$
                   while gleaning $\{\widehat{o}^k_t, \widehat{a}^k_t\}$
                    \label{alg_play_real_short}
            \STATE Review hidden states $\{ \widehat{s}_t^k\}_{t=1}^{H+1}$ \emph{in hindsight}.
                    \label{alg_hindsight_short}
            \STATE ${\widehat{N}_{h}^{k+1}}(s,a) \gets \sum_{\kappa=1}^k \mathds{1} \left\{{\hat s_h^\kappa}=s, {\hat a_h^\kappa}=a\right\}$
            \STATE $\hphantom{{\widehat{N}_{h}^{k+1}}(s,a)} \mathllap{{\widehat{N}_{h}^{k+1}(s)}}  \gets  \sum_{\kappa=1}^k \mathds{1} \left\{{\hat s_h^\kappa}=s\right\}$
            \STATE Update
            $\widehat{\mu}_1^k(\cdot), 
            \widehat{\mathbb{T}}_{h,a}^{k+1}(\cdot|s), 
            \widehat{\mathbb{O}}_{h}^{k+1}(\cdot|s)$ by Eq.~\eqref{explain_approx_update}
            \label{alg_approx_update}
        \ENDFOR
    \end{algorithmic}
\end{algorithm}
In what follows, we present a sample-efficient algorithm that solves the partially observable risk-sensitive RL problem.
We name the algorithm Beta Vector Value Iteration (BVVI), which is a UCB algorithm \cite{lattimore2020bandit} that encourages explorations by a bonus function:
\begin{equation}\label{bonus_short}
\mathsf{b}_h^k(s_h,a_h)=
\abs{e^{\gamma(H-h+1)}-1}\cdot 
\min\{\ 1, \mathsf{t}_{h}^k(s_h,a_h)
+\sum_{s^\prime}\widehat{\mathbb{T}}^k_{h,a_h}(s^\prime|s_h)
{\mathsf{o}_{h+1}^k(s^\prime)}\}
\end{equation}
where $\mathsf{t}_{h}^k(s_h,a_h)$ and $\mathsf{o}_{h+1}^k(s^\prime)$ are defined as
\begin{equation}\label{def_residue_short}
\begin{aligned}
\mathsf{t}_h^k(s_h,a_h):=&\min \left\{1, 3\sqrt{\frac{S H\cdot \iota}{\widehat{N}_h^k(s_h,a_h) \vee 1}} \right\}\\
{\mathsf{o}_{h+1}^k(s_{h+1})} :=&\min\left\{1,3\sqrt{\frac{O H\cdot \iota}{ \widehat{N}_{h+1}^{k+1}(s_{h+1}) \vee 1}}\right\}
\end{aligned}
\end{equation}
In BVVI, the agent updates the empirical model using hindsight observations:
\begin{equation}
    \begin{aligned}
        \label{explain_approx_update}
            \widehat{\mu}_1^{k+1}(s)
            \gets &
            \sum_{\kappa =1}^k \frac{\mathds{1}\{\widehat{s}^{\kappa}_1=s\}}{k}
            \\
            \widehat{\mathbb{T}}_{h,a}^{k+1}(s'|s)
            \gets &
            \sum_{\kappa=1}^k
            \frac{\mathds{1} \{\hat{s}_{h+1}^\kappa=s'\ , \ {\hat s_h^\kappa} =s\ , \ {\hat a_h^\kappa} =a\}}{{\widehat{N}_h^{k+1}}(s,a) \vee 1}   
            \\
            \widehat{\mathbb{O}}_{h}^{k+1}(o|s)
            \gets &
            \sum_{\kappa=1}^k \frac{\mathds{1} \{{\hat o_{h}^\kappa}=o\ , \ {\hat s_h^\kappa} =s\}}{\widehat{N}_h^{k+1}(s) \vee 1}
    \end{aligned}
\end{equation}

In the planning phase, the agent preemptively estimates the accumulated risks across the hidden states in a belief propagation process (Line \ref{alg_prop}), before she plans for the optimal policy in the empirical model, according to the Bellman equations provided in Eq.~\eqref{bellman_Eq_short}. To simplify computations, we express the value function with the help of a set of new variables $\widehat{\sigma}_{h,f_h}^k$ and $\widehat{\beta}_{h,f_h}^{k,\widehat{\pi}^k}$, which will be introduced later in Section~\ref{section_risk_belief_beta_vector}. BVVI provides exact solution to the planning problem so assignments to the functions traverse their domains.

\section{Main Results}\label{section_main_results_main}
In this section, we present the theoretical guarantee for algorithm \ref{alg:BVVI_short}. 
The proofs will be overviewed in Section~\ref{section_risk_belief_beta_vector}.
\begin{theorem}(Regret)\label{theorem_regret_short}
With probability at least 
$1-4\delta$, 
algorithm \ref{alg:BVVI_short} achieves the following regret upper bound: 
\begin{equation}\begin{aligned}\label{regret_short}
\mathcal{O}
\bigg(
\underbrace{\frac{e^{\abs{\gamma } H}-1}{\abs{\gamma}H}
}_{\text{risk awareness
}}
\underbrace{H^2\sqrt{KS^2AO}
}_{\text{Statistical error
}}
\underbrace{
\sqrt{H\ln\frac{KHSOA}{\delta}}
}_{\text{History-dependency
}}
\bigg)
\end{aligned}\end{equation}
\end{theorem}
The detailed expression of the regret bound is presented in Eq.~\eqref{regret_full_expression}. Using the online-to-PAC conversion argument\cite{jin2018qlearning} 
we also obtain the sample complexity of BVVI: 
\begin{corollary}(Sample complexity)
For algorithm \ref{alg:BVVI_short}, the uniform mixture of its output policies ensures
$$
\mathbb{P}
\bigg(
V^{\star}-\sum_{k=1}^K V^{\widehat{\pi}^k} \lt \epsilon
\bigg)
\geq 
1-\delta
$$
when episode number $K$ satisfies
$$
K \gtrsim \frac{1}{\epsilon^2 \delta^2}\left(\frac{e^{\abs{\gamma }H}-1}{\abs{\gamma} H}\right)^2
\cdot H^5 S^2 O A \cdot \ln \left(\frac{KHSOA}{\delta}\right)
$$
\end{corollary}
\begin{remark}
The first factor in Eq.~\eqref{regret_short} reveals the risk awareness of the agent, which disappears in the risk-neutral scenario ($\gamma \to 0$).
The second term demonstrates the statistical error brought by our UCB-style algorithm. 
The last factor represents how the inherent history-dependency of the POMDP discourages confidence in exploration, where the additional $\sqrt{H}$ comes from the uncertainty in the face of the large history space: $\sqrt{\ln \frac{\abs{\sF{h}}}{\delta}}\leq \sqrt{\ln \frac{\abs{OA}^H}{\delta}}\leq \sqrt{H \iota}$. 
\end{remark}
The mathematical analysis behind the main results will be presented in the following two sections.

\section{Risk Belief and Beta Vector}\label{section_risk_belief_beta_vector}
In what follows we lay the foundation for comprehending Theorem \ref{theorem_regret_short} by the introduction of several essential concepts and lemmas. 

A crucial element of the analysis is the risk belief vector, which plays a pivotal role in deriving novel value functions and Bellman equations presented in Eq.~\eqref{def_value_short} and \eqref{Bell_eqs_short}.
Moreover, we will put forward the concept of beta vector, which is the cornerstone for the bonus design in Eq.~\eqref{bonus_short} and the regret analysis presented in Section~\ref{section_regret_ana_main}.

\subsection{Risk Belief and the Bellman Equations}\label{section_risk_belief_beta_main}
To capture the structural properties characteristic of our problem, we construct a \emph{risk belief} that estimates the accumulated risk based on historical observations, which also excavates the dynamic programming structure implicit in the risk-sentive POMDP model.
\begin{definition}
\raggedright(Risk Belief \cite{Baras&Elliott94})
\label{def_risk_belief_short}
\\
For all $h \in [H+1], \f_h\in \sF{h},\ s_h\in \sS,$  the risk beliefs are random vectors in $\R^{S}$, in which $\vec{\sigma}_{1}$ is defined as $\vec{\mu}_1$ 
and
    \begin{equation*}
        \begin{aligned}
        [\vec{\sigma}_{h,\mathbf{F}_h}]_{s_h}:=
        \mathbb{E}_{\mathcal{P}^\prime}^\pi 
        \left[
        \mathbf{D}_h
        \mathds{1} \{ \S_h=s_h \}
        e^{\gamma \ \sum_{t=1}^{h-1} r_t(\S_t,\A_t)}
        \bigg|
        \mathbf{F}_h
        \right]
        \end{aligned}
    \end{equation*}
\end{definition}
We can view the vector $\vec{\sigma}_{h,\mathbf{F}_h}$ 
as a list of cumulative risks estimated from the observable history, 
when the agent is in each of the hidden states.

Using the risk beliefs, we can express the optimization objective in Eq.~\eqref{optim_obj_main} with a simple form:
\begin{equation}
\begin{aligned}
\label{represent_obj}
J(\pi;\mathcal{P}, \gamma)
=\frac{1}{\gamma} \ln \mathbb{E}_{\mathcal{P}^\prime}^\pi
\left[
\langle
{\vec{\sigma}_{H+1,\mathbf{F}_{H+1}}}, 
\vec{1}_{S}
\rangle
\right]
\end{aligned}
\end{equation}
With some derivation (cf. Theorem \ref{theorem_evolution_of_risk_sensitive_belief}), we can obtain the evolution law of the stochastic process $\vec{\sigma}_{h,\boldsymbol{F}_h}$:\footnote{For the evolution law of the empirical beliefs $\widehat{\sigma}_{h+1}^k$, 
we replace the transition and emission matrices with their empirical approximations.}
\begin{equation}\label{belief_update_short}
    [\vec{\sigma}_{h+1,\boldsymbol{F}_{h+1}}]_{s_{h+1}} 
    = 
    \sum_{s_{h}}
    \bT_{h,\boldsymbol{A_{h}}}(s_{h+1}|s_h)\cdot
    e^{\gamma r_h(s_h,\boldsymbol{A_{h}})}
    \frac{\bO_{h+1}(\boldsymbol{O_{h+1}}|s_{h+1})}{\bO_{h+1}^\prime (\boldsymbol{O_{h+1}})}
    \left[
    \vec{\sigma}_{h,\boldsymbol{F_h}}
    \right]_{s_h}
\end{equation}
which has a matrix representation in the tabular case: 
$\vec{\sigma}_{h+1,f_{h+1}}=\mathsf{U}_{a_h,o_{h+1}} \vec{\sigma}_{h,f_h}$.
The inner product in Eq.~\eqref{represent_obj} illuminates the existence of linear structure in $J(\pi;\mathcal{P}, \gamma)$. 
Next, we will introduce another stochastic process, called the conjugate beliefs, which works in a concerted effort with the risk beliefs to expose the linearity within the optimization objective.
\begin{definition}(Conjugate beliefs \cite{Baras&Elliott94}) 
\label{def_conjugate_beliefs_short}
Let $\vec{\nu}_{H+1}:=\vec{1}_{S}$. 
For all $h\in [H]$ and $\boldsymbol{\bar{F}}_{h}=\boldsymbol{F_{H+1}}\setminus \boldsymbol{F_h}$, 
the conjugate beliefs is a series of random vectors in $\R^S$ which is defined iteratively: $\vec{\nu}_{h,\boldsymbol{\bar{F}}_{h}}=\mathsf{U}_{\boldsymbol{A_h,O_{h+1}}}^\top \vec{\nu}_{h+1, \boldsymbol{\bar{F}}_{h+1}}$
\end{definition}
The the update operator of the conjugate beliefs is the transpose of that of the risk belief, which immediately implies their inner product is invariant with time
\begin{equation}\label{inner_invariant}
\begin{aligned}
\langle \vec{\sigma}_{H+1,\f_{H+1}}, \vec{1} \rangle
\equiv 
\langle \vec{\sigma}_{h,\f_h}, \vec{\nu}_{h,\bar{\f}_h}\rangle
,\quad \forall h\in [H+1]
\end{aligned}
\end{equation}
Consequently, we have
\begin{equation*}
\begin{aligned}
\label{inner}
&J(\pi;\mathcal{P}, \gamma)
\equiv 
\frac{1}{\gamma}\ln 
\mathbb{E}_{\mathcal{P}^\prime}^\pi
\left[\langle{\vec{\sigma}_{h, \mathbf{F}_{h}}, \vec{\nu}_{h, \mathbf{\bar{F}}_{h}}\rangle}\right], \forall h\in[H+1]
\end{aligned}
\end{equation*}
Motivated by the equation above we introduce the value functions and Q-functions in our problem setting, which take equivalent forms presented in Definition \ref{def_value_short}.
\begin{equation}
\begin{aligned}\label{def_val_raw_main}
\mathsf{V}_h^{\pi}(\boldsymbol{F_h})
:=&
\frac{1}{\gamma}\ln \mathbb{E}^\pi_{\mathcal{P}^\prime}\left[\langle{\vec{\sigma}_{h, \boldsymbol{F_h}}, \vec{\nu}_{h, \boldsymbol{\bar{F}}_h}\rangle \bigg|\boldsymbol{F_h}}\right]
\\
\mathsf{Q}_h^\pi(\boldsymbol{F_h},\boldsymbol{A_h})
:=&
\frac{1}{\gamma}\ln \mathbb{E}_{\mathcal{P}^\prime}
\left[\langle{\vec{\sigma}_{h, \boldsymbol{F}_{h}}, 
\vec{\nu}_{h, {\boldsymbol{\bar{F}}}_{h}}\rangle  
\bigg|\boldsymbol{F_h},\boldsymbol{A_h}}\right]
\end{aligned}
\end{equation}
We can immediately obtain the Bellman equations introduced in 
Section~\ref{bellman_Eq_short} by the iterative expectation formula. 

\subsection{Beta Vector and the Bonus Design}\label{section_beta_vector_main}

To recover the Markov property of the value functions,  
in this work, we put forward the concept of beta vectors $\vec{\beta}_{h,f_h}^{\pi}$, 
and take it as the surrogate for $\mathsf{V}_{h}^\pi(f_h)$. Beta vectors utilize the hindsight observations to obtain a polynomial regret for our algorithm. We will study the statistical error of approximating a beta vector and then design a bonus that ensures optimism in the value functions and encourages greedy exploration.
\begin{definition}{(Beta vector)}\label{def_beta_vector_short}
The beta vectors  of a risk-sensitive POMDP $\mathcal{P}$ with policy $\pi$
is a series of random vectors in $\R^{S}$, which are specified by
    \begin{equation}
        \begin{aligned}
        \vec{\beta}^\pi_{H+1,\boldsymbol{F_{H+1}}}:=&\vec{1}_{S}
        \\
        \vec{\beta}^\pi_{h,\boldsymbol{F_h}}:=&
        \mathbb{E}_{\mathcal{P}^\prime}^\pi
        \left[\vec{\nu}_{h,\boldsymbol{\bar{F}_h}}|\boldsymbol{F_h}\right]
        ,
        \quad \forall 2 \leq h \leq H.
        \\
        \vec{\beta}^\pi_1 :=&
        \mathbb{E}_{\mathcal{P}^\prime}^\pi
        \left[\vec{\nu}_{1,\boldsymbol{\bar{F}1}}\right]
        \end{aligned}
    \end{equation}
\end{definition}
where $\vec{\nu}_{f,\boldsymbol{\bar{F}_h}}$ are the conjugate beliefs specified in Definition~\ref{def_conjugate_beliefs_short}.

The beta vectors can be viewed as the risk-sensitive counterpart of the alpha-vector well-known in the POMDP literature \cite{PBVI}.\footnote{For a detailed comparison, please refer to Appendix~\ref{remark_alpha-beta-comparison}.}
Similarly, they help to represent the value function in a simple form:
\begin{theorem}(Beta Vector Representation)\label{beta_represent_short}
\\
$\forall h\in [H+1], f_{h} \in \sF{h}:\ \ $ 
$\mathsf{V}_{h}^{\pi}(f_h)={    \frac{1}{\gamma}    \ln    \langle    \vec{\sigma}_{h,f_h}    \vec{\beta}^\pi_{h,f_h}    \rangle    }$
\end{theorem}
As will be shown in Theorem \ref{theorem_beta_vector_representation of value functions}, the beta vector $\left[\vec{\beta}^{\pi}_{h,
    \boldsymbol{F_h}
    }\right]_{s_h}$ is equal to the following quantity:
\begin{equation}
    \begin{aligned}
    \label{beta_update_law_short}
    \mathbb{E}_{{\A_h}\sim \pi_h(\cdot|\boldsymbol{F_h})}
    \left[e^{\gamma r_h(s_h,{{\A_h}})}
    \sum_{{{s_{h+1}}}\in \mathscr{S}}
    \mathbb{T}_{h,{{\A_h}}}({{s_{h+1}}}|{s_h})
    \sum_{{{o_{h+1}}}\in \mathscr{O}}
    \mathbb{O}_{h+1}({{o_{h+1}}}|{{s_{h+1}}})
    \left[ 
    \vec{\beta}^{\pi}_{h+1,f_{h+1}=(\boldsymbol{F_h},{{\A_h}},{{o_{h+1}}})}
    \right]_{{{s_{h+1}}}}
    \right]\nonumber
    \end{aligned}
\end{equation}
where the expectation is computed with respect to probability measures that solely rely on the previous states $s_h$, rather than the entire history. We will see in Section~\ref{section_regret_ana_main} that the Markov property of the stochastic process $\{\beta_{h,\boldsymbol{F_h}}^{\pi}\}_{h\geq 1}$ will cooperate with the hindsight observations to secure a polynomial sample complexity for our algorithm. 
Consequently, we will focus on studying the beta vectors in the proceeding bonus design and regret analysis.

\paragraph{Bonus Function}\label{section_bonus_short}
In reinforcement learning, we are concerned with the statistical error of the 
value functions brought by the inaccurate estimate of the environment $\mathcal{P}$.
Using Theorem \ref{beta_represent_short}, 
it suffices to calculate the error that occurred in the evolution of the beta vectors. Based on concentration inequalities, detailed analysis in Appendix~\ref{section_lower_bound_beta_vector_error} shows that the empirical error of beta vectors is controlled by the following bound with probability at least $1-2\delta$:
\begin{equation*}
\begin{aligned}
&
\abs{
\left(\mathbb{E}\pi_{_{\widehat{\mathcal{P}}^{k}(\cdot,\cdot|s_h)}}
-
\mathbb{E}^\pi_{\mathcal{P}(\cdot,\cdot|s_h)}
\right)
e^{\gamma r_h(s_h, \pi_h(f_h))}
\beta^{\pi}_{h+1}(\cdot, \cdot;f_h)
}
    \\\leq& 
    \abs{e^{\gamma(H-h+1)}-1}
    \cdot 
    \left[
    \min\left\{1, 
    3
    \sqrt{\frac{SH \iota}{\widehat{N}_h^k(s_h,\pi_h(f_h))\vee 1}}
    \right\}
    \right.
    \\+&
    \left.
    \sum_{s^\prime}\widehat{\mathbb{T}}^k_{h,\pi_h(f_h)}(s^\prime|s_h)
    \min\left\{
    1,
    3\sqrt{\frac{OH \iota}{{\widehat{N}_{h+1}^{k+1}(s^\prime)\vee 1}}}
    \right\}
    \right]
\end{aligned}
\end{equation*}
where we have temporarily viewed the beta vector as a binary function over $\sO\times \sS$. We refer to $\widehat{N}_{h+1}^{k+1}(s)$ and $\widehat{N}_{h+1}^{k+1}(s,a)$ as the occurrence frequencies of states and actions in the data samples obtained in the learning process.

We will abbreviate the two minimums in Eq.~\eqref{def_residue_short} as the \emph{the transition error residue} ${\mathsf{t}_h^k(s_h,a_h)}$ and \emph{the emission error residue} ${\mathsf{o}_{h+1}^k(s_{h+1})}$ respectively.
With some minor adjustments, the upper bound in Eq.~\eqref{def_residue_short} will be used to define our bonus function for all the state-action pairs under the risk level $\gamma$, which is presented in Eq.~\eqref{bonus_short}. 

\paragraph{Optimism}
In the algorithm, we have the freedom to design the empirical value function, as well as the empirical beliefs $\widehat{\sigma}_{h,f_h}^k$ and the beta vectors $\widehat{\beta}_{h,f_h}^ {k,\pi}$.
In this study, the empirical beliefs are determined by Eq.~\eqref{belief_update_short} and we define the empirical beta vectors by the iterative formulas below:
\begin{equation}
\begin{aligned}\label{update_empirical_beta_short}
\left[\widehat{\beta}_{H+1}^{k,\pi}\right]_{s_1} :=& 1\\
\left[\widehat{\beta}^{k, \pi}_{h,f_h}\right]_{s_h}
:=&
\mathbb{E}_{a_h \sim \pi_h(\cdot|f_h)} 
\left[e^{\gamma r_h(s_h,a_h)}
\right.
\sum_{s^\prime}
\mathbb{\widehat{T}}_{h,a_h}^k(s^\prime|s_h)
\\
&\cdot
\left.
\sum_{o^\prime}
\mathbb{\widehat{O}}_{h+1}^k(o^\prime|s^\prime)
\left[\widehat{\beta}^{k, \widehat{\pi}^k}_{h+1,f_{h+1}=(f_h,a_h,o^\prime)}\right]_{s^\prime}+
{\sign{\gamma}\cdot \mathsf{b}_{h}^k(s_h,a_h;\gamma)}
\right] 
\end{aligned}
\end{equation}
Finally, we mimic the representation Theorem \ref{beta_represent_short}
and construct the empirical value function by
$$
\widehat{V}_h^{\pi}(f_h)
:=
\frac{1}{\gamma} \ln \langle 
\widehat{\sigma}^k_{h,f_h} , \widehat{{\beta}}^{k,\pi}_{h,f_h}
\rangle
$$
As will be shown in Appendix~\ref{section_optimism_in_value_functions}, 
the introduction of an additional bonus term in Eq.~\eqref{update_empirical_beta_short} 
will ensure that the value functions will be over-estimated in the empirical model:
\begin{corollary}(Optimism in value functions)\label{optimism_short}
\\For any risk-sensitivity level $\gamma \ne 0$ and episode number 
$k\in [K]$, we have $V_1^{\pi^\star} \leq \widehat{V}_1^{\widehat{\pi}^k}$, 
where $\widehat{\pi}^k$ is the optimal policy in $\mathcal{\widehat{P}}^k$. 
\end{corollary}

\section{Regret Analysis}\label{section_regret_ana_main}
We now give an overview of the proof of Theorem \ref{theorem_regret_short}. Technical details are provided in Appendix~\ref{section_regret_analysis}.

\subsection{From Regret to Beta Vectors}
With the help of Corollary \ref{optimism_short} and Theorem \ref{beta_represent_short}, we can control the regret by the risk beliefs and the beta vectors:
\begin{equation}
    \begin{aligned}
        \label{regret2betaerr_short}
        \text{Regret}(K; \mathcal{P}, \gamma)
        \leq \mathsf{K}_{\gamma}
        \sum_{k=1}^K
        \abs{
        \langle \widehat{\sigma}_1^k-\vec{\sigma}_1, \widehat{\beta}_1^{k,\widehat{\pi}^k} \rangle }
        +
        \abs{
        \mathbb{E}_{\mathcal{P}}
        \left(
        \widehat{\beta}_{1}^{k,\widehat{\pi}^k}
        -\vec{\beta}_{1}^{\widehat{\pi}^k}
        \right)}
    \end{aligned}
\end{equation}
where $\mathsf{K}_{\gamma}$ is the Lipschitz constant of the entropic risk measure, which is $\frac{e^{(-\gamma)^+}}{\gamma}$ for all $\gamma \ne 0$.\footnote{Please refer to Lemma~\ref{lemma_linearize} for a detailed derivation.}

Using concentration inequalities \ref{lemma_concentration_in_one_norm} and \ref{lemma_naive_concentration}, the first term in Eq.~\eqref{regret2betaerr_short} is bounded by
\begin{equation}\begin{aligned}\label{initial_error_main}
\mathsf{K}_{\gamma}
\sum_{k=1}^K
\abs{
\langle \widehat{\sigma}_1^k-\vec{\sigma}_1, \widehat{\beta}_1^{k,\widehat{\pi}^k} \rangle 
}
\leq 
\frac{e^{\abs{\gamma }H}-1}{\abs{\gamma}}\sqrt{2KS\ln \frac{K}{\delta}}
\end{aligned}\end{equation}
with probability at least $1-\delta$.
We name the right-hand side of Eq.~\eqref{initial_error_main} as ``the prior error," which arises from the inaccurate estimate of the prior distribution $\mu_1$.

Next, our attention turns to bounding the second term in Eq.~\eqref{regret2betaerr_short}, which we refer to as the initial beta vector error and denote as $\Delta_1^k$.

\subsection{Control the Error Between Beta Vectors}
We extend the definition of beta vector errors to $h\in [H+1]$:
$$\Delta_h^k:=
\abs{
\mathbb{E}_{\mathcal{P}}
\left[
\widehat{\beta}_{h,\boldsymbol{F_h}}^{k,\widehat{\pi}^k}(\mathbf{S}_h)
-
{\vec{\beta}}_{h,\boldsymbol{F_h}}^{\widehat{\pi}^k}(\mathbf{S}_h)
\right]
}$$
which demonstrate the average error of the beta vectors at step $h$ incurred by the inaccurate empirical estimate of the POMDP model.
We observe that
\footnote{We have abbreviated the transition and emission matrices as operators in Eq.~\eqref{delta_hk_short}. We also omitted the variables of $\mathsf{b}_h^k(\cdot,\cdot)$.}

\begin{equation}\label{delta_hk_short}
\begin{aligned}
\Delta_h^k=&
|
\underbrace{\text{sgn}(\gamma)
\mathbb{E}_{\mathcal{P}}[\mathsf{b}_h^k]
}_{\text{I
}}
+
\underbrace{\mathbb{E}_{\mathcal{P}}
[(\widehat{\mathbb{T}}^k_h\widehat{\mathbb{O}}^k_{h+1}-{\mathbb{T}_h}{\mathbb{O}_{h+1}})e^{\gamma r_h}{\vec{\beta}^{\widehat{\pi}^k}_{h+1}}]
}_{\text{II
}}
\\+&
\underbrace{\mathbb{E}_{\mathcal{P}}
[(\widehat{\mathbb{T}}^k_h\widehat{\mathbb{O}}^k_{h+1}-{\mathbb{T}_h}{\mathbb{O}_{h+1}})
(
e^{\gamma r_h}\widehat{\beta}^{\widehat{\pi}^k}_{h+1}-
e^{\gamma r_h}{\vec{\beta}^{\widehat{\pi}^k}_{h+1}})]
}_{\text{III
}}
\\+&
\underbrace{\mathbb{E}_{\mathcal{P}}
[{\mathbb{T}_h}{\mathbb{O}_{h+1}}(e^{\gamma r_h}\widehat{\beta}^{\widehat{\pi}^k}_{h+1}-e^{\gamma r_h}{\vec{\beta}^{\widehat{\pi}^k}_{h+1}})]
}_{\text{IV
}}
|
\end{aligned}
\end{equation}
The terms in Eq.~\eqref{delta_hk_short} are controlled by the concentration inequalities. 
In Appendix~\ref{section_control the evolution error}, we show that 
\begin{equation}\begin{aligned}\label{beta_err_recur_Eq._short}
\Delta^k_{h}
\leq \ e^{\gamma^+}\Delta^k_{h+1}+4\mathbb{E}_{\mathcal{P}}\left[\mathsf{b}_h^k\right]
\quad \forall h\in [H]
\end{aligned}\end{equation}
Then we invoke Lemma \ref{lemma_backward_linear_Eq.} to obtain an upper bound on $\Delta_1^k$ based on Eq.~\eqref{beta_err_recur_Eq._short}. 
In the end, we prove that the sum of the initial beta vector error is dominated by the value of the bonus functions on the sampled trajectories and the bias in the bonus function incurred by the empirical estimator:

\begin{equation}\begin{aligned}\label{regret2bias_short}
&
\mathsf{K}_{\gamma} \sum_{k=1}^K \Delta_{1}^k
\lt
4 
\mathsf{K}_\gamma
\sum_{h=1}^{H}
e^{\gamma^+(h-1)}
\cdot 
\left[
\sum_{k=1}^K
\underbrace{
\mathsf{b}_h^k
({{{\hat{s}^k_{h}}}},{{{\hat{a}^k_{h}}}};\gamma)
}_{\text{Bonus samples
}}+
\sum_{k=1}^{K}
\underbrace{
\mathbb{E}_{\mathcal{P}}
\mathsf{b}_h^k
(\boldsymbol{S_h}, \boldsymbol{A_h};\gamma)
-
\mathsf{b}_h^k
({{{\hat{s}^k_{h}}}},{{{\hat{a}^k_{h}}}};\gamma)
}_{\text{Empirical bias
}}
\right]
\end{aligned}\end{equation}
We will first try to find an upper bound
for the second summation in Eq.~\eqref{regret2bias_short}.
Concentration inequality \ref{Lemma_concentration_MDS} implies that with probability at least $1-\delta$, 
\begin{equation}
\sum_{k=1}^{K}
\underbrace{
\mathbb{E}_{\mathcal{P}}
\mathsf{b}_h^k
(\boldsymbol{S_h}, \boldsymbol{A_h};\gamma)
-
\mathsf{b}_h^k
({{{\hat{s}^k_{h}}}},{{{\hat{a}^k_{h}}}};\gamma)
}_{\text{Empirical bias
}}
\leq
\abs{e^{\gamma (H-h+1)}-1}
\sqrt{K/2\cdot \iota}
\end{equation}
A bound on the first term in Eq.~\eqref{regret2bias_short} can also be derived from Eq.~\eqref{bonus_short}:
\begin{equation}\label{bonus_samples_step1}
\sum_{k=1}^K
\underbrace{
\mathsf{b}_h^k
({{{\hat{s}^k_{h}}}},{{{\hat{a}^k_{h}}}};\gamma)
}_{\text{Bonus samples
}}
\leq
\abs{e^{\gamma(H-h+1)}-1}
\cdot\sum_{k=1}^K
\mathsf{t}^k_{h}({{\hat{s}^k_h,\hat{a}^k_h}})
+
\sum_{s^\prime}
\widehat{\mathbb{T}}^k_{h,{{\hat{a}^k_h}}}
(s^\prime|{{{{\hat{s}^k_h}}}})
\mathsf{o}_{h+1}^k(s^\prime)
\end{equation}
To establish an upper bound for the summation in Eq.~\eqref{bonus_samples_step1},
we will telescope the equation twice and utilize the results from Lemmas
\ref{lemma_naive_concentration} and \ref{Lemma_concentration_MDS}
to bound the statistical error in the transition process and the emission process.
Subsequently, we can show that with a probability of at least $1-2\delta$,
\begin{equation}\begin{aligned}\label{bonus_samples_step2}
&\sum_{k=1}^K \mathsf{t}^k_{h}({{\hat{s}^k_h,\hat{a}^k_h}})
+
\sum_{s^\prime}\widehat{\mathbb{T}}^k_{h,{{\hat{a}^k_h}}}(s^\prime|{{{{\hat{s}^k_h}}}})\mathsf{o}_{h+1}^k(s^\prime)
\\=&
\sum_{k=1}^K
\underbrace{
\left[
{\mathbb{T}}^k_{h,{{\hat{a}^k_h}}}(\cdot|{{{{\hat{s}^k_h}}}})\mathsf{o}_{h+1}^k(\cdot)
-
\mathsf{o}_{h+1}^k({{\hat{s}^k_{h+1}}})
\right.
}_{\text{Concentration of MDS
}}
+
\underbrace{
\mathsf{t}^k_{h}({{\hat{s}^k_h,\hat{a}^k_h}})
+
\mathsf{o}_{h+1}^k({{\hat{s}^k_{h+1}}})
}_{\text{Residues
}}
\\&+
\underbrace{
\left.
\widehat{\mathbb{T}}^k_{h,{{\hat{a}^k_h}}}(\cdot|{{{{\hat{s}^k_h}}}})\mathsf{o}_{h+1}^k(\cdot)
-
{\mathbb{T}}^k_{h,{{\hat{a}^k_h}}}(\cdot|{{{{\hat{s}^k_h}}}})\mathsf{o}_{h+1}^k(\cdot)
\right]
}_{\text{Azuma Hoeffding
}}
\\
\leq 
&
\sqrt{2 K \frac{\ln HSA}{\delta}}
+
2
\sum_{k=1}^K
\mathsf{t}_h^k({{\hat{s}^k_h}},{{\hat{a}^k_h}})
+
\mathsf{o}_{h+1}^k({{\hat{s}^k_{h+1}}})
\end{aligned}\end{equation}
The relations above imply that the right-hand side of Eq.~\eqref{regret2bias_short} is dominated by the transition and emission residues, which are defined in Eq.~\eqref{def_residue_short}.

\subsection{Sum Up the Residue}
Finally, the pigeon-hole Lemma \ref{lemma_pigeon-hole} will help us compute the summation of the residue terms. Indeed, 
\begin{equation*}\begin{aligned}\label{bonus_samples_step3}
\sum_{k=1}^k {\mathsf{t}_h^k({{\hat{s}^k_{h}, \hat{a}^k_h}})}
=&\sum_{k=1}^k
\min\left\{\ 1, \ 3\sqrt{\frac{S H{\ln\frac{KHSOA}{\delta}}}{\widehat{N}_h^k({\hat{s}^k_{h}, \hat{a}^k_h})\vee 1)}}
\ \right\}
\leq \left(3\sqrt{SH}  \iota \right) 2\sqrt{K SA}
\end{aligned}\end{equation*}
Similarly, we can control the summation of emission residues by the following bound:
$$
\sum_{k=1}^K\mathsf{o}_{h+1}^k({{\hat{s}^k_{h+1}}})\leq 3\sqrt{OH} \cdot \iota \cdot 2\sqrt{KS}
$$
Bringing these relations back to Eqs.~\eqref{bonus_samples_step2} and \eqref{bonus_samples_step1}, we obtain
\begin{equation}\begin{aligned}\label{step_minus1}
&
\sum_{k=1}^K
\mathsf{b}_h^k({\hat{s}^k_h,\hat{a}^k_h};\gamma)
+
\mathbb{E}_{\mathcal{P}}[\mathsf{b}_{h}^k(\boldsymbol{S_h}, \boldsymbol{A_h};\gamma)]-\mathsf{b}_h^k({\hat{s}^k_h,\hat{a}^k_h};\gamma)
\\\leq &
12 \cdot
\underbrace{\abs{e^{\gamma(H-h+1)}-1}
}_{\text{Bonus magnitude
}}
\cdot	
\underbrace{{{{\sqrt{H}}}\cdot
\sqrt{\ln \left(\frac{KHSOA}{\delta}\right)}}
}_{\text{History-dependency of POMDP
}}
\cdot\quad
\left(
\underbrace{
{{\sqrt{{KS^{{{^2}}}A}}}}
}_{\text{Hidden state error
}}
+
\underbrace{{{\sqrt{{KS}{{{{O}}}}}}
}
}_{\text{Observation error
}}
+
\underbrace{\sqrt{K}
}_{\text{Empirical bias
}}
\
\right)
\end{aligned}\end{equation}
The last step remaining is to take Eq.~\eqref{step_minus1} back to Eq.~\eqref{regret2bias_short} and then bring Eq.~\eqref{regret2bias_short} with Eq.~\eqref{initial_error_main} to Eq.~\eqref{regret2betaerr_short}. 
Rearranging terms, we conclude that with probability at least $1-4\delta$, 
\begin{equation}\label{regret_full_expression}
\text{Regret}(K;\mathcal{{P}},{\gamma}) 
\leq 
48
\underbrace{{\frac{e^{\abs{\gamma} H}-1}{\abs{\gamma}}}
}_{\text{Risk measure
}}
\underbrace{{{{\sqrt{H\cdot
{\ln \frac{KHSOA}{\delta}}}}}}
}_{\text{History-dependency
}}
\cdot
\underbrace{
\big(\sqrt{KS}
}_{\text{Prior error
}}
+
\underbrace{
{{
H\sqrt{{KS^{{{^2}}}A}}}}
}_{\text{Transition error
}}
+
\underbrace{{{
H\sqrt{{KS}{{{{O}}}}}}
}
}_{\text{Emission error
}}
+
\underbrace{
H\sqrt{K}
\big)
}_{\text{Empirical bias
}}
\end{equation}
Neglecting lower order terms, we obtain the upper bound presented in Theorem \ref{theorem_regret_short}:
$$
\operatorname{Regret}(K;\mathcal{P},\gamma)
\leq 
\mathcal{\tilde{O}}
\left(
\frac{e^{\abs{\gamma } H}-1}{\abs{\gamma}H}
H^\frac{5}{2}\sqrt{KS^2AO}
\right)
$$
\paragraph{Discussion}
In the risk-neutral setting, our regret improves the result given by \cite{Lee2023hindsight}
\footnote{For details, please refer to Theorem C.1 of \cite{Lee2023hindsight}.}
\begin{equation*}
\begin{aligned}
&
\operatorname{Regret}(K;\mathcal{P})
\leq 
\tilde{O}
\left(
\sqrt{S A H^4 K }
+\boldsymbol{H^3  S} \sqrt{O}
\right.
+\left.
\boldsymbol{H^4 S^2 A}(1+\ln K)
+
H \boldsymbol{S A} \sqrt{H^3 }
\right)
\end{aligned}
\end{equation*}
in the order of $S,A$, and $H$. 
The improvement is attributed to the refined analysis in this work. 
Our sample complexity also nearly reaches the lower bound of learning a hindsight POMDP, 
which is
$
\Omega\left(\frac{SO}{\epsilon^2}\right)
$ according to \cite{Lee2023hindsight}.

In the completely observable setting, with some adjustments, our algorithm can degenerate to the algorithm 1 in \cite{Fei21improve} and thus matches their upper bound
\footnote{Please refer to Appendix~\ref{section_compare_other_setting} for details.}
$$
\operatorname{Regret}(K;\mathcal{M},\gamma)
\leq \tilde{O}
\left(\frac{e^{|\gamma| H}-1}{|\gamma| H} \sqrt{K H^4 S^2 A}
\right)
$$
Moreover, our regret achieves the lower bound of risk-sensitive RL \cite{Fei20} concerning $K$ and $\gamma H$, with the order of $H$ only slightly higher.
$$
\operatorname{Regret}(K; \mathcal{M}, \gamma) 
\geq \tilde{O}
\left(
\frac{e^{\abs{\gamma }\frac{H}{2}}-1}{\abs{\gamma}H} \cdot H^{\frac{3}{2}}\sqrt{K}
\right)
$$

\section{Conclusion and Future Work}\label{section_conclusion_futuer_main}
In this study, we introduce a novel formulation of risk-sensitive RL in a partially observable environment with hindsight observations. We provide the first provably sample-efficient algorithm tailored for the new setting, whose regret improves existing upper bounds and nearly reaches the lower bounds in the degenerated cases. 
Our analysis also explains how the sample complexity is affected by the risk-awareness and history-dependency inherent in our problem.

One future direction is to derive similar results in the function-approximation setting. Another avenue is to extend our findings to risk measures other than the utility functions.


\bibliography{ref}
\bibliographystyle{plain}

\newpage
\onecolumn

\appendix


\appendixpage

\section{Notations and Concepts}
In this section we provide several additional concepts and notations not mentioned in Section \ref{section_intro_main}.

\paragraph{Additional Notations}
Given a vector $\mathbf{x}\in\mathbb{R}^d$, we denote its $i^{th}$ entry as $\mathbf{x}(i)$ or $\left[\mathbf{x}\right]_{i}$. For a matrix $\A\in\mathbb{R}^{m\times n}$, we use $\A_{ij}$ or $[\A]_{i,j}$ to indicate the $(i,j)^{\text{th}}$ entry. The comparison and expectation of random vectors are defined component-wise. To represent the indicator operator, we employ $\mathds{1}\{\cdot\}$, and the signature function is denoted as $\text{sgn}(\cdot)$. Additionally, we use $\text{Unif}(\mathcal{X})$ to express the uniform distribution over the finite space $\mathcal{X}$. 

\paragraph{Remark on the POMDP}

\begin{remark}\label{remark_first_observation_not_defined}
    Following the convention of POMDP literature \cite{MonahanSurvey1982, Baras&Elliott94, Cavazos2005, Golowich2022-a}, no observation is made at the first step. The definition of $\mathbf{O}_h$ begins from $h=2$. 
    The first action $\mathbf{A}_1$ is chosen based on the agent's prior knowledge. 
    For the sake of consistent notations, we still adopt the notation of $\boldsymbol{F_1}$ 
    and we use $\mathbb{P}(\cdot|\boldsymbol{F}_1)$ and $\mathbb{P}(\cdot)$ interchangeably. 
    We permit the environment to generate the observation $\mathbf{O}{H+1}$ when in state $\mathbf{S}{H+1}$, 
    but the agent abstains from taking any actions at $H+1$.
\end{remark}

We will use the following fact extensively which expresses the recursive relation of the ``history'' defined in Section \ref{section_formulate_main}. 
\begin{fact}
\label{fact_f_h+1}
$\forall h \in [H-1]:\quad \boldsymbol{F_{h+1}}=(\boldsymbol{F_h},\A_h,\O_{h+1})
, \qquad \text{where } \A_h=\pi_h(\boldsymbol{F}_h)$
\end{fact}
Throughout this study, $f_{h+1}$ and $(f_h, a_h, o_{h+1})$ will be used interchangeably.

Another concept that relates to the ``history'' is the trajectory $\boldsymbol{\tau}_h$ of the Markov process.
\begin{definition}(Trajectory)
\label{def_trajectories}
\begin{equation}
\begin{aligned}
&\text{Full trajectory}\quad 
\boldsymbol{\bar{\tau}_h}:=
(\S_1,\A_1,\ldots \S_h,\O_h,\A_h), \forall h\in [H]
\quad, \boldsymbol{\bar{\tau}_{H+1}}:=
(\S_1,\A_1,\ldots,\S_H,\O_H,\A_H,\S_{H+1})\\
&\text{Observable trajectory}\quad 
\boldsymbol{\tau_h}:=
(\A_1,\ldots,\O_h,\A_h), \forall h\in[H]
\end{aligned}
\end{equation}
\end{definition}

\paragraph{Optimization Objective using General Utility Risk Measure}
In this work we refer the utility risk as any strictly increasing function that is continuously differentiable. 
We can extend many results in this work to general utility risk measures. We will present our proofs using the utility function $U$ and instantiate it to the entropic risk ($U (\cdot)=\gamma e^{\gamma (\cdot)}$) when necessary.

The optimization objective using arbitrary utility risk measure $U$ is defined as
\begin{equation}
   \begin{aligned}\label{objective_entropic}
   \underset{\pi}{\text{maximize}} \ \ 
   U^{-1}\mathbb{E}_{\mathcal{P}}^{\pi}
   U
   \left[\sum_{t=1}^H r_t(\boldsymbol{S_t,A_t})\right]   
   \end{aligned}
\end{equation}

\section{The Structure of Risk-sensitive POMDP}\label{section_structure_risk_POMDP}
In what follows, we present the theoretical framework of partially observable reinforcement learning using arbitrary utility risk measures. Our framework builds upon the studies of \cite{Baras&Elliott94, Cavazos2005, baauerle2017partially}. Furthermore, we introduce novel concepts and provide several new proofs in a more comprehensive setting, enhancing the existing literature.

Given that the studies of risk-sensitive POMDP are relatively historical, we will provide detailed discussions about the intuition and implications behind various concepts and results. We aim to elucidate these findings, as they will serve as a foundation for the algorithm design and regret analysis in the subsequent sections.

\subsection{Change of Measure}\label{section_change_of_measure}
In reinforcement learning, the lack of knowledge about the emission process $\mathbf{O}_h$ presents a significant challenge for statistical inference, which motivates us to devise a surrogate POMDP $\mathcal{P'}$ named ``reference model", which possesses a simplified emission process. 
\begin{definition}(Reference model of a POMDP)
    \\
    Given a POMDP model $\mathcal{P}=(\mathscr{S,O,A}; \mu_1, \mathbb{T}, \mathbb{O};K, H, r)$ and a reference measure $\mathbb{O}^\prime(\cdot)\in \Delta(\sO)$, the reference model of $P$ specified by $\mathbb{O}^\prime$ is another partially observable Markov decision process $\mathcal{P}^\prime=(\mathscr{S,O,A}; \mu_1, \mathbb{T}, \mathbb{O}^\prime;K, H, r)$, in which for all $h\in [H]$ and $s_h\in \sS$, we have
    $\mathbb{O}^\prime_{h}(\cdot|s_h)=\mathbb{O}^\prime(\cdot)$. 
\end{definition}
In the reference model, the initial distribution and transition matrices mirror those of the real-world POMDP $\mathcal{P}$. However, the observations and hidden states are statistically independent and the emission process is stationary with a predefined observation probability. Consequently, the observations are separate from the underlying transition process and are independent of the history.

The probability of generating a  full trajectory in the two models can be expressed as
\begin{equation}
\label{trajectory_probabilities}
\begin{aligned}
\mathbb{P}_{\mathcal{P}}^\pi(\mathbf{\bar{\tau}_h})
=&
\mu_1(\S_1)\bO_{1}(\O_1|\S_1)\pi_1(\A_1|\O_1)
\cdot 
\bT_{h,\A_1}(\S_2|\S_1) \bO_{2}(\O_2|S_2)\pi_2(\A_2|\F_2)
\\&
\ldots
\bT_{h-1,\A_{h-1}}(\S_{h+1}|\S_h)\bO_{h}(\O_h|\S_h)\pi_h(\A_h|\F_h)
\\=&
\left[\mu_1(\S_1)\prod_{t=1}^{h-1} \bT_{t,\A_t}(\S_{t+1}|\S_t)\right]
\cdot
\left[\prod_{t=2}^h \bO_{t}(\O_t|\S_t)\right]
\cdot
\left[\prod_{t=1}^h \pi_{t}(\A_t|\F_t)\right]\\
\mathbb{P}_{\mathcal{P^\prime}}^\pi(\mathbf{\bar{\tau}_h})
=&
\left[\mu_1(\S_1)\prod_{t=1}^{h-1} \bT_{t,\A_t}(\S_{t+1}|\S_t)\right]
\cdot
\left[\prod_{t=2}^h \bO^\prime_{t}(\O_t|\S_t)\right]
\cdot
\left[\prod_{t=1}^h \pi_{t}(\A_t|\F_t)\right]\\
\end{aligned}
\end{equation}
Eq.~\eqref{trajectory_probabilities}
suggests that conditioned on the generated sigma-algebra $\mathscr{G}_h=\sigma\left(\{\S_t,\O_t,\A_t\}_{t=1}^h\right)$, the Radon-Nykodym derivative between the two trajectory probabilities takes the form of 
\begin{equation*}
\begin{aligned}
\frac{d\mathbb{P}_{\mathcal{P}}^\pi}{d\mathbb{P}_{\mathcal{P}^\prime}^\pi}\bigg|_{\mathscr{G}_h}
=
\prod_{t=2}^h \frac{\bO_t(\O_t|\S_t)}{\bO^\prime_t(\O_t|\S_t)}
:=
{D}_{h}(\mathbf{O}_{2:h}, \mathbf{S}_{2:h})
:=
\mathbf{D}_h
\end{aligned}
\end{equation*}
By Theorem \ref{rn_derivative}, for any measurable function $f$ of the full trajectory $\bar{\tau}_h$,
\footnote{
We should also guarantee that the reference measure $\bO^\prime $
is strictly positive a.s. and $\mathbb{P}_{\mathcal{P}^\prime}^\pi << \mathbb{P}_{\mathcal{P}}^\pi$.
See Section \ref{rn_derivative} for details.}
\begin{equation}
\begin{aligned}
\mathbb{E}^\pi_{\mathcal{P}}[f(\mathbf{\bar{\tau}_h})]
=
\mathbb{E}^\pi_{\mathcal{P}^\prime}[\mathbf{D}_h\cdot f(\mathbf{\bar{\tau}_h})]
=
\int_{\sT{h}}\ 
({D}_h f)_{(\bar{\tau}_h)}\cdot d\mathbb{P}^\pi_{\mathcal{P}^\prime}(\bar{\tau}_h)
\end{aligned}
\end{equation}
The two expectations are taken with respect to the randomness in the transitions, emissions and the same policy. Then we can rewrite our optimization objective in Eq.~\eqref{objective_entropic} by the change of measure
\begin{equation}
\begin{aligned}
J(\pi;\mathcal{P}):=
\frac{1}{\gamma}\ln
\mathbb{E}^{\pi}_{\mathcal{P}}
\left[e^{\gamma \sum_{h=1}^H r_h(\mathbf{\S_h,\A_h})}\right]
=
\frac{1}{\gamma}\ln 
\mathbb{E}^{\pi}_{\mathcal{P}^\prime}
\left[\mathbf{D}_H \cdot 
e^{\gamma \sum_{h=1}^H r_h(\mathbf{S_h,A_h})}\right]
\end{aligned}
\end{equation}
\begin{remark}
In general, the conditional expectation $\mathbb{E}_{\mathcal{P}^\prime}^\pi[\mathbf{D}_h\cdot 
\ | \f_h]$ in Definition \ref{def_risk_sensitive_belief} cannot be replaced by $\mathbb{E}_{\mathcal{P}}^\pi[\cdot \ | \f_h]$, as our RN derivative $\mathbf{D}_h$ is calculated from the joint but not conditional probability.
\end{remark}
\subsection{Risk-sensitive Belief}
One of the key concepts in the study of risk-neutral POMDP \cite{MonahanSurvey1982}is the ``belief state', 
which is the posterior distribution of the hidden states given the observable history. 
$$
\vec{b}_h(\cdot;\f_h)=
\mathbb{P}_{\mathcal{P}}^\pi
\left\{
\S_h=\cdot|\mathbf{F}_h=\f_h
\right\}
$$
Since we can always view a probability from the perspective 
of expectation, we observe that \footnote{In the continuous case 
we replace the indicator $\mathds{1}\{\cdot\}$ 
with the Dirac-delta function and the 
following definitions should be  
modified accordingly.}
\begin{equation}
\begin{aligned}\label{optim_obj}
\vec{b}_h(\cdot;\f_h)=&
\mathbb{P}_{\mathcal{P}}^\pi
\left\{
\S_h=\cdot|\mathbf{F}_h=\f_h
\right\}
\\
\equiv &
\mathbb{E}_{\mathcal{P}}^\pi
\left[
\mathbf{1}{\{\S_h=\ \cdot\ \}}
\mid
\mathbf{F}_h=\f_h
\right]
\\
=&
\mathbb{E}_{\mathcal{P}}^\pi
\left[
\mathbf{1}{\{\S_h=\ \cdot\ \}}
e^{\gamma \sum_{t=1}^{h-1}
r_t(\S_t,\A_t)}
\mid
\mathbf{F}_h=\f_h
\right]
\bigg|_{\gamma=0}
\end{aligned}
\end{equation}
The risk-sensitive counterpart of the belief is inspired by Eq.~\eqref{optim_obj}.
\begin{definition}(Risk-sensitive belief, Definition 2.9 in \cite{Cavazos2005})

For all $ h\in[H+1], \f_h\in \sF{h},\ s_h\in \sS:$
    \label{def_risk_sensitive_belief}
    \begin{equation}
        \begin{aligned}
        [\vec{\sigma}_{1}]_{s_1}:=&\mu_1(s_1)
        \\
        [\vec{\sigma}_{h,\f_h}]_{s_h}:=& \mathbb{E}_{\mathcal{P}^\prime}^\pi 
        \left[
        \mathbf{D}_h 
        \cdot 
        \mathds{1} \{ \S_h=s_h \}
        \exp{\gamma \ \sum_{t=1}^{h-1} r_t(\S_t,\A_t)}
        \bigg| 
        \mathbf{F}_h = \f_h
        \right]
        \end{aligned}
    \end{equation}
\end{definition}
\begin{remark}
The risk belief in this study is not normalized, since we make it a carrier of the one-step risk-sensitive reward. However, some literature \cite{Cavazos2005,baauerle2017partially} still defines a normalized belief.
\end{remark}
\begin{remark}
    In reinforcement learning, the risk beliefs corresponding to the empirical models $\widehat{\mathcal{P}}^k=(\widehat{\mu}^k_1, \widehat{\mathbb{T}}^k, \mathbb{\widehat{O}}^k)$ will be defined in a similar manner and referred to as the empirical belief $\widehat{\sigma}^k_h$.
\end{remark}
\paragraph{Relationship with the Optimization Objective}
We can use the risk belief to express the optimization objective defined in Eq.~\eqref{optim_obj}. 
\begin{equation}
\begin{aligned}
\label{why_beliefs_condition_expect}
J(\pi;\mathcal{P}):=&
U^{-1}\mathbb{E}_{\mathcal{P}}^\pi
\left[U{\sum_{h=1}^H r_h(\mathbf{\S_h,\A_h})}\right]
\\\equiv&
U^{-1}\mathbb{E}^{\pi}_{\mathcal{P}^\prime}
\left[\mathbf{D}_{H+1} \cdot 
U{\sum_{h=1}^H r_h(\mathbf{S_h,A_h})}\right]
\quad //{\text{Change of measure}}
\\=&
U^{-1}{\mathbb{E}^{\pi}_{{\mathcal{P}^\prime}}
\left[
\mathbb{E}^{\pi}_{{\mathcal{P}^\prime}}
\left[
{\mathbf{D}_{H+1}} \cdot 
U{\sum_{h=1}^H r_h(\mathbf{S_h,A_h})}
\bigg| \mathbf{F}_{H+1}
\right]
\right]
}
\\
\equiv &
U^{-1}\mathbb{E}_{\mathcal{P}^\prime}^\pi
\left[
\sum_{s_{H+1}\in\sS}
{
\mathbb{E}^{\pi}_{\mathcal{P}^\prime}
\left[
\mathbf{1}
\{
\S_{H+1}=s_{H+1}
\}
\mathbf{D}_{H+1}
\cdot 
U{\sum_{h=1}^H r_h(\mathbf{S_h,A_h})}
\bigg| \mathbf{F}_{H+1}
\right]
}
\cdot
1\ 
\right]
\ //{\text{Lemma \ref{Lemma_sum_dirac_expect_condition}}}
\\\equiv&
U^{-1}\mathbb{E}_{\mathcal{P}^\prime}^\pi
\left[
\langle
{\vec{\sigma}_{H+1,\mathbf{F}_{H+1}}}, 
\vec{1}_{S}
\rangle
\right]
\end{aligned}
\end{equation}
If we can discover the evolution law of the new belief then we will be able to break the structure of $J(\pi;\mathcal{P})$ down by dynamic programming equations, as is presented in Section \ref{section_Value functions and Bellman equations}.

\paragraph{Closed-form expression}
To gain more concrete understanding of the the specific structure of $\sigma$ we may first utilize 
the Markov property of the hidden states to expand the condition measure of $\S_h$ given $\f_h$: 
\begin{observation}(Expansion of conditional probability)
    \label{observation_expansion of conditional probability}
$\forall \pi\in \Pi, \mathcal{P}^\prime = (\mu_1,\{\bT_h\}, \{\bO^\prime_h\}), 
h\in [H+1], 
s_{1:h}\in \sS^{h}, $
$\f_h=(o_1,a_1,\ldots,o_{h-1},a_{h-1},o_h) \in \sF{h+1}, $
\begin{equation}
\begin{aligned}
   {\mathbb{P}_{\mathcal{P}^\prime}^\pi
    (s_{1:h},a_{1:h}|\f_h)}
    =&
    \mu_1(s_1)
    \cdot
    \prod_{t=1}^{h-1}
   \bT_{t,a_t}(s_{t+1}|s_t)
    \prod_{t=1}^{h}
    \pi_t(a_t|f_t)
\end{aligned}
\end{equation}
\end{observation}
The belief can then be computed by
\begin{equation}
\begin{aligned}
    \label{compute_belief}
    [\vec{\sigma}_{h+1, f_{h+1}}]_{s_{h+1}}=
    \sum_{\tilde{s}_{1:h+1}}
    {
    \mathbb{P}_{\mathcal{P}^\prime}^\pi
    (\tilde{s}_{1:h+1},{a}_{1:h+1}|f_{h+1})
    }
    \cdot
    \left[
    \mathds{1}\{{\tilde{s}_{h+1}}=s_{h+1}\}
    \cdot
    {D}_{h+1}({\tilde{s}_{2:h+1}},o_{2:h+1})
    \cdot
    \exp\gamma \sum_{t=1}^{h}r_t({\tilde{s}_t}, 
    {a_t})
    \right]
\end{aligned}
\end{equation}
\vspace{-1cm}
\paragraph{A Special Case}
In tabular case, when we select the emission matrix as the uniform distribution, 
$D_h(o_{2:h};s_{2:h})$ will become $\frac{1}{\mathbb{O}_1^\prime(o_1)}^h \prod_{t=2}^h \mathbb{O}_t(o_{t}|s_t)$.
Moreover, if the policies are deterministic,
$$
\mathbb{P}_{\mathcal{P}^\prime}^\pi(s_{1:h},a_{1:h}|\f_h)=\mu_1(s_1)\prod_{t=1}^{h-1} \bT_{t,a_t}(s_{t+1}|s_t)
$$ 
plugging the two terms together in Eq.~\eqref{compute_belief} we conclude that
\begin{corollary}(Belief vector using uniform emission matrix, Eq. (2.9) in \cite{Cavazos2005})
\text{Suppose that }
$\bO^\prime_{t}(\cdot|\S_t) =\text{Unif}\mathscr{O}$
and the policies are deterministic, then
\label{corollary_Definition of belief under uniform emission}
    \begin{equation}
        \begin{aligned}\label{update_risk_belief_full}
        [\vec{\sigma_{1}}]_{s_1}
        =&
        \mu_1(s_1)
        \\
        [\vec{\sigma}_{h,\f_h}]_{s_h}=
        &
        \abs{\sO}^h \mathbb{E}_{\mathcal{P}}^\pi 
        \left[
        \prod_{t=2}^h \mathbb{O}_t(o_t|\mathbf{S}_t)
        \mathds{1} \{ \S_h=s_h \}
        e^{\gamma \ \sum_{t=1}^{h-1} r_t(\S_t,\A_t)}\right]
        , \quad 
        \forall 2\leq h\leq H+1
        \end{aligned}
    \end{equation}
\end{corollary}

\paragraph{Evolution law}
The risk belief evolves in a Markovian manner, incorporating recent action $a_h$ and observation $o_{h+1}$ to the new belief. 
We will use $\Psi(\cdot, a_h, o_{h+1}): \R^S \to \R^S$  to denote the update operator of $\vec{\sigma}_{h,f_h}$, whose matrix representation will be denoted as $\mathsf{U}_{a_h,o_{h+1}} \in \R^{S\times S}$. The details of the update process is specified by the following theorem.
\begin{theorem}(Evolution of risk-sensitive belief, adapted from theorem 2.2 of \cite{Baras&Elliott94})
\label{theorem_evolution_of_risk_sensitive_belief}
    \begin{equation}
	\begin{aligned}
        \label{update_belief}
	\forall  h=H,H-1,\cdots,1, &\f_{h+1}=(f_h,a_h,o_{h+1})\in \sF{h+1}, s_{h+1}\in \sS:\\
		[\vec{\sigma}_{h+1,f_{h+1}}]_{s_{h+1}} 
            =&
            \Psi(\vec{\sigma}_{h,f_h}, a_h, o_{h+1})
            =
		\left[\mathsf{U}_{a_h,o_{h+1}}
            \vec{\sigma}_{h,f_h}
		\right]
		\\=&
            \sum_{s_{h}}
		\bT_{h,a_{h}}(s_{h+1}|s_h)
		\bO_{h+1}(o_{h+1}|s_{h+1})
		\cdot 
		{
		\left(
		\frac{e^{\gamma r_h(s_h,a_h)}}{\bO_{h+1}^\prime (o_{h+1}|s_{h+1})}
		\right)
		}
		\left[
		\vec{\sigma}_{h,\f_h}
		\right]_{s_h}
	\end{aligned}
\end{equation}
\end{theorem}
\begin{remark}
    The proof for this theorem in the continuous case is provided by \cite{Baras&Elliott94}. 
    However, their proof was written in the language of functional analysis and they have restricted the transition and observation probabilities to be i.i.d. Gaussian distributions.
    In the tabular setting, though \cite{Cavazos2005} have presented a similar result in Eq.(2.10), they have omitted the proof and restricted the reference measure $\mathbb{O}^\prime$ as the uniform distribution.
    For the reader's convenience, in what follows we will prove Theorem \ref{theorem_evolution_of_risk_sensitive_belief} in the tabular case using simple algebraic calculations, which also accommodates arbitrary structures of $\mathbb{O}^\prime$, $\mathbb{T}$ and $\mathbb{O}$.
\end{remark}
\begin{proof}
\begin{multline}
RHS=\sum_{s_h} \bT_{h,a_h} (s_{h+1}|s_h) 
\frac{\bO_{h+1}}{\bO^\prime_{h+1}}(o_{h+1}|s_{h+1})
\exp\gamma r_h(s_h,a_h)
\\
\left[
\sum_{\tilde{s}_{1:h}} 
\mathds{1} \{\tilde{s}_h=s_h\}
\prod_{t=2}^h 
\frac{\bO_{t}}{\bO^\prime_t}(o_{t}|\tilde{s}_{t})
\exp\gamma \sum_{t=1}^{h-1}r_t(\tilde{s}_t, a_t)
{
\mathbb{P}_{\mathcal{P}^\prime}^\pi\left(\tilde{s}_{1:h}|f_h\right)
}
\right]
\quad 
//{\text{Definition \ref{def_risk_sensitive_belief}}}
\end{multline}
\begin{multline}
=\sum_{s_h} 
{\bT_{h,a_h} (s_{h+1}|s_h)}
\frac{\bO_{h+1}}{\bO^\prime_{h+1}}(o_{h+1}|s_{h+1})
{\exp\gamma r_h(s_h,a_h)}
\\
\sum_{\tilde{s}_{1:h}}
\prod_{t=2}^h 
\frac{\bO_{t}}{\bO^\prime_t}(o_{t}|\tilde{s}_{t})
{ 1 \{\tilde{s}_h=s_h\}}
\exp\gamma\sum_{t=1}^{h-1}r_t(\tilde{s}_t, a_t)
{
\mu_1(\tilde{s}_1)
\prod_{t=1}^{h-1}
\mathbb{P}_{\mathcal{P}^\prime}^\pi(\tilde{s}_{t+1}|\tilde{s}_t, a_t)
}
\quad 
//{\text{Observation 
\ref{observation_expansion of conditional probability}}}
\end{multline}
Since $\mathcal{P}^\prime$ and $\mathcal{P}$
share the same transition matrix, 
rearranging terms by 
Fubini's theorem we have
\begin{equation}
\begin{aligned}
&RHS
\\=&
\sum_{\tilde{s}_{1:h-1}}
{
\left(
\sum_{{s}_h}\mathbb{P}_{\mathcal{P}^\prime}^\pi(s_{h+1}|s_h,a_h)
\exp\gamma r_h(s_h,a_h)
\sum_{\tilde{s}_h}\mathds{1} \{\tilde{s}_h=s_h\}
\mathbb{P}_{\mathcal{P}^\prime}^\pi(\tilde{s}_h|\tilde{s}_{h-1},a_h)
\right)
}
\cdot
\left(\mu_1(\tilde{s}_1)
\prod_{t=1}^{h-2}
\mathbb{P}_{\mathcal{P}^\prime}^\pi(\tilde{s}_{t+1}|\tilde{s}_t, a_t)
\right)
\\
&\cdot
\left(\frac{\bO_{h+1}}{\bO^\prime_{h+1}}(o_{h+1}|s_{h+1})
\prod_{t=2}^h
\frac{\bO_{t}}{\bO^\prime_t}(o_{t}|\tilde{s}_{t})
\right)
\cdot\left(
\exp\gamma r_h(s_h,a_h)
\exp\gamma \sum_{t=1}^{h-1}r_t(\tilde{s}_t, a_t)
\right)
\\=&
\sum_{\tilde{s}_{1:h-1}}
{
\left(
\sum_{{s}_h}\mathbb{P}_{\mathcal{P}^\prime}^\pi(s_{h+1}|s_h,a_h)
\cdot
\mathbb{P}_{\mathcal{P}^\prime}^\pi(s_h|\tilde{s}_{h-1},a_h)
\right)
}
\cdot
\left(\mu_1(\tilde{s}_1)
\prod_{t=1}^{h-2}
\mathbb{P}_{\mathcal{P}^\prime}^\pi(\tilde{s}_{t+1}|\tilde{s}_t, a_t)
\right)
\\
&\cdot
\left(\frac{\bO_{h+1}}{\bO^\prime_{h+1}}(o_{h+1}|s_{h+1})
\prod_{t=2}^h 
\frac{\bO_{t}}{\bO^\prime_t}(o_{t}|\tilde{s}_{t})
\right)
\cdot\left(
\exp\gamma r_h(s_h,a_h)
\exp\gamma \sum_{t=1}^{h-1}r_t(\tilde{s}_t, a_t)
\right)
\end{aligned}
\end{equation}
Relabel $s_h$ as 
${\tilde{s}_h}$
and invoke the equality $f(s_{h+1})\equiv \sum_{{\tilde{s}_{h+1}}}f({\tilde{s}_{h+1}})\mathds{1}\{{\tilde{s}_{h+1}}=s_{h+1}\}$, we conclude
\begin{equation*}
\begin{aligned}
&RHS
\\=&
{
\sum_{{\tilde{s}_{h+1}}}
\mathds{1}\{{\tilde{s}_{h+1}}=s_{h+1}\}}
\left[
\sum_{\tilde{s}_{1:h-1}}
\sum_{{\tilde{s}_h}}
\left(
\mathbb{P}_{\mathcal{P}^\prime}^\pi({\tilde{s}_{h+1}}|{\tilde{s}_h},a_h)
\cdot
\mathbb{P}_{\mathcal{P}^\prime}^\pi({\tilde{s}_h}|\tilde{s}_{h-1},a_h)
\right)
\cdot
\left(\mu_1(\tilde{s}_1)
\prod_{t=1}^{h-2}
\mathbb{P}_{\mathcal{P}^\prime}^\pi(\tilde{s}_{t+1}|\tilde{s}_t, a_t)
\right)
\right]
\\&
\cdot
\left(\frac{\bO_{h+1}}{\bO^\prime_{h+1}}(o_{h+1}|{\tilde{s}_{h+1}})
\prod_{t=2}^h 
\frac{\bO_{t}}{\bO^\prime_t}(o_{t}|\tilde{s}_{t})
\right)
\cdot\left(
\exp\gamma r_h({\tilde{s}_h},a_h)
\exp\gamma \sum_{t=1}^{h-1}r_t(\tilde{s}_t, a_t)
\right)
\\
=&
\sum_{\tilde{s}_{1:h+1}}
\mathds{1}\{{\tilde{s}_{h+1}}=s_{h+1}\}
\cdot
\left[
{
\mu_1(\tilde{s}_1)
\prod_{t=1}^{h}
\mathbb{P}_{\mathcal{P}^\prime}^\pi(\tilde{s}_{t+1}|\tilde{s}_t, a_t)
}
\right]
\cdot
\left(
\prod_{t=2}^{h+1}
\frac{\bO_{t}}{\bO^\prime_t}(o_{t}|\tilde{s}_{t})
\right)
\cdot
\left(
\exp\gamma \sum_{t=1}^{h}r_t(\tilde{s}_t, a_t)
\right)
\\
=&
\sum_{\tilde{s}_{1:h+1}}
{
\mathbb{P}_{\mathcal{P}^\prime}^\pi(\tilde{s}_{1:h+1}|f_{h+1})
}
\cdot
\left[
\mathds{1}\{{\tilde{s}_{h+1}}=s_{h+1}\}
\cdot
\mathbf{D}_{h+1}(o_{2:t};\tilde{s}_{1:t})
\cdot
\exp\gamma \sum_{t=1}^{h}r_t(\tilde{s}_t, a_t)
\right]
\\
=&
\mathbb{E}_{\mathcal{P}^\prime}^\pi
\left[
\mathbf{1}
\{\S_{h+1}=s_{h+1}\}
\cdot
\mathbf{D}_{h+1}(\O_{2:t};\S_{1:t})
\cdot
\exp\gamma \sum_{t=1}^{h}r_t(\S_t, \A_t)
\bigg|
\mathbf{F}_{h+1}=f_{h+1}
\right]
\\=&
LHS {\quad //\text{Definition \ref{def_risk_sensitive_belief}}}
\end{aligned}
\end{equation*}
\end{proof}
\begin{remark}
    When we specify $\bO^\prime_h(\cdot|s_h)$ as the uniform distribution $\text{Unif}(\sO)$, 
    we can write the update formula as
    \begin{equation}
        \begin{aligned}
        \label{update_rule_matrix_representation}
        \vec{\sigma}_{h+1,f_{h+1}}
        =
        \mathsf{U}_{a_h,o_{h+1}}\ 
        \vec{\sigma}_{h,f_h}
        =
        \abs{\sO}
        \diag{\bO_{h+1}(o_{h+1}|\cdot)}
        \bT_{h,a_h}
        \diag{\exp \gamma r_h(\cdot,a_h)}
        \vec{\sigma}_{h,f_h}
        \end{aligned}
    \end{equation}
\end{remark}

\begin{remark} (Initial belief)
\label{remark_initi_belief}
There are multiple ways to define our initial belief according to Definition \ref{def_beta_vector}, since $\sum_{h=1}^0$ is ill-defined in nature. 
The optimization problem also poses no restriction on $\sigma_1$, since we present the optimization objective by $\sigma_{H+1}$ instead. However, since we wish to represent $\sigma_{H+1}$ by its predecessors, an appropriate definition of $\sigma_1$ should be compatible with our update rule, so that we can derive $\sigma_1$ from 
the $\sigma_2$ by Eq.~\eqref{update_belief}. A simple calculation will show that such constraint impels $\sigma_1(s_1)=\mu_1(s_1)$.
\end{remark}

\begin{remark}
For simplicity, we have presented the theorem in the tabular case.
With slight modifications, similar result holds in the continuous case. However, when the spaces are infinite, the evolution operator $\mathsf{U}^\star$ may not have a matrix representation as presented in Eq.~\eqref{update_rule_matrix_representation}.
\end{remark}

\subsection{Conjugate Beliefs}
\label{section_Value functions and Bellman equations}
In the analysis of Eq.~\eqref{why_beliefs_condition_expect}, 
we can express the objective by the terminal belief $\{\vec{\sigma}_{H+1}\}$
\begin{equation}
\begin{aligned}
\label{summary_objective_by_belief}
J(\pi;\mathcal{P})=&
U^{-1}
\mathbb{E}_{\mathcal{P}^\prime}^\pi
\left[
\langle
\vec{\sigma}_{H+1, \mathbf{F}_{H+1}}, 
\vec{1}_{S}
\rangle
\right]
\end{aligned}
\end{equation}
It is reasonable to express $\sigma_{H+1}$ by 
its predecessors, which posses simpler structures. However, as Theorem \ref{theorem_evolution_of_risk_sensitive_belief}
suggests, $\vec{\sigma}_t$ evolves forward in time, which 
hinders us from writing $\sigma_{H+1}$ in terms of $\{\sigma_{t}\}_{t \leq H}$.
\begin{equation}
\begin{aligned}
\relax [\vec{\sigma}_{1}]_{s_1} :=&
\vec{\mu}_1(s_1), \quad  
\vec{\sigma}_{h+1, f_{h+1}}=\mathsf{U}_{a_h,o_{h+1}}   \vec{\sigma}_{h,f_h}, \forall h\in [H]\\
\end{aligned}
\end{equation}
To bridge the gap in the the direction of evolution, we would like to introduce another process $\vec{\nu}_t$ that evolves backward in time, so that it is straightforward to express her initial state $\vec{\nu}_{H+1}$ by the predecessors, according to the new update rule.

We find it convenient to first introduce several concepts that describe how a stochastic process evolves backward in the episodic setting.
\begin{definition}(Backward History)
\label{def_backward_history}
\begin{equation}
\begin{aligned}
&
\boldsymbol{\bar{F}}_{H+1}=\emptyset
\qquad 
\forall h=H,H-1,\ldots,1
\ 
\boldsymbol{\bar{F}}_h=(\A_h,\O_{h+1},\ldots,\A_{H-1},\O_{H},\A_H),
\quad 
\boldsymbol{\bar{F}}_0=\tau_H\\
&
\forall h \in [H]:\quad 
\boldsymbol{\bar{F}}_{h}=(\A_h,\O_{h+1},\boldsymbol{\bar{F}}_{h+1})
\qquad 
\sigma(\boldsymbol{\bar{F}}_0)
\supset 
\sigma(\boldsymbol{\bar{F}}_{1}) \ldots 
\supset  \sigma(\boldsymbol{\bar{F}}_{H+1})\\
\end{aligned}
\end{equation}
\end{definition}
Definition \ref{def_backward_history} implies that $\{\boldsymbol{\bar{F}}_t\}_{t\geq 0}$ and $\{\boldsymbol{F}_t\}_{t\geq 0}$
are complementary at all times. 
\begin{observation}(Complementary relation)
\label{def_forward_backward_filters}
$\forall h=H+1, H,\cdots,0,\quad (\boldsymbol{F}_h,\boldsymbol{\bar{F}}_h)=\boldsymbol{\tau}_H $
\end{observation}
Now we are ready to define the backward process $\{\vec{\nu}_t\}_{t\geq 0}.$, whose update operator will be the Hilbert-adjoint operator
\footnote{For a rigorous definition please refer to Section \ref{Adjoint operator}.}
of that of $\vec{\sigma}_t$.

\begin{definition}(Conjugate Beliefs, Definition 2.8 in \cite{Baras&Elliott94})
\label{def_adjoint_belief}
\begin{equation}
\begin{aligned}
   \vec{\nu}_{H+1}(\cdot)
        	:\equiv&
        \vec{1}_{S}\\
        \vec{\nu}_{h,\bar{f}_{h}} :=&
        \mathsf{U}_{a_h,o_{h+1}}^\top
        \vec{\nu}_{h+1, \bar{f}_{h+1}}, 
        \text{for all }\  
        h=H,H-1,\cdots,1, \ \ 
        \bar{f}_h=(a_h,o_{h+1},\bar{\f}_{h+1})
        \in \sA \times \sO \times \bar{\mathscr{F}}_h
\end{aligned}
\end{equation}  
\end{definition}
\begin{remark}
    In tabular case when we select the 
    emission measure of 
    the reference model as uniform distribution, 
    we have:
    \begin{equation}
        \begin{aligned}
            [\vec{\nu}_{h,\bar{f}_{h}}]_{s_{h}}=&
            \frac{e^{\gamma r_{h}(s_h,a_h)}}{\mathbb{O}^\prime(o_{h+1})}
            \sum_{s_{h+1}\in \mathscr{S}}
            \mathbb{T}_{h,a_h}(s_{h+1}|s_{h})
            \mathbb{O}_{h+1}(o_{h+1}|s_{h+1})
            [\vec{\nu}_{h+1, \bar{f}_{h+1}}]_{s_{h+1}}
        \end{aligned}
    \end{equation}
\end{remark}
We have carefully designed the update rule of the conjugate belief: the complementary relation \eqref{def_forward_backward_filters} immediately implies the inner product between $\vec{\sigma}_{t}$ and $\vec{\nu}_t$ does not change with time. This result helps to excavate the dynamic programming structure hidden within the optimization objective.
\begin{observation}(Conjugate evolution, Eq.(2.9) in \cite{Baras&Elliott94})
For all $h\in [H+1]$, 
\label{observation_conjugate_evolve}
\begin{equation}
\begin{aligned}
\label{Eq_conjugate_evolve}
\langle \vec{\sigma}_{H+1,\f_{H+1}}, \vec{1} \rangle
=
\langle \vec{\sigma}_{h,\f_h}, \vec{\nu}_{h,\bar{\f}_h}\rangle 
=\cdots
\langle 
\vec{\mu}_1(s_1), \vec{\nu}_{1,\bar{\f}_1} 
\rangle
\end{aligned}
\end{equation}
\end{observation}

\begin{proof}
    By the definition of 
    adjoint operator,
    $$
    \langle \vec{\sigma}_{t,\f_t}, \vec{\nu}_{t,\bar{f}_t} \rangle 
    =
    \langle \mathsf{U}_{a_{t-1},o_{t}}
    \vec{\sigma}_{t-1,\f_{t-1}}, \vec{\nu}_{t,\bar{f}_t} \rangle
    =
    \langle \vec{\sigma}_{t-1,\f_{t-1}}, 
    \mathsf{U}^\top_{a_{t-1},o_{t}} 
    \vec{\nu}_{t,\bar{f}_t} \rangle
    =
    \langle \vec{\sigma}_{t-1,\f_{t-1}}, 
    \vec{\nu}_{t-1, \bar{f}_{t-1}} \rangle
    $$
\end{proof}
Bringing Eq.~\eqref{Eq_conjugate_evolve} back to
\eqref{summary_objective_by_belief} we immediately conclude
that for all $h\in [H]$, 
\begin{equation}
\begin{aligned}
\label{objective_time_evolution}
J(\pi;\mathcal{P})
=
U^{-1}
\mathbb{E}_{\mathcal{P}^\prime}^\pi
\left[
U
\langle 
\vec{\sigma}_{{H+1},\mathbf{{F}_{{H+1}}}},
\vec{1})
\rangle 
\right]
=
U^{-1}
\mathbb{E}_{\mathcal{P}^\prime}^\pi
\left[
U
\langle 
\vec{\sigma}_{{h}, \mathbf{{F}_{{h}}}},
\vec{\nu}_{{h}, \mathbf{\bar{F}_{{h}}}}
\rangle 
\right]
=
U^{-1}
\mathbb{E}_{\mathcal{P}^\prime}^\pi
\left[
U
\langle
\vec{\sigma}_{{1}, \mathbf{F}_1}
,
\vec{\nu}_{{1}, \mathbf{\bar{F}_{{1}}}}
\rangle 
\right]
\end{aligned}
\end{equation}

\subsection{Value functions, Q-functions and Bellman equations}
\label{section_Value functions, Q-functions and Bellman equations}
In this section we derive the value functions and Bellman equations for general tabular POMDP models using arbitrary utility risk measure. 
Our derivation and Bellman equations are different from previous works including \cite{Baras&Elliott94,Cavazos2005,baauerle2017partially}. 

From the subscripts in Eq.~\eqref{objective_time_evolution}
we can witness the trace of time evolution hidden within 
$J(\pi;\mathcal{P})$. To expose the dynamic programming structure more explicitly, we will follow the rationale behind the design of belief states, 
utilizing the iterated expectation formula to define a series of intermediate variables that dissect the information at each step. These variables will be called the {partially observable risk-sensitive value functions}.
For all $t\in[H+1]$, 
\begin{equation*}
\begin{aligned}
J(\pi;\mathcal{P})
:=&
U^{-1}
\mathbb{E}_{\mathcal{P}}^\pi
\left[
U
\sum_{t=1}^H r_t(\S_t,\A_t)
\right]
\\=&
U^{-1}
\mathbb{E}_{\mathcal{P}^\prime}^\pi
[
\langle{
\vec{\sigma}_{H+1, \mathbf{F}_{H+1}}, 
\vec{\nu}_{H+1, \mathbf{\bar{F}}_{H+1}}
\rangle
}
]
\quad //{\text{Belief representation by Eq.~\eqref{why_beliefs_condition_expect}}}
\\=&
U^{-1}
\mathbb{E}_{\mathcal{P}^\prime}^\pi
[
\langle{
\vec{\sigma}_{t, \mathbf{F}_{t}}, 
\vec{\nu}_{t, {\mathbf{\bar{F}}}_{t}}
\rangle
}
]
\quad //{\text{Conjugate evolution property proved in Eq.~\eqref{Eq_conjugate_evolve}}}
\\=&
U^{-1}
\mathbb{E}_{\mathcal{P}^\prime}^\pi
\left[
U
{
U^{-1}
\mathbb{E}_{\mathcal{P}^\prime}^\pi
\left[
\langle{\vec{\sigma}_{t, \mathbf{F}_{t}}, 
\vec{\nu}_{t, {\mathbf{\bar{F}}}_{t}}
\rangle
}
\bigg|\mathbf{F}_{t}
\right]}
\right]
\\
:=&
U^{-1}
\mathbb{E}_{\mathcal{P}^\prime}^\pi
\left[
U
{{\mathsf{V}^\pi_{t}
(\boldsymbol{F}_{t})}}
\right]
\end{aligned}
\end{equation*}
Now we formally 
define the family of value functions in our problem. 

\begin{definition}(Partially observable risk-sensitive value functions, Definition 2.13 in \cite{Baras&Elliott94})
\label{def_familiy_value}
\begin{equation}
\begin{aligned}
\mathsf{V}_{H+1}^{\pi}(f_{H+1})
:=&
U^{-1}
\mathbb{E}^\pi_{\mathcal{P}^\prime}
\left[\langle{\vec{\sigma}_{H+1, \mathbf{F}_{H+1}}, 
\vec{1}_{S}
\rangle 
\bigg|\mathbf{F}_{H+1}=f_{H+1}}\right]
=
U^{-1}\norm{\vec{\sigma}_{h+1, f_{h+1}}}_1
\\
\mathsf{V}_h^{\pi}(f_h)
:=&
U^{-1}
\mathbb{E}^\pi_{\mathcal{P}^\prime}
\left[
\langle{\vec{\sigma}_{h, \mathbf{F}_h}, 
\vec{\nu}_{h, \mathbf{\bar{F}}_h}
\rangle  
\bigg|\mathbf{F}_h=f_h}\right]
\quad 
\forall 2\leq h\leq H, f_h\in \sF{h}\\
\mathsf{V}_{1}^{\pi}(f_1):=&
U^{-1}\mathbb{E}_{\mathcal{P}^\prime}^\pi \left[
\langle \vec{\mu}_1, \vec{\nu}_{1,\bar{\mathbf{F}}_1} \rangle 
\right]
\end{aligned}
\end{equation}    
\end{definition}

\begin{remark}
\label{remark_express_regret_by_value_functions}
The objective $J(\pi;\mathcal{P})$
can be expressed by the value function:
$$
J(\pi;\mathcal{P})=U^{-1}\mathbb{E}_{\mathcal{P}^\prime}\left[U\mathsf{V}_1^\pi\right]=
U^{-1}
\mathbb{E}_{\mathcal{P}^\prime}
\langle \vec{\sigma}_1, \vec{\nu}_1\rangle
$$
\end{remark}

In reinforcement learning, we are also curious about how the action to take might affect future rewards. We can separate the stochasticity within $\pi$ and $\mathcal{P}$ by the total expectation formula: 
\begin{equation*}
\begin{aligned}
\mathsf{V}_h^{\pi}(f_h)
:=&
U^{-1}
\mathbb{E}^\pi_{\mathcal{P}^\prime}
\left[\langle{\vec{\sigma}_h(\boldsymbol{F_h}), 
\vec{\nu}_h({\boldsymbol{\bar{F}}}_h)\rangle  
\bigg|\boldsymbol{F_h}=f_h}\right]
\quad 
//{\text{Definition \ref{def_familiy_value}}}
\\=&
U^{-1}
\mathbb{E}^\pi
\left[
\mathbb{E}_{\mathcal{P}^\prime}
\left[
{
\mathbb{E}^\pi_{\mathcal{P}^\prime}
\left[\langle{
\vec{\sigma}_{h+1, \boldsymbol{F_{h+1}}}, 
\vec{\nu}_{{h+1}, {\mathbf{\bar{F}}}_{{h+1}}}\rangle  
\bigg|{\boldsymbol{F}}_{h},\mathbf{A}_h,\mathbf{O}_{h+1}}\right]
\bigg|
{\boldsymbol{F}}_{h},\mathbf{A}_h
}
\right]
\big|
\boldsymbol{F_h}=f_h
\right]
\\
=&
U^{-1}
\mathbb{E}^\pi
\left[
{
\mathbb{E}_{\mathcal{P}^\prime}
\left[
{U^{-1}\mathsf{V}_{h+1}^\pi(\boldsymbol{F_{h+1}})}
\big|
{\boldsymbol{F}}_{h},\mathbf{A}_h
\right]
}
\big|\boldsymbol{F_h}=f_h
\right]
\\:=&
U^{-1}
\mathbb{E}^\pi
\left[
U
{{U^{-1}
\mathbb{E}_{\mathcal{P}^\prime}
\left[
U^{-1}\mathsf{V}_{h+1}^\pi
(\boldsymbol{F_{h+1}})
\big|
{\boldsymbol{F}}_{h},\mathbf{A}_h
\right]
}}
\big|\boldsymbol{F_h}=f_h
\right]
\\=&
U^{-1}
\mathbb{E}^\pi
\left[
U
{{U^{-1}
\mathbb{E}_{\mathcal{P}^\prime}
\left[
{U^{-1}\mathsf{V}_{h+1}^\pi(\boldsymbol{F_{h+1}})}
\big|
{f}_{h},\mathbf{A}_h
\right]
}}
\big|\boldsymbol{F_h}=f_h
\right]
\\:=&
U^{-1}\mathbb{E}_{\mathbf{A_h}\sim \pi_h(\cdot|f_h)}
U
{
\mathsf{Q}_{h}^\pi(f_h,\mathbf{A}_h)
}
\end{aligned}
\end{equation*}
Now we formally introduce the Q functions in our problem:
\begin{definition}(Partially observable risk-sensitive Q-functions) For all $(h,s_h,a_h) \in [H]\times \mathscr{S\times A}$, 
\begin{equation*}
\begin{aligned}
\mathsf{Q}^\pi_{h}(f_{h},a_h)
:=&
U^{-1}
\mathbb{E}_{\mathcal{P}^\prime}
\left[
U
\mathsf{V}_{h+1}^\pi(\mathbf{F}_{h+1}=(f_h,a_h,\mathbf{O}_{h+1}))
\big|
f_h,a_h
\right]
\end{aligned}
\end{equation*}
\end{definition}
\begin{remark}
\label{remark_Q_value_relation}
We can also represent the Q-function by risk beliefs. Indeed,
\begin{equation}
\begin{aligned}
\label{beta_vector_representation_of_Q_function}
\mathsf{Q}_h^\pi(f_h,a_h)
=&
U^{-1}
\mathbb{E}
_{\mathbf{O}_{h+1}\sim 
\mathbb{P}_{\mathcal{P^\prime}}(\cdot|{f}_{h},a_h)}
\left[
U\mathsf{V}_{h+1}^{\pi}(f_h,a_h,\boldsymbol{O}_{h+1})
\right]
\\=&
U^{-1}\mathbb{E}_{\mathcal{P}^\prime}
\left[\langle{\vec{\sigma}_{H+1, \mathbf{F}_{H+1}}, 
\vec{\nu}_{H+1, \mathbf{\bar{F}}_{H+1}}\rangle  
\bigg|f_{h},a_h}\right]
=
U^{-1}\mathbb{E}_{\mathcal{P}^\prime}
\left[\langle{\vec{\sigma}_{h, \mathbf{F}_{h}}, 
\vec{\nu}_{h, {\mathbf{\bar{F}}}_{h}}\rangle  
\bigg|f_{h},a_h}\right]
\end{aligned}
\end{equation}
where the last step is due to Eq.~\eqref{Eq_conjugate_evolve}. 
We can also use \eqref{beta_vector_representation_of_Q_function} as an alternative definition of the Q function.
\end{remark}
The relationship between the value function and the Q function is summarized as the Bellman equations:
\begin{corollary}
(Bellman equations for risk-sensitive POMDP)
\label{corollary_Bellman_equation_PORL}
\begin{equation}
\begin{aligned}\label{blmaneqs_raw}
\mathsf{V}_{H+1}^\pi=&
U^{-1}
\norm{\vec{\sigma}_{h+1, f_{h+1}}}_1
\\
\forall h=H:1, \  \ 
\mathsf{Q}_h^\pi(f_h,a_h)
=&
U^{-1}
\mathbb{E}
_{\mathbf{O}_{h+1}\sim \mathbb{O}^\prime(\cdot)}
\left[
U
{\mathsf{V}_{h+1}^\pi
(f_h,a_h,\mathbf{O}_{h+1})}
\right]
\\
\mathsf{V}_h^{\pi}(f_h)
=&
\mathbb{E}_{\mathbf{A_h}\sim \pi_h(\cdot|f_h)}
{
\mathsf{Q}_{h}^\pi(f_h,\mathbf{A}_h)}
\\
J(\pi;\mathcal{P})=&
U^{-1}\mathbb{E}_{\mathcal{P}^\prime}
\left[
U\mathsf{V}_1^{\pi}
\right]
\end{aligned}
\end{equation}
\end{corollary}
The expectation in the second equation should have been taken with respect to 
$\mathbb{P}_{\mathcal{P^\prime}}^\pi(\cdot|{f}_{h},a_h)$. However, since the observations are disjoint from the POMDP in the reference model, we obtain a simpler expression in Eq.~\eqref{blmaneqs_raw}. 

\subsection{Optimal Policy}
\label{section_existence_optimal_policy}
\paragraph{Optimal substructure}
Since our utility risk measure is increasing, to optimize our objective $U^{-1} \mathbb{E}U(\sum_{h=1}^H r_h)$ is equivalent to maximize $\mathbb{E}U(\sum_{h=1}^H r_h)$. According to Theorem 3.3 of \cite{baauerle2017partially}, the latter problem has a globally optimal policy, and so does the former. 
Moreover, the following theorem shows the existence of an optimal substructure in the planning problem of risk-sensitive POMDP. This result justifies the use of dynamic programming equations in our algorithm.
\begin{theorem}
(Bellman optimality equations, extended from theorem 2.5 of \cite{Baras&Elliott94})
\label{theorem_Bellman Optimality equations for risk-sensitive POMDP}
When the utility risk function is strictly increasing and the 
referential emission matrix $\bO^\prime_t(\cdot|s)$
is irrelevant with the history\ $\mathbf{F}_h$, 
the locally optimized policy will bring globally optimized value. Formally, the locally optimal values defined by
\begin{equation*}
\begin{aligned}
\mathsf{V}_1^{\star}:=\max_{\pi} \mathsf{V}_1^{\pi}, 
\qquad 
\mathsf{V}_h^{\star}(f_h):= \max_{\pi} \mathsf{V}_h^{\pi}(f_h)
, \quad 
\forall 2 \leq h \leq H+1
\end{aligned}
\end{equation*}
can be computed recursively:
\begin{equation}
\left\{
\begin{aligned}
V_{1}^\star=&
\Max{a_1\in \sA}
\ 
U^{-1}
\mathbb{E}_{\mathcal{P}^\prime}
\left[
U
V_{2}^\star(a_1,\mathbf{O}_{2})
\right]
\\
V_{h}^\star(f_h)=&
\Max{a_h\in \sA}
\ 
U^{-1}
\mathbb{E}_{\mathcal{P}^\prime}
\left[
U
V_{h+1}^\star(f_h,a_h,\mathbf{O}_{h+1})
\right]
, \quad \forall h= H:2\\
V_{H+1}^\star(f_{H+1})=&
U^{-1}\norm{\vec{\sigma}_{h+1, f_{h+1}}}_1\\
\end{aligned}
\right.
\end{equation}
\end{theorem}
\begin{proof}
\begin{equation*}
\begin{aligned}
V_h^\star(f_h):=&
\Max{a_{h:H}}
U^{-1}
\mathbb{E}_{\mathcal{P}^\prime}
\left[
\langle \vec{\sigma}_{h, \mathbf{F}_h}, \vec{\nu}_{h,\bar{\mathbf{F}}_h} \rangle
\ | \ \mathbf{F}_h=f_h
\right]
//{\text{Definition of value functions in  \ref{def_familiy_value}}}
\\=&
U^{-1}
\Max{a_{h:H}} \mathbb{E}_{\mathcal{P}^\prime}
\left[
\langle \vec{\sigma}_{h, \mathbf{F}_h}, \vec{\nu}_{h,\bar{\mathbf{F}}_h} \rangle
\ | \ \mathbf{F}_h=f_h
\right]
\quad
//{\text{$U$ monotonically increases, so does $U^{-1}$.}}
\\=&
U^{-1}
\Max{a_h}\bigg\{
\Max{a_{h+1:H}}
\mathbb{E}_{\mathcal{P}^\prime}
\left[\ 
\mathbb{E}_{\mathcal{P}^\prime}
\left[
\langle 
\vec{\sigma}_{h, \mathbf{F}_h}, \vec{\nu}_{h,\bar{\mathbf{F}}_h} 
\rangle
\big|\mathbf{F}_{h+1}
\ \right]
\
\big|
\mathbf{F}_{h}=\f_h
\right]
\bigg\}
\\=&
U^{-1}
\Max{a_h}\bigg\{
\Max{a_{h+1:H}}
\mathbb{E}_{\mathcal{P}^\prime}
\left[\ 
\mathbb{E}_{\mathcal{P}^\prime}
\left[
\langle 
\vec{\sigma}_{H+1, \mathbf{F}_{H+1}}, 
\vec{\nu}_{h+1, \bar{\mathbf{F}}_{h+1}} \rangle
\big|\mathbf{F}_{h+1}
\ \right]
\big|
\mathbf{F}_{h}=\f_h
\right]
\bigg\}
\ //{\text{Observation \ref{observation_conjugate_evolve}}}
\\=&
U^{-1}
\Max{a_h}\bigg\{
{\Max{a_{h+1:H}}
\mathbb{E}_{\mathcal{P}^\prime}}
\left[
\mathbb{E}_{\mathcal{P}^\prime}
\left[
\langle 
\vec{\sigma}_{h+1, \mathbf{F}_{h},a_h,\mathbf{O}_{h+1}},
\vec{\nu}_{h+1, a_{h+1},\mathbf{O}_{h+2},\cdots, 
\mathbf{O}_H,a_H} \rangle
\big|\mathbf{F}_h,a_h,\mathbf{O}_{h+1}
\ \right]
\big|
\mathbf{F}_{h}=\f_h
\right]
\bigg\}
\\
=&
U^{-1}
\Max{a_h}\bigg\{
{\mathbb{E}_{\mathcal{P}^\prime}}
\left[
{\Max{a_{h+1:H}}}
\mathbb{E}_{\mathcal{P}^\prime}
\left[
\langle 
\vec{\sigma}_{h+1,\mathbf{F}_{h},a_h,\mathbf{O}_{h+1}}, 
\vec{\nu}_{h+1, a_{h+1},\mathbf{O}_{h+2},\cdots, 
\mathbf{O}_H,a_H} \rangle
\big|\mathbf{F}_h,a_h,\mathbf{O}_{h+1}
\ \right]
\big|
\mathbf{F}_{h}=\f_h
\right]
\bigg\}
\\
=&
U^{-1}
\Max{a_h}
\mathbb{E}_{\mathcal{P}^\prime}
[UV^\star_{h+1}(\mathbf{F}_{h+1}=(\mathbf{F}_h,a_h,\mathbf{O}_{h+1})
\big|\mathbf{F}_{h}=f_h]
\\=&
U^{-1}
\Max{a_h}
\mathbb{E}_{\mathcal{P}^\prime}
[UV^\star_{h+1}
(f_h,a_h,\mathbf{O}_{h+1})
] \text{//Section \ref{section_change_of_measure}}
\end{aligned}
\end{equation*}
\end{proof}
The proof is inspired by \cite{Baras&Elliott94,baauerle2017partially} 
and we have generalized their result beyond Gaussian transition matrices and the entropic risk. We will present the proof in the tabular case. Regularity conditions will be needed in the sixth step when we generalize the theorem to the continuous setting.

Our proof also yields the following corollary, providing justification for selecting greedy policies in our algorithm:
\begin{corollary}(Adapted from Theorem 3.3 of \cite{baauerle2017partially})
    There always exists an optimal policy for a risk-sensitive tabular POMDP using utility risk measure, which is deterministic and history-dependent.
\end{corollary}

\subsection{Beta Vectors}
\label{section_beta_vectors}
Inspired by the alpha vector representation method in the study of POMDP(\cite{MonahanSurvey1982}), we will exploit the structure of
risk-sensitive value functions and represent them in a simple form.
Recall that in Definition \ref{def_familiy_value}, the value functions are specified as inner products. Since $\sigma(\mathbf{F}_h)$ is already determined by the condition on $f_h$, we can write the value function as 
$\mathsf{V}_{h}^\pi(f_h)
=
U^{-1}
\left\langle 
\vec{\sigma}_{h,f_h}\ , \ 
{
\mathbb{E}_{\mathcal{P}^\prime}^{\pi}
\left[
\vec{\nu}_{h,\bar{\mathbf{F}}_h}
\mid
\mathbf{F}_h=f_h
\right]
}
\right\rangle$. 
, which motivates us to introduce the concept of "beta" vector. 
\begin{definition}{(Beta vector)}
The beta vector of a risk-sensitive POMDP model $\mathcal{P}=(\mu_1, \mathbb{T}, \mathbb{O})$ under policy $\pi$
is a series of random vectors in $\R^{S}$, which are specified as
\label{def_beta_vector}
    \begin{equation}
        \begin{aligned}
        \vec{\beta}^\pi_{H+1,\boldsymbol{F_{H+1}}}:=&\vec{1}_{S}
        \\
        \vec{\beta}^\pi_{h,\boldsymbol{F_h}}:=&
        \mathbb{E}_{\mathcal{P}^\prime}^\pi
        \left[\vec{\nu}_{h,\boldsymbol{\bar{F}_h}}|\boldsymbol{F_h}\right]
        ,\quad \forall 2 \leq h \leq H.
        \\
        \vec{\beta}^\pi_1 :=&
        \mathbb{E}_{\mathcal{P}^\prime}^\pi
        \left[\vec{\nu}_{1,\boldsymbol{F_1}}\right]
        \end{aligned}
    \end{equation}
\end{definition}
where $\nu_t$ is the conjugate belief defined in \ref{def_adjoint_belief}.

Next, we will try to obtain the evolution law of the beta vector from the way $\vec{\nu}_h$ is updated:
\begin{equation}
\begin{aligned}
\label{update_adjoint_belief}
[\vec{\nu}_{h, {{\bar{f}_{h}}}}]_{s_{h}}
=\frac{1}{\mathbb{O}_1^\prime(o_1)}
\cdot
e^{\gamma r_h(s_h,{{a_h}})}
\cdot
\sum_{s_{h+1}\in \mathscr{S}}
\mathbb{T}_{h,{{a_h}}}(s_{h+1}|s_h)
\mathbb{O}_{h+1}({{o_{h+1}}}|s_{h+1})
[\vec{\nu}_{h+1, {{\bar{f}_{h+1}}}}]_{s_{h+1}}
\end{aligned}
\end{equation}
Under the reference model $\mathcal{P}^\prime$, we can compute the probability of witnessing an observable trajectory $\bar{f}_h$ given previous history $f_h$ by the following equation:
\begin{equation}
\begin{aligned}
\label{conditional_prob}
\mathbb{P}_{\mathcal{P}^\prime}^{\pi}(\bar{f}_{h}|f_h)
=
\pi_h(a_h|f_h)
{\mathbb{O}_1^\prime(o_1)}
\mathbb{P}_{\mathcal{P}^\prime}^{\pi}(\bar{f}_{h+1}|f_h,a_h,o_{h+1})
=
\left({\mathbb{O}_1^\prime(o_1)}\right)^{H-h+1}
\cdot
\prod_{t=h}^{H}\pi_h(a_t|f_t)
\end{aligned}
\end{equation}
where $\bar{f}_h=(a_h,o_{h+1},\bar{f}_{h+1})=(a_h,o_{h+1},\cdots, a_{H-1},o_{H},a_H)$. Combining Eqs.\eqref{conditional_prob}, \eqref{update_adjoint_belief}, we obtain
\begin{equation}
\begin{aligned}
\label{condition_expect_express}
&\mathbb{E}_{\mathcal{P}^\prime}^{\pi}
[\nu_{h,\mathbf{\bar{F}_h}}|f_h]
\\=&
\sum_{{a_h}}
\pi_h({{a_h}}|f_h)
\sum_{{o_{h+1}}}
{\mathbb{O}_1^\prime(o_1)}
\sum_{\bar{f}_{h+1}}
\mathbb{P}_{\mathcal{P}^\prime}^{\pi}(\bar{f}_{h+1}|f_h,a_h,o_{h+1})
\\
&\qquad 
\frac{1}{\mathbb{O}_1^\prime(o_1)}
\cdot
e^{\gamma r_h(s_h,{{a_h}})}
\cdot
\sum_{s_{h+1}\in \mathscr{S}}
\mathbb{T}_{h,{{a_h}}}(s_{h+1}|s_h)
\mathbb{O}_{h+1}({{o_{h+1}}}|s_{h+1})
[\vec{\nu}_{h+1}({{\bar{f}_{h+1}}})]_{s_{h+1}}
\\=&
\sum_{a_h\in \mathscr{A}}\pi_h({{a_h}}|f_h)
\left\{
e^{\gamma r_h(s_h,{{a_h}})}
\sum_{s_{h+1}\in \mathscr{S}}
\mathbb{T}_{h,{{a_h}}}(s_{h+1}|s_h)
\sum_{o_{h+1}\in \mathscr{O}}
\mathbb{O}_{h+1}({{o_{h+1}}}|s_{h+1})
\mathbb{E}^\pi_{\mathcal{P}^\prime}
\left[\vec{\nu}_{h+1,\mathbf{\bar{F}}_{h+1}}|f_{h+1}\right]_{s_{h+1}}
\right\}
\end{aligned}
\end{equation}

The previous derivations leads to the fundamental theorem below, which forms the cornerstone of our subsequent analysis:
\begin{theorem}{($\beta$-vector representation of value functions)}
\label{theorem_beta_vector_representation of value functions}
For any $h \in \{1,\cdots,H+1\}$ and $f_h\in \sF{h}$, 
the value function defined in \ref{def_familiy_value}
can be expressed as the inner product
between the risk-sensitive beliefs 
defined in \ref{def_risk_sensitive_belief}
and the beta vector defined in 
\ref{def_beta_vector}:
\begin{equation*}
    \begin{aligned}
    \mathsf{V}_{1}^{\pi}
    =&
    U^{-1}
    \langle
    \vec{\sigma}_{1}
    \ , 
    \ 
    \vec{\beta}^\pi_{1}
    \rangle
    =
    {
    \frac{1}{\gamma}
    \ln
    \langle
    \vec{\sigma}_{1}
    \ , 
    \ 
    \vec{\beta}^\pi_{1}
    \rangle
    }
    \\
    \mathsf{V}_{h}^{\pi}(f_h)
    =&
    U^{-1}
    \langle
    \vec{\sigma}_{h,f_h}
    \ , 
    \ 
    {\vec{\beta}^\pi_{h,f_h}}
    \rangle
    =
    {
    \frac{1}{\gamma}
    \ln
    \langle
    \vec{\sigma}_{h,f_h}
    \ , 
    \ 
    \vec{\beta}^\pi_{h,f_h}
    \rangle
    }
    \end{aligned}
\end{equation*}
Moreover, the beta vectors evolve by the following rule:
For all $2 \leq h \leq H, \f_h\in \sF{h}, \ s_h\in \sS$
\begin{equation}
    \begin{aligned}
    \label{update_rule_real_beta}
    \vec{\beta}_{H+1,f_{H+1}}=&\vec{1}_{S}\\
    \left[\vec{\beta}_{h,f_h}\right]_{s_h} =&\ 
    \mathbb{E}_{{a_h}\sim \pi_h(\cdot|f_h)}
    \left[e^{\gamma r_h(s_h,{{a_h}})}
    \sum_{{{s_{h+1}}}}
    \mathbb{T}_{h,{{a_h}}}({{s_{h+1}}}|{s_h})
    \sum_{{{o_{h+1}}}}
    \mathbb{O}_{h+1}({{o_{h+1}}}|{{s_{h+1}}})
    \left[ 
    \vec{\beta}_{h+1,f_{h+1}=(f_h,{{a_h}},{{o_{h+1}}})}
    \right]_{{{s_{h+1}}}}
    \right]
    \end{aligned}
\end{equation}
\end{theorem}
We remind the reader that this theorem holds for randomized policies.

\begin{remark}\label{remark_alpha-beta-comparison}
When it is clear from the context, we omit the policy sign $^\pi$ for the beta vectors. We can also define the beta vectors for the empirical POMDP model by replacing the matrices $\mathbb{T}$ and $\mathbb{O}$ with their empirical approximations $\mathbb{\widehat{T}}^k$ and $\mathbb{\widehat{O}}^k$. The corresponding beta vector will be called the ``empirical beta vector'', which is denoted as $\widehat{\beta}^{k,\pi}_{h,f_h}$ and abbreviated as $\widehat{\beta}_{h}$.
\end{remark}

\begin{remark}(Motivation behind the beta vector)
According to \cite{MonahanSurvey1982,PBVI,Lee2023hindsight} the value function of a risk-neutral POMDP is the inner product of the risk-neutral belief $\vec{b}_h$ and another function named the ``alpha vector''. The concept of $\alpha-$vector has been widely adopted in the algorithm design of POMDPs \cite{PBVI}.
\begin{equation*}
\begin{aligned}
&V_{h}^\pi(\vec{b}_h;\f_h)=
\mathbb{E}_{a_h\sim \pi(\cdot|\f_h)}
\left[
r(\vec{b}_h,a_h;\f_h)
+
\sum_{o_{h+1} \in \sO}
{\eta_h(o_{h+1}|{\f_h,a_h})}
V^{\pi}_{h+1}
(\vec{b}_{h+1};\f_{h+1}=(\f_h,a_h,o_{h+1}))
\right]
\\
&V_h^\pi(\vec{b}_h;\f_h) = \langle \vec{b}_h, \alpha_{h,\f_h}^\pi \rangle
\\
&\begin{cases}
\{
\vec{\alpha}_{H+1,f_{H+1}}^\pi \equiv  0
\\
\left[
\vec{\alpha}_{h,\f_h}^\pi
\right]_{s_h}
=
\mathbb{E}_{a_h\sim \pi(\cdot|\f_h)}
        \left[
        r(s_h,a_h)+ 
        {
        \sum_{s_{h+1}}\bT_{h,a_h}(s_{h+1}|{s_h})
        \sum_{o_{h+1}}\bO_{h+1}(o_{h+1}|{s_{h+1}})
        }
        \left[\vec{\alpha}^\pi_{h+1,\f_{h+1}=(\f_{h},a_h,o_{h+1})} \right]_{s_{h+1}}
\right]
\end{cases}
    \end{aligned}
\end{equation*}
Though the update rule of the value function inevitably relies on the entire history,
the alpha vectors evolve in a {Markovian} way. This finding helps us obtain a
polynomial regret bound in risk-neutral POMDP using hindsight observations,
which is discovered by \cite{Lee2023hindsight}. To gain some insights, recall
that the pigeon-hole lemma \ref{lemma_pigeon-hole} suggests that the sum of
concentration errors is the polynomial function of the cardinality of the space
on which the transition probability is dependent.

\begin{equation*}
\begin{aligned}
\widehat{P}_h(s^\prime|{s,a}):=
 \frac{\sum_{\kappa=1}^k\mathds{1}
 \{\hat{s}_{h+1}^\kappa=s'\ , \ {\hat s_h^\kappa} =s\ , \ {\hat a_h^\kappa} =a\}}
 {\max\left\{1, \widehat{N}_h^{k+1}({s,a})\right\}}
\quad  
\sum_{k=1}^K \frac{1}{\sqrt{\max \{1,N_h^{k+1}({\widehat{s}_h^k, \widehat{a}_h^k})\}}} \lt 2\sqrt{K \cdot {SA}}
\end{aligned}
\end{equation*}

In the POMDP setting, if we refuse to represent the value function by the alpha vector, the upper bound in the right should be replaced with
$2\sqrt{K\cdot \abs{\sF{h}\times \sA}}=2O^{\frac{h}{2}}A^{\frac{h}{2}}$, which then brings a factor of $O^HA^H$ to our regret. However, under hindsight observability, since we can calculate the occurrences of the hidden states after each episode, we use the alpha vectors to represent the value function. Consequently, we can replace $O^HA^H$ with  $2\sqrt{S^2}$, which is polynomial in H again.
However, we cannot directly utilize the alpha vector in the risk-sensitive setting, as it may no longer preserve the Markov property. For the reader's convenience, we excerpt the critical steps in the proof of the evolution law \cite{MonahanSurvey1982}

\begin{equation*}
\begin{aligned}
{=}&
\langle \vec{b}_h(\f_h), r_h(\cdot,a_h)\rangle
+
{\r_{o_{h+1}\sim \eta_{h+1}(\cdot|\f_h,a_h)}}
\left[
\left\langle 
{
\frac{{\diag{\bO_{h+1}(o_{h+1}|\cdot)}\ \bT_{h,a_h}\ \vec{b}_h(\f_h)}}{\eta_{h+1}(o_{h+1}|\f_h,a_h)}
}
, \alpha_{h+1,\f_{h+1}}^\pi 
\right\rangle
\right]
\\
{=}&
\langle \vec{b}_h(\f_h), r_h(\cdot,a_h)\rangle
+
\left\langle
\sum_{o_{h+1}\in \sO}
{\bcancel{\eta_{h+1}(o_{h+1}|\f_h,a_h)}}
\frac{{\diag{\bO_{h+1}(o_{h+1}|\cdot)}\ \bT_{h,a_h}\ \vec{b}_h(\f_h)}}{{\bcancel{\eta_{h+1}(o_{h+1}|\f_h,a_h)}}}
, \alpha_{h+1,\f_{h+1}}^\pi 
\right\rangle
\end{aligned}
\end{equation*}

In the risk-neutral setting, the risk measure ${\rho}$ is represented by the expectation operator $\mathbb{E}$. The simple linear structure of $\mathbb{E}$ helps to cancel the partition function $\eta_{h+1}(\cdot|f_h,a_h)$ when the alpha vector updates, thus eliminating the culprit of the exponential dependency on H. 
However, the previous analysis becomes invalid when using a non-linear risk measure. That's why the beta vector is introduced in risk-sensitive POMDPs to mitigate the impact of historical dependencies under a non-linear criterion \cite{WATER&Willems81,Whittle91,Baras97}. 
\end{remark}

\begin{remark}(Comparison with the alpha vector in risk-neutral POMDP)
The value function can be expressed as an inner product in risk-sensitive settings, because the utility risk measure keeps a layer of $\mathbb{E}$, which still bears a linear structure. Both $\alpha$ and beta vectors evolve in a Markovian way, making it possible to control the regret under a polynomial upper bound given hindsight observability. However, the terminal value of $\vec{\beta}_{h,f_h}$ is $\vec{1}$
while that of $\vec{\alpha}$ is $\vec{0}$; The beta vector undergoes multiplicative updates, whereas the alpha vector renews through additive increments. If we brutally set $\gamma=0$ we cannot reduce the beta vector to the alpha vector, as $\gamma=0$ is a singular point of the risk measure, which will cause all the beta vectors to collapse to $\vec{1}$. We will introduce the correct way to degenerate our result to the classical setting in Section \ref{section_compare_other_setting}.
\end{remark}

\section{Detailed Algorithm Design}
In what follows, we present the algorithm introduced
in Section \ref{section_algo_design_main} in detail, along with several remarks additional to the discussion in Section \ref{section_algo_design_main}.
\begin{remark}(Computation issues)
    The BVVI algorithm, as well as other exact algorithms for POMDP, are inefficient in computation complexity\cite{1998exact_approx_alg_pomdp}, which is due to the inherent complexity of the POMDP model\cite{Tsitsiklis1987}. Using similar techniques as the point-based algorithms\cite{PBVI}, we can develop approximate solutions to our problem based on BVVI.
\end{remark}

\begin{remark}(Explanation of line \ref{alg_clip_and_sgn})
    The operation in line \ref{alg_clip_and_sgn} is equivalent to the assignment below:
    \begin{equation}
        \begin{aligned}
        \forall s\in \sS:\quad 
        \left[\widehat{\beta}_{h,f_h}^k\right]_{s}
        \gets 
        \begin{cases}
        e^{\gamma^+(H-h+1)}  &, \left[\widehat{\beta}_{h,f_h}^k\right]_{s} \geq e^{\gamma^+(H-h+1)}
        \\
        e^{\gamma^-(H-h+1)}  &, \left[\widehat{\beta}_{h,f_h}^k\right]_{s} \leq e^{\gamma^-(H-h+1)}
        \\
        \left[\widehat{\beta}_{h,f_h}^k\right]_{s} &,\text{else}
        \end{cases}
        \end{aligned}
    \end{equation}
\end{remark} 

\newpage

\begin{algorithm}[H]
    \caption{\text{Beta Vector Value Iteration(BVVI)}}    \label{alg:Beta Vector Value Iteration}
    \begin{algorithmic}[1]
        \STATE {\bfseries Input}
        risk sensitivity $\gamma \ne 0$, 
        confidence level $\delta\in(0,1)$, 
        episode number K, horizon length H.
        \STATE{ {\bfseries{Initialize}}
        $
        {\widehat{\mu}^k_1(\cdot), 
        \widehat{\bT}_{h,a}^1}(\cdot|s)
        \gets 
        \text{Unif}({\sS}), 
        \quad 
        {\widehat{\bO}_{h}^1}(\cdot|s)
        \gets 
        \text{Unif}({\sO}) 
        $
        for all 
        $(h,s,a,o,s')\in [H]\times \sS\times \sA\times \sO\times \sS.$
        }
        \FOR{$k = 1:K$}
            \STATE \algcom{//Planning}
            \STATE \algcom{//Forward belief propagation}
            \STATE{$
            \widehat{\sigma}^k_1
            \gets 
            \widehat{\mu}_1^k
            $}
            \label{alg_init_belief}
            \FOR{
            $h=1:H, \quad s_{h+1},s_h\in \sS, \quad \f_{h+1}=(\f_h,a_h,o_{h+1})\in \sF{h}\times \sA\times \sO$}        
                \STATE{
                \algcom{//Update risk belief by Eq.~\eqref{update_belief}}}
                \STATE{$
                \left[\widehat{\sigma}^k_{h+1, \f_{h+1}}\right]_{s_{h+1}}
                \gets 
                \sum_{s_{h}}
                \widehat{\bT}^k_{h,a_{h}}(s_{h+1}|s_h)
                \widehat{\bO}^k_{h+1}(o_{h+1}|s_{h+1})
                \cdot 
                {
                \left(
                \frac{e^{\gamma r_h(s_h,a_h)}}{\bO_{h+1}^\prime (o_{h+1}|s_{h+1})}
                \right)
                }
                \left[\widehat{\sigma}^k_{h,f_h}\right]_{s_h}
                $}
                \label{alg_belief_prop}
                \STATE
                    \algcom{\text{//Residue terms}}
                \STATE  $
                        {\mathsf{t}_h^k(s_h,a_h)}
                        \gets 
                        \min \left\{1, 3\sqrt{\frac{S H \ln KHSOA/\delta}{\widehat{N}_h^k(s_h,a_h) \vee 1}} \right\}
                        $
                \STATE  $
                        \hphantom{\mathsf{t}_h^k(s_h,a_h)}
                        \mathllap{\mathsf{o}_{h+1}^k(s_{h+1})}
                        \gets 
                        \min\left\{1,3\sqrt{\frac{O H \ln KHSOA/\delta}{ \widehat{N}_{h+1}^{k+1}(s_{h+1}) \vee 1}}\right\}
                        $
                \STATE \algcom{\text{//Prepare bonus by Eq.~\eqref{bonus_design}}}
                \STATE{
                    $
                \hphantom{\widehat{\sigma}^k_{h+1, \f_{h+1}}}
                \mathllap{{\mathsf{b}_h^k(s_h,a_h)}}
                \gets
                \abs{e^{\gamma(H-h+1)}-1}
                \cdot
                \min\left\{\ 1, \ 
                \mathsf{t}_h^k(s_h,a_h)+
                \sum_{s_{h+1}}
                {
                \widehat{\mathbb{T}}^k_{h,a_h}}(s_{h+1}|s_h)
                \mathsf{o}_{h+1}^k(s_{h+1})
                \right\}
                    $
                }
                \label{alg_bonus_definition}
            \ENDFOR
            \STATE  \algcom{//Backward dynamic programming}
            \label{hihii}
        \STATE{$\
        \widehat{\beta}^k_{H+1,f_{H+1}}
            \gets 
            \vec{1}_{S}
            $}
            \FOR{$h = H:1$}
                \FOR{$\f_h=(a_1,o_2,\cdots,a_h,o_h)\in \sF{h}, \ a_h \in \sA$}
                \STATE  \algcom{\text{//Invoke Bellman equation \eqref{blmaneqs_raw} under beta vector representation}}
                \STATE{$
                \widehat{\mathsf{Q}}_h^k(a_h;f_{h})
                \gets 
                \frac{1}{\gamma}
                \ln
                \mathbb{E}_{\boldsymbol{O^\prime} 
                \sim \text{Unif}\mathscr{O}}
                \left \langle 
                \widehat{\sigma}^k_{h+1, f_{h+1}=(f_h,a_h,\boldsymbol{O^\prime})}
                \ ,\ 
                \widehat{\beta}^k_{h+1,f_{h+1}=(f_h,a_h,\boldsymbol{O^\prime})}
                \ 
                \right \rangle 
                $}
                \label{alg_beta_represent}
                \STATE{
                $\hphantom{\widehat{\mathsf{Q}}_h^k(a_h;f_{h})}
                \mathllap{\mathsf{\widehat{V}}^k_{h}(f_h)}
                \gets
                \Max{a\in \sA} \ 
                \mathsf{\widehat{Q}}^k_h(a_h;f_h)         
                $}
                \STATE{$
                \hphantom{\widehat{\mathsf{Q}}_h^k(a_h;f_{h})}
                \mathllap{\widehat{\pi}^k_h(f_h)}
                \gets 
                \Argmax{a\in \sA}
                \mathsf{\widehat{Q}}^k_h(a;f_h)   
                $}
                \algcom{\text{// Obtain the greedy policy}}
                \label{greedy_policy_definition}
                \STATE 
                    \algcom{\text{//Update beta vector by Eq.~\eqref{update_empirical_beta}}}
                \STATE{
                    \begin{multline*}
                    \hphantom{\widehat{\mathsf{Q}}_h^k(a_h;f_{h})}
                    \mathllap{\widehat{\beta}^k_{h,f_h}(s_h)}
                    \gets
                    e^{\gamma r_{h}(s_h,\widehat{\pi}^k_h(f_h))}
                    \sum_{s_{h+1}}\mathbb{\widehat{T}}_{h,\widehat{\pi}^k_h(f_h)}^k(s_{h+1}|s_h)
                    \\\sum_{o_{h+1}}\mathbb{\widehat{O}}_{h+1}^k(o_{h+1}|s_{h+1})
                    \left[\widehat{\beta}^k_{h+1,f_{h+1}=(f_h,\widehat{\pi}^k_h(f_h),o_{h+1})}\right]_{s_{h+1}}
                    +
                    \boldsymbol{\sgn(\gamma)}
                    \cdot
                    \mathsf{b}_{h+1}^k(s_h,\widehat{\pi}^k_h(f_h))
                    \end{multline*}
                }
                \label{alg_update_beta}
                \STATE \algcom{\text{//Control the range of beta vector}}
                \STATE{
                $
                \hphantom{\widehat{\mathsf{Q}}_h^k(a_h;f_{h})}
                \mathllap{\widehat{\beta}_{h,f_h}^k}
                \gets
                \text{Clip}
                \left(
                \widehat{\beta}_{h,f_h}^k
                \ , \  
                \left[e^{\gamma^-(H-h+1)}, 
                e^{\gamma^+(H-h+1)}\right]
                \right)
                $
                \label{alg_clip_and_sgn}
                }
                \ENDFOR
            \ENDFOR
            \STATE  \algcom{//Learning}
            \STATE{Play with the real environment
            under policy $\{\widehat{\pi}_h^k\}_{h=1}^H$ 
            and collect a trajectory
            $({\hat a_1^k},\hat{o}_2^k,\ldots, \hat o_H^k, \hat a_H^k)$}
            \label{alg_empirical_collect}
            \STATE{Reveal the hidden states $\{{\hat s_h^k}\}_{h=1}^H$ in the previous H steps to the agent.}
            \algcom{//Hindsight observation}
            \label{alg_hindsight}
            \STATE \algcom{//{
            Update the empirical model
            }}
            \FOR{$h = 1:H$}
                \FOR{$ (s,a,o,s') \in \sS\times\sA\times\sO\times\sS$}
                \STATE ${\widehat{N}_{h}^{k+1}}(s)
                        \gets  
                        \sum_{\kappa=1}^k \mathds{1} \left\{{\hat s_h^\kappa}=s\right\}
                        \qquad
                        {\widehat{N}_{h}^{k+1}}(s,a)\gets  \sum_{\kappa=1}^k \mathds{1} \left\{{\hat s_h^\kappa}=s, {\hat a_h^\kappa}=a\right\}
                        \qquad$
                \label{alg_count_occur}
                \STATE  $
                        {\widehat{\mu}_1^k(s)}
                        \gets
                        \frac{1}{k} \sum_{\kappa=1}^k \mathds{1}\{\widehat{s}_1^{\kappa}=s\}
                        $
                \STATE  $
                        {\widehat{\mathbb{T}}_{h,a}^{k+1}}(s'|s)
                        \gets
                        \frac{\sum_{\kappa=1}^k\mathds{1}
                        \{\hat{s}_{h+1}^\kappa=s'\ , \ {\hat s_h^\kappa} =s\ , \ {\hat a_h^\kappa} =a\}}
                        {\max \left\{ 1, 
                        {\widehat{N}_h^{k+1}}
                        (s,a)\right\}} 
                        \quad 
                        {\widehat{\mathbb{O}}_{h}^{k+1}}(o|s)
                        \gets
                        \frac{\sum_{\kappa=1}^k \mathds{1} \{{\hat o_{h}^\kappa}=o\ , \ {\hat s_h^\kappa} =s\}}
                        {\max\left\{1, {\widehat{N}_h^{k+1}}(s)\right\}}
                        $
                \label{alg_approximate_matrix}
                \ENDFOR
            \ENDFOR 
        \ENDFOR
    \end{algorithmic} 
\end{algorithm}

\section{Regret analysis}
\label{section_regret_analysis}
Our regret analysis takes several steps.
First, we study the statistical error of the beta vectors and design a new bonus for our problem, which ensures optimism in value functions. 
Later, we represent the regret by the beta vectors before we roll down the Bellman equations to calculate the error accumulated during the entire Markov process. 

Several facts will be repeatedly used in the proceeding derivations.
\subsection{Preparation}
The following result captures the empirical error of emission matrix $\mathbb{O}_h(\cdot|s_h)$:
\begin{fact}(Empirical error of the emission matrix)
With probability at least $1-\delta$, 
    \begin{equation}
    \begin{aligned}
    \norm{\widehat{\mathbb{O}}^k_h(\cdot|s_h)-\mathbb{O}_h(\cdot|s_h)}_1
    \leq 
    \min \left\{ \ 2 \  , \ 
    \sqrt{\frac{2 O \ln \frac{KHS}{\delta}}{N_h^{k+1}(s_h) \vee 
    1}}
    \ \right\}
    \end{aligned}
    \end{equation}
\end{fact}
The proof is similar to Lemma C.1 in \cite{Lipschitz23}. This relation holds without hindsight observability.

The upper and lower bounds for the beta vectors will also be useful.
\begin{fact}
\label{lemma_beta_vector_uppper_lower_bounds}
(Bounds of the beta vectors)
$e^{(H-h+1)\gamma^-}\leq \left[\vec{\beta}_{h,f_h}\right]_{s_h} \leq e^{(H-h+1)\gamma^+}$ for all $f_h\in \sF{h}$ and $s_h \in \sS$.
\end{fact}
The proof is done by a simple induction on h.

\subsection{Optimism}
\label{section_optimism_and_bonus}
\subsubsection{Statistical Error of The Beta vector}
\label{section_lower_bound_beta_vector_error}
In this section, we quantify the error in beta vectors caused by the inaccurate approximation of the transition and emission probabilities. For the convenience of narration, we will temporarily view the beta vector as a binary function over the space of $\mathscr{S}\times \mathscr{O}$. Since our result holds for any given episode, we omit the subscripts $k$ in the following derivations.

The update rule of the beta vectors in \eqref{update_rule_real_beta} is equivalent to
\begin{equation*}
\begin{aligned}
\vec{\beta}_{h,f_h}(s_h)
=&
e^{\gamma r_h(s_h,\pi_h(f_h))}
\sum_{\mathbf{s_{h+1}}}\mathbb{T}_{h,\pi_h(f_h)}(\mathbf{s_{h+1}}|s_h)
\sum_{\mathbf{o_{h+1}}}\mathbb{O}_{h+1}(\mathbf{o_{h+1}}|\mathbf{s_{h+1}})
\vec{\beta}_{h+1,f_{h+1}=(f_h,\pi_h(f_h),\mathbf{o_{h+1}})}(\mathbf{s_{h+1}})
\\
=&
\mathbb{E}_{\mathbf{S}_{h+1}\sim \mathbb{T}_{h,\pi_h(f_h)}(\cdot|s_h),\mathbf{O_{h+1}}\sim \mathbb{O}_{h+1}(\cdot|\mathbf{s_{h+1}})}
\left[
e^{\gamma r_h(s_h,\pi_h(f_h))}
\vec{\beta}_{h+1,f_{h+1}=(f_h,\pi_h(f_h),\mathbf{O}_{h+1})}(\mathbf{S}_{h+1})
\right]
\end{aligned}
\end{equation*}
During the update of beta vectors we view $\vec{\beta}_{h,f_h}(s_h)$ as a
binary functions over the space of $\mathscr{Z}:=\mathscr{S\times O}$. 
With a slight abuse of notations, we denote $(s_h,o_h)$ as $z_h$
and write $\vec{\beta}_{h,f_h=(\tau_{h-1},o_h)}(s_h)$ as $\vec{\beta}_h(z_h ;\tau_{h-1})$.
We also abbreviate $r_h(s_h,\pi_h(f_h))$ as $\mathsf{r}_h(z_h)$. The transition law of the beta vector can be written as a joint distribution over the space of $\sS\times \sO$:
$\mathbb{P}^{\pi_h}_h(z_{h+1}|s_h)
:=
\mathbb{T}_{h,\pi_h(f_h)}(\mathbf{s_{h+1}}|s_h)
\mathbb{O}_{h+1}(\mathbf{o_{h+1}}|{s_{h+1}})$

Using the newly introduced short hands we can write the update rule of beta vectors as
\begin{equation*}
\begin{aligned}
\label{distribution_err_real_beta}
\vec{\beta}_h(z_h;\tau_{h-1})=&
\mathbb{E}_{\mathbb{P}^{\pi_h}_h(\cdot|s_h)}\left[e^{\gamma \mathsf{r}_h(z_h)}\vec{\beta}_{h+1}(\mathbf{Z}_{h+1};\tau_{h})\right]
=
\mathbb{E}_{\mathbb{P}^{\pi_h}_h}
\left[e^{\gamma \left(\mathsf{r}_h(z_h)+\frac{\ln \vec{\beta}_{h+1}(\mathbf{Z}_{h+1};\tau_{h})}{\gamma}\right)}\right]
=
\mathbb{E}_{\mathbb{P}^{\pi_h}_h} 
\left[e^{\gamma \mathcal{V}_h (\mathbf{Z}_{h+1};\tau_{h})}\right]
\end{aligned}
\end{equation*}
We remind the reader that we have introduced a new function $\mathcal{V}_h(z_{h+1})$ to 
abbreviate $\mathsf{r}_h(z_h)+\frac{\ln \vec{\beta}_{h+1}(z_{h+1};\tau_{h})}{\gamma}$. 
The function $\mathcal{V}_h$ is bounded in $[1,H-h+1]$ according to Lemma \ref{lemma_beta_vector_uppper_lower_bounds}. Next, we will calculate the approximation error occurred when the beta vector updates. The error is caused by the inaccurate estimate of the joint transition law.
\begin{equation}
\begin{aligned}
\label{last_line}
&
\abs{
\left(\mathbb{E}_{\mathbb{\widehat{P}}_h^{\pi_h}(\cdot|s_h)}
-\mathbb{E}_{\mathbb{P}_h^{\pi_h}(\cdot|s_h)}\right)
(e^{\gamma r_h}\vec{\beta}_{h+1}( \cdot ;\tau_{h}))
}
\\=&
\abs{
\mathbb{E}_{\mathbf{Z}_{h+1}\sim \mathbb{\widehat{P}}_h^{\pi_h}(\cdot|s_h)}
[e^{\gamma \mathcal{V}_h(\mathbf{Z}_{h+1};\tau_{h})}]
-
\mathbb{E}_{\mathbf{Z}_{h+1}\sim {\mathbb{P}}_h^{\pi_h}(\cdot|s_h)}
[e^{\gamma \mathcal{V}_h(\mathbf{Z}_{h+1};\tau_{h})}]
}
\\=&
\abs{\left(
\sum_{s_{h+1}}\widehat{\mathbb{T}}_{h,a_h}(s_{h+1}|s_h)
\sum_{o_{h+1}}\widehat{\mathbb{O}}_{h+1}(o_{h+1}|s_{h+1})
-
\sum_{s_{h+1}}{\mathbb{T}}_{h,a_h}(s_{h+1}|s_h)
\sum_{o_{h+1}}{\mathbb{O}}_{h+1}(o_{h+1}|s_{h+1})
\right)
e^{\gamma \mathcal{V}_{h+1}(s_{h+1},o_{h+1};\tau_{h})}
}
\\
=&
\sum_{s_{h+1}}
\widehat{\mathbb{T}}_{h,a_h}(s_{h+1}|s_h)
\abs{
\sum_{o_{h+1}}
\widehat{\mathbb{O}}_{h+1}(o_{h+1}|s_{h+1})
-
{\mathbb{O}}_{h+1}(o_{h+1}|s_{h+1})
{e^{\gamma \mathcal{V}_{h+1}(s_{h+1},o_{h+1};\tau_{h})}}
}
\\+&
\abs{
\sum_{s_{h+1}}
\left(
\widehat{\mathbb{T}}_{h,a_h}(s_{h+1}|s_h)
-
{\mathbb{T}}_{h,a_h}(s_{h+1}|s_h)
\right)
\cdot
{
\sum_{o_{h+1}}{\mathbb{O}}_{h+1}(o_{h+1}|s_{h+1})
e^{\gamma \mathcal{V}_{h+1}(s_{h+1},o_{h+1};\tau_{h})}
}
}
\end{aligned}
\end{equation}

Lemma \ref{lemma_naive_concentration} implies a natural upper bound for the error in Eq.~\eqref{last_line}:
\begin{equation}\begin{aligned}\label{bonus_bound_1}
\abs{
\left(\mathbb{E}_{\mathbb{\widehat{P}}_h^{\pi_h}(\cdot|s_h)}
-\mathbb{E}_{\mathbb{P}_h^{\pi_h}(\cdot|s_h)}\right)
(e^{\gamma r_h}\vec{\beta}_{h+1}( \cdot ;\tau_{h}))
}\leq 
\abs{e^{\gamma^+(H-h+1)}-e^{\gamma^-(H-h+1)}}\cdot 
\norm{\widehat{\mathbb{P}}_h^{\pi_h}-{\mathbb{P}}_h^{\pi_h}}_{tv}
\leq \abs{e^{\gamma(H-h+1)}-1}\cdot 1
\end{aligned}\end{equation}
We can also use Lemma \ref{lemma_concentration_Hoeffding_function} to derive another upper bound. With probability at least $1-2\delta$,
\begin{equation}
    \begin{aligned}
    \label{bonus_bound_2}
    &
    \abs{
    \left(\mathbb{E}_{\mathbb{\widehat{P}}_h^{\pi_h}(\cdot|s_h)}
    -\mathbb{E}_{\mathbb{P}_h^{\pi_h}(\cdot|s_h)}\right)
    (e^{\gamma r_h(\cdot)}\vec{\beta}_{h+1}(\cdot;\tau_h))
    }
    \\\leq & 
    \abs{e^{\gamma(H-h+1)}-1}
    \cdot 
    \left[
    \min\left\{\ 1,\  
    3
    \sqrt{\frac{S\cdot H\ln \frac{KHSOA}{\delta}}{\widehat{N}_h^k(s_h,\widehat{\pi}^k_h(s_h))\vee 1}}
    \right\}
    \right.
    \\&
    \left.
    +
    \sum_{s_{h+1}}\widehat{\mathbb{T}}^k_{h,\widehat{\pi}^k_h(s_h)}(s_{h+1}|s_h)
    \min\left\{
    \ 1, \ 
    3\sqrt{\frac{O\cdot H\ln \frac{KHSOA}{\delta}}{{\widehat{N}_{h+1}^{k+1}(s_{h+1})\vee 1}}}
    \right\}
    \right]
    \end{aligned}
    \end{equation}

The additional $\sqrt{H}$ comes from the inherent history-dependency of the of POMDP: 
$$
\sqrt{\ln \frac{KHO(O^hA^h)}{\delta}}, 
\sqrt{\ln \frac{KHSA(O^hA^h)}{\delta}}
\lt \sqrt{H \frac{\ln KHSOA}{\delta}}
$$.

Putting Eq.~\eqref{bonus_bound_1} and \eqref{bonus_bound_2} together, we conclude that with probability at least $1-2\delta$:
\begin{equation}
\begin{aligned}
\label{bonus_bound_all}
&\left(
\sum_{s_{h+1}}\widehat{\mathbb{T}}^k_{h,\widehat{\pi}^k_h(s_h)}(s_{h+1}|s_h)
\sum_{o_{h+1}}\widehat{\mathbb{O}}^k_{h+1}(o_{h+1}|s_{h+1})
\right.\\&
\left.-
\sum_{s_{h+1}}{\mathbb{T}}_{h,\pi_h(f_h)}(s_{h+1}|s_h)
\sum_{o_{h+1}}{\mathbb{O}}_{h+1}(o_{h+1}|s_{h+1})
\right)
e^{\gamma r_h(s_h,\widehat{\pi}_h(s_h))}
\vec{\beta}_{h+1,f_{h+1}=(f_h,\pi_h(f_h), o_{h+1})}(s_{h+1})
\\
\leq &
\abs{e^{\gamma(H-h+1)}-1}
\min\left\{
\ 1, \ 
\min\left\{\ 1,\  
3
\sqrt{\frac{S\cdot H\ln \frac{KHSOA}{\delta}}{\widehat{N}_h^k(s_h,\widehat{\pi}^k_h(s_h))\vee 1}}
\right\}
\right.
\\&\qquad \qquad \qquad \qquad \qquad +
\left.
\sum_{s_{h+1}}\widehat{\mathbb{T}}^k_{h,\widehat{\pi}^k_h(s_h)}(s_{h+1}|s_h)
\min\left\{
\ 1, \ 
3\sqrt{\frac{O\cdot H\ln \frac{KHSOA}{\delta}}{{\widehat{N}_{h+1}^{k+1}(s_{h+1})\vee 1}}}
\right\}
\right\}
\end{aligned}
\end{equation}

We design the exploration bonus according to Eq.~\eqref{bonus_bound_all}.
\begin{definition}(Bonus)
Bonuses are a series of real-valued functions $\mathsf{b}_h^k(\cdot,\cdot;\gamma): \mathscr{S\times A} \to \R_{\geq 0}$ specified as
\begin{multline}
\label{bonus_design}
{
\mathsf{b}^k_h(s_h,a_h;\gamma)}
:=
\abs{e^{\gamma(H-h+1)}-1}
\cdot 
\min
\left\{
\ 1, \ 
{
\min
\left\{
\ 1, \ 
3\sqrt{\frac{S\cdot 
H
{\ln\frac{KHSOA}{\delta}}}{\widehat{N}_h^k(s_h,a_h)\vee 1}}
\right\}
}
\right.
\\
+
\left.
\sum_{s_{h+1}}\widehat{\mathbb{T}}^k_{h,\widehat{\pi}^k_h(s_h)}(s_{h+1}|s_h)
\ 
{
\min\left\{
\ 1, \ 
3\sqrt{\frac{O\cdot 
H
{\ln\frac{KHSOA}{\delta}}}{{\widehat{N}_{h+1}^{k+1}(s_{h+1})\vee 1}}}
\right\}
}
\ \ 
\right\}
\end{multline}
\end{definition}

Previously, we have analyzed the upper bound of the approximation error incurred during the update process. We will utilize the relevant results and define the series of ``empirical beta vectors'' $\widehat{\beta}_h$ that helps approximate the value function.
\begin{definition}(Empirical beta vectors)
Given any episode $k\in[K]$,
if we use $a_h$ to abbreviate the action selected by the greedy policy $\widehat{\pi}_h^k(f_h)$ given by algorithm \ref{alg:Beta Vector Value Iteration}, 
then the empirical beta vector is defined as 
\label{def_empirical_beta}
\begin{equation}
\begin{aligned}
\label{update_empirical_beta}
&\forall f_{H+1}\in\sF{H+1},s_{H+1}\in\sS: 
\quad
\widehat{\beta}^k_{H+1,f_{H+1}}(s_{H+1}):=1
\\
&\forall h\in [H],f_h\in\sF{h}, s_h\in\sS, \\
&
\quad
\widehat{\beta}^k_{h,f_h}(s_h)
:=
e^{\gamma r_h(s_h,a_h)}
\sum_{s_{h+1}\in \sS}
\mathbb{\widehat{T}}_{h,a_h}^k(s_{h+1}|s_h)
\sum_{o_{h+1} \in \sO}
\mathbb{\widehat{O}}_{h+1}^k(o_{h+1}|s_{h+1})
\widehat{\beta}^k_{h+1,f_{h+1}=(f_h,a_h,o_{h+1})}(s_{h+1})
\\
&\qquad \qquad \qquad 
+
{
\sign{\gamma}
\cdot 
\mathsf{b}_{h}^k(s_h,a_h;\gamma)
}
\end{aligned}
\end{equation}
\end{definition}

\begin{remark}(Range of the empirical beta vectors)
\label{beta_vector_error_range}
{If we pose no further restrictions} on the range of $\widehat{\beta}_h^k$ we can show that 
the upper bound on $\widehat{\beta}_h^k$ will inevitably depend 
on $e^{\gamma H^2}$, which will cause an additional factor of $H$ in our regret.
To circumvent this issue we have manually clipped the value of empirical beta vector in our algorithm(line \ref{alg_clip_and_sgn}), so as to force $\widehat{\beta}_h$ to stay in the same range as $\vec{\beta}_h$. 
As a consequence, the difference between the two beta vectors, which will be called the {``beta-vector error''}, is controlled by
\begin{equation}
\begin{aligned}
\label{range_of_empirical_beta}
{
\abs{\widehat{\beta}_h^k - \vec{\beta}_h} \leq \abs{e^{\gamma (H-h+1)} -1}
}
\end{aligned}
\end{equation}
\end{remark}

Apart from the absolute error between the beta vectors, we are also concerned with their 
size relationship.
\begin{corollary}(Size relationship between beta vectors)
\label{corollary_Size_beta}
\\
For all $s_h \in \sS:$
$\quad \widehat{\beta}^k_{h,f_h}(s_h) \geq \vec{\beta}_{h,f_h}(s_h)$ when $\gamma \gt 0$
and
$\widehat{\beta}^k_{h,f_h}(s_h)\leq \vec{\beta}_{h,f_h}(s_h)$ when $\gamma \lt 0$. 
\end{corollary}
\begin{proof}
We only prove the case when $\gamma \lt 0$. The other way is similar. The statement holds when $h=H+1$ since both vectors are defined to be $\vec{1}_{S}$. Suppose that it holds at $h+1$, 
then by the induction hypothesis we have $\widehat{\beta}_h^k
\leq 
\widehat{\mathbb{T}}
\widehat{\mathbb{O}}
e^{\gamma r_h} \vec{\beta}_{h+1}^k
{-\mathsf{b}_h^k}
$
=$
\left[ 
(\widehat{\mathbb{T}}
\widehat{\mathbb{O}}
-
\mathbb{T}\mathbb{O}
)
e^{\gamma r_h} \vec{\beta}_{h+1}^k
{-\mathsf{b}_h^k}
\right]
+
\mathbb{T}\mathbb{O}\vec{\beta}_{h+1}^k
$
=
$
\left[ 
(\widehat{\mathbb{T}}
\widehat{\mathbb{O}}
-
\mathbb{T}\mathbb{O}
)
e^{\gamma r_h} \vec{\beta}_{h+1}^k
{-\mathsf{b}_h^k}
\right]
+\vec{\beta}_{h}^k
$. 
Equation \eqref{bonus_bound_all} implies that the terms in the bracket are less than zero, which helps us complete the proof.
\end{proof}

\subsubsection{Optimism in Value Functions}
\label{section_optimism_in_value_functions}
Corollary \ref{corollary_Size_beta} directly results in the
optimism in value functions:
for all $k\in[K]$,
\begin{equation}
\begin{aligned}
\label{optimism_val_fuctions}
V_1^{\pi^\star} \leq \widehat{V}_1^{\widehat{\pi}^k}
\end{aligned}
\end{equation}
\begin{proof}
$$
V_{1}^{\pi^\star}-\widehat{V}_1^{\pi^\star}
=
\frac{1}{\gamma}\ln \langle \vec{\sigma}_1, \vec{\beta}_1 \rangle - \frac{1}{\gamma} \ln\langle \widehat{\sigma}_1^k,\widehat{\beta}^k_1 \rangle
=\frac{1}{\gamma} \ln \langle \vec{\mu}_1 , \vec{\beta}_1 \rangle -
\frac{1}{\gamma} \ln \langle \vec{\mu}_1 , \widehat{{\beta}}^k_1\rangle
\leq 0
$$
The last step remaining is the fact that $\widehat{V}_1^{\pi{^\star}} \leq \widehat{V}_1^{\widehat{\pi}^k}$.
\end{proof}

\paragraph{Additional notations}\label{all_bounds}
For ease of notations we will abbreviate several upper and lower bounds in the following analysis.
\begin{equation*}
\begin{aligned}
&\text{Upper and lower bounds on the risk measure} \quad
\underline{B}_{u}:=e^{-\gamma^-}
\leq e^{\gamma r_h(s_h,a_h)} \leq e^{\gamma^+}:=\overline{B_{u}}
\\&\text{Upper and lower bounds on the beta-vectors}\quad 
\underline{B}_{\vec{\beta}_h}
:=e^{\gamma^{-}(H-h+1)}
\leq \vec{\beta}_h \leq e^{\gamma^{+}(H-h+1)}
:=\overline{B}_{\vec{\beta}_h}
\\&{\text{Upper bound of the bonus}}\quad 
0\leq 
\mathsf{b}_h^k(s_h,a_h)
\leq
\abs{
e^{{\gamma}(H-h+1)}-1
}
:=
\overline{B}_{\mathsf{b}_h}
\quad 
\\&{\text{Upper bound of the ``beta vector error''}}
0\leq 
\abs{
\widehat{\beta}_{h+1}^k-\vec{\beta}_{h+1}
}
\leq
\abs{e^{\gamma(H-h+1)}-1}
:=\overline{B}_{\Delta \vec{\beta}_{h+1}}
\quad 
\end{aligned}
\end{equation*}

\begin{remark}\label{remark_degeneracy_relation}
According to the update rule in equation \eqref{update_rule_real_beta} and \eqref{update_empirical_beta}, when the risk-sensitivity parameter
$\gamma$ tends to zero, the beta vectors will degenerate to $\vec{1}$. 
The bonus function, as well as the beta-vector error will vanish with in the rate of $H$
\begin{equation}
\begin{aligned}
\label{bound_degeneration_to_risk_neutral}
&\lim_{\gamma \to 0} \vec{\beta}_{h}=
\lim_{\gamma \to 0} \widehat{\beta}_{h}
=\vec{1}
\qquad
\lim_{\gamma \to 0} \mathsf{b}_h^k(\cdot,\cdot)
=
\lim_{\gamma \to 0} B_{\Delta \vec{\beta}_h}=0
\quad 
\lim_{\gamma \to 0}
\frac{\overline{B}_{\mathsf{b}_h}}{\gamma}=
\lim_{\gamma \to 0}
\frac{\overline{B}_{\Delta \vec{\beta}_h}}{\gamma}=H
\end{aligned}
\end{equation}
which contributes an additional $H$ to our regret.
\end{remark}

\subsection{Regret Calculation}

\subsubsection{Represent the Regret by Beta Vectors}
The regret can be represented as the approximation error of beta vectors at the initial time step.
\begin{equation}
\begin{aligned}\label{split_regret_to_two_errors}
&
\text{Regret}(K; \mathcal{P})
:=
\sum_{k=1}^{K}
U^{-1}\mathbb{E}_{\mathcal{P}}
\left[
U V_1^{\pi^{\star}} \right]
-
U^{-1}\mathbb{E}_{\mathcal{P}}
\left[
U V_{1}^{\widehat{\pi}^k}
\right]
\\\leq&
\sum_{k=1}^{K}
U^{-1}\mathbb{E}_{\mathcal{P}}
\left[
U\widehat{V}_1^{k,\widehat{\pi}^k}
\right]
-
U^{-1}\mathbb{E}_{\mathcal{P}}
\left[
UV_{1}^{\widehat{\pi}^k}
\right]
\quad //{\text{Section \ref{section_optimism_in_value_functions}}}
\\=&
\sum_{k=1}^{K}
U^{-1}\mathbb{E}_{\mathcal{P}}
\left[
\langle \widehat{\sigma}^k_1\ , \
\widehat{\beta}_{1}^{k,\widehat{\pi}^k} \rangle
\right]
-
U^{-1}\mathbb{E}_{\mathcal{P}}
\left[
\langle \vec{{\sigma}}_1\ , \
\vec{\beta}_{1}^{\widehat{\pi}^k} \rangle
\right]
\quad //{\text{Theorem
\ref{theorem_beta_vector_representation of value functions}}}
\\=&
\sum_{k=1}^{K}
\left(
U^{-1}\mathbb{E}_{\mathcal{P}}
\left[
\langle \widehat{\sigma}^k_1\ , \
\widehat{\beta}_{1}^{k,\widehat{\pi}^k} \rangle
\right]
-
U^{-1}\mathbb{E}_{\mathcal{P}}
\left[
\langle \vec{\sigma}_1\ , \
\widehat{\beta}_{1}^{k,\widehat{\pi}^k} \rangle
\right]
\right)
+
\left(
U^{-1}\mathbb{E}_{\mathcal{P}}
\left[
\langle \vec{\sigma}_1\ , \
\widehat{\beta}_{1}^{k,\widehat{\pi}^k} \rangle
\right]
-
U^{-1}\mathbb{E}_{\mathcal{P}}
\left[
\langle \vec{{\sigma}}_1\ , \
\vec{\beta}_{1}^{\widehat{\pi}^k} \rangle
\right]
\right)
\\=&
\sum_{k=1}^{K}
U^{-1}
\langle \widehat{\mu}^k_1\ , \
\widehat{\beta}_{1}^{k,\widehat{\pi}^k} \rangle
-
U^{-1}
\langle \vec{\mu}_1\ , \
\widehat{\beta}_{1}^{k,\widehat{\pi}^k} \rangle
+
\sum_{k=1}^{K}
U^{-1}
\mathbb{E}_{\mathbf{S}_1\sim \vec{\mu}_1}
\left[
\widehat{\beta}_{1}^{k,\widehat{\pi}^k}(\mathbf{S}_1)
\right]
-
U^{-1}
\mathbb{E}_{\mathbf{S}_1\sim \vec{\mu}_1}
\left[
\vec{\beta}_{1}^{\widehat{\pi}^k}(\mathbf{S}_1)
\right]
\\=&
\underbrace{\sum_{k=1}^{K}
U^{-1}
\mathbb{E}_{\mathcal{\widehat{P}}^k}
\left[
\widehat{\beta}_{1}^{k,\widehat{\pi}^k}(\mathbf{S}_1)
\right]
-
U^{-1}
\mathbb{E}_{\mathcal{P}}
\left[
\widehat{\beta}_{1}^{k,\widehat{\pi}^k}(\mathbf{S}_1)
\right]
}_{\text{Prior error
}}
+
\underbrace{\sum_{k=1}^{K}
U^{-1}
\mathbb{E}_{\mathcal{P}}
\left[
\widehat{\beta}_{1}^{k,\widehat{\pi}^k}(\mathbf{S}_1)
\right]
-
U^{-1}
\mathbb{E}_{\mathcal{P}}
\left[
\vec{\beta}_{1}^{\widehat{\pi}^k}(\mathbf{S}_1)
\right]
}_{\text{Evolution error
}}
\end{aligned}
\end{equation}
The terms in the last step of Eq. \eqref{split_regret_to_two_errors} hold significant physical meanings. The first term, referred to as the ``prior error", is incurred by the imprecise estimate of the prior distribution of the hidden states. The second term, named ``evolution error", represents the accumulation of error throughout the entire horizon of the Markov process, resulting from an inaccurate estimate of the beta vector.
In what follows we will try to find upper bounds for the two error terms.

\subsubsection{Bound the Prior Error}\label{section_Bound the prior error}
First, we bound the prior error in Eq. \eqref{split_regret_to_two_errors}.
With probability at least $1-\delta$,
\begin{equation}\begin{aligned}\label{initial_error}
&\underbrace{\sum_{k=1}^{K}
U^{-1}
\mathbb{E}_{\mathcal{\widehat{P}}^k}
\left[
\widehat{\beta}_{1}^{k,\widehat{\pi}^k}(\mathbf{S}_1)
\right]
-
U^{-1}
\mathbb{E}_{\mathcal{P}}
\left[
\widehat{\beta}_{1}^{k,\widehat{\pi}^k}(\mathbf{S}_1)
\right]
}_{\text{Prior error
}}
\leq 
\mathsf{K}_{\gamma} \sum_{k=1}^K 
\abs{
\sum_{s_1}
\left(\widehat{\mu}_1^k(s_1)-\vec{\mu}_1(s_1) \right) \widehat{\beta}_1^{k,\widehat{\pi}^k}(s_1)
}
\\ \leq &
\mathsf{K}_{\gamma}
\sum_{k=1}^K
\frac{\abs{e^{\gamma^+ H}-e^{\gamma^- H}}}{2}
\norm{\widehat{\mu}_1^k(\cdot)-\vec{\mu}_1(\cdot)}_1
\leq 
\frac{e^{(-\gamma)^+}}{\abs{\gamma}}
\sum_{k=1}^K
\frac{\abs{e^{\gamma^+ H}-e^{\gamma^- H}}}{2}
\sqrt{\frac{2S }{k}\ln \frac{K}{\delta}}
\leq 
\frac{e^{\abs{\gamma }H}-1}{\abs{\gamma}}\sqrt{2KS\ln \frac{K}{\delta}}
\end{aligned}\end{equation}
where the second and third inequalities are due to 
Lemma \ref{lemma_naive_concentration} and Lemma  \ref{lemma_concentration_in_one_norm} respectively. 

\subsubsection{Control the Evolution Error}\label{section_control the evolution error}
Next, we will use the Bellman equations to show that
\begin{equation}
\begin{aligned}
\underbrace{\sum_{k=1}^{K}
U^{-1}
\mathbb{E}_{\mathcal{P}}
\left[
\widehat{\beta}_{1}^{k,\widehat{\pi}^k}(\mathbf{S}_1)
\right]
-
U^{-1}
\mathbb{E}_{\mathcal{P}}
\left[
\vec{\beta}_{1}^{\widehat{\pi}^k}(\mathbf{S}_1)
\right]
}_{\text{Evolution error
}}
\leq 
\mathcal{{O}}
\left(
\frac{e^{\abs{\gamma} {H}}-1}{\abs{\gamma}H}
\cdot
{{{H}^{5/2}}}
\sqrt{KS^2OA}
\cdot 
\ln\frac{KHSOA}{\delta}
\right)
\end{aligned}
\end{equation}
\begin{proof}
\vspace{-10pt}
By the Lipschitz continuity of $U$, we have
\begin{equation}
\begin{aligned}\label{regret_init_beta_error}
\underbrace{\sum_{k=1}^{K}
U^{-1}
\mathbb{E}_{\mathcal{P}}
\left[
\widehat{\beta}_{1}^{k,\widehat{\pi}^k}(\mathbf{S}_1)
\right]
-
U^{-1}
\mathbb{E}_{\mathcal{P}}
\left[
\vec{\beta}_{1}^{\widehat{\pi}^k}(\mathbf{S}_1)
\right]
}_{\text{Evolution error
}}
=&
\sum_{k=1}^{K}
\frac{1}{\gamma}\ln 
\mathbb{E}_{\mathcal{P}}
\left[
e^{\gamma V_1^{\pi^{\star}}}
\right]
-
\frac{1}{\gamma}\ln 
\mathbb{E}_{\mathcal{P}}
\left[
e^{\gamma V_{1}^{\widehat{\pi}^k}}
\right]
\\
&\leq 
\mathsf{K}_\gamma \cdot 
\sum_{k=1}^K
\abs{
\mathbb{E}_{\mathcal{P}}
\left[
\widehat{\beta}^k_{1}-{\beta}_{1}
\right]
}
\end{aligned}
\end{equation}

where $\mathsf{K}_{\gamma} = \frac{e^{(-\gamma)^+ H}}{\abs{\gamma}}$. 
Next, we find the recurrence relation between $\mathbb{E}[\widehat{\beta}_h-\vec{\beta}_h]$ and $\mathbb{E}[\widehat{\beta}_{h+1}-\vec{\beta}_{h+1}]$.

\begin{equation*}
\begin{aligned}
&\abs{
\mathbb{E}_{\mathcal{P}}
\left[
\widehat{\beta}_{h,f_h=\boldsymbol{(\tau}_{h-1},\mathbf{O}_h)}^{k,\widehat{\pi}^k}(\mathbf{S}_h)
-
{\beta}_{h,f_h=\boldsymbol{(\tau}_{h-1},\mathbf{O}_h)}^{\widehat{\pi}^k}(\mathbf{S}_h)
\right]
}
\\=&
\left|\ \ 
\mathbb{E}_{\mathcal{P}}
\left[
u(\mathbf{S}_h,\mathbf{A}_h)
\sum_{s_{h+1},o_{h+1}}\widehat{\mathbb{T}}_{h,a_h}(s_{h+1}|s_h)\widehat{\mathbb{O}}_{h+1}(o_{h+1}|s_{h+1})
\widehat{\beta}_{h+1,(\boldsymbol{F_h,A_h,O_{h+1}})}(s_{h+1})
+\text{sgn}(\gamma) \mathsf{b}_h^k(\boldsymbol{S_h,A_h})
\right.
\right.
\\&
-
\left.
\left.
u(\mathbf{S}_h,\mathbf{A}_h)
\sum_{s_{h+1},o_{h+1}}{\mathbb{T}}_{h,a_h}(s_{h+1}|s_h){\mathbb{O}}_{h+1}(o_{h+1}|s_{h+1})
{\beta}_{h+1,(\boldsymbol{F_h,A_h,O_{h+1}})}(s_{h+1})
\right] \ \ 
\right|
\\
=&
\abs{
\mathbb{E}_{\mathcal{P}}
\left[
u\widehat{\mathbb{T}}\widehat{\mathbb{O}}\widehat{\beta}_{h+1}
-
u\mathbb{T}\mathbb{O}\vec{\beta}_{h+1}
+
\text{sgn}(\gamma)\mathsf{b}_h^k
\right]
}
\\:=&
\abs{
\mathbb{E}_{\mathcal{P}}
\left[
u\widehat{\mathbb{T}}\widehat{\mathbb{O}}\widehat{\beta}_{h+1}
-
u\mathbb{T}\mathbb{O}\vec{\beta}_{h+1}
+
\text{sgn}(\gamma)\mathsf{b}_h^k
+
{u\widehat{\mathbb{T}}\widehat{\mathbb{O}}{\beta}_{h+1}
-
u\widehat{\mathbb{T}}\widehat{\mathbb{O}}{\beta}_{h+1}
+
u{\mathbb{T}}{\mathbb{O}}(\widehat{\beta}_{h+1}-{\beta}_{h+1})
-
u{\mathbb{T}}{\mathbb{O}}(\widehat{\beta}_{h+1}-{\beta}_{h+1})}
\right]
}
\\
=&
\abs{
\underbrace{\text{sgn}(\gamma)
\mathbb{E}_{\mathcal{P}}[\mathsf{b}_h^k]
}_{\text{I
}}
+
\underbrace{\mathbb{E}_{\mathcal{P}}
[(\widehat{\mathbb{T}}\widehat{\mathbb{O}}-{\mathbb{T}}{\mathbb{O}})u\vec{\beta}_{h+1}]
}_{\text{II
}}
+
\underbrace{\mathbb{E}_{\mathcal{P}}
[(\widehat{\mathbb{T}}\widehat{\mathbb{O}}-{\mathbb{T}}{\mathbb{O}})(
u\widehat{\beta}_{h+1}-
u{\beta}_{h+1})]
}_{\text{III
}}
+
\underbrace{\mathbb{E}_{\mathcal{P}}
[{\mathbb{T}}{\mathbb{O}}(u\widehat{\beta}_{h+1}-u{\beta}_{h+1})]
}_{\text{IV
}}
}
\end{aligned}
\end{equation*}

By Eq.~\eqref{bonus_bound_all} we observe that 
$\abs{\text{II}}\leq\mathbb{E}_{\mathcal{P}}[\mathsf{b}_h^k]$. Similarly, 
\text{Eq.~\eqref{range_of_empirical_beta}} 
implies that  
$\abs{\text{III}}\leq 
2 \mathbb{E}_{\mathcal{P}}[\mathsf{b}_h^k]
$. 
As for ${\text{IV}}$, we always have $\abs{\text{IV}} \leq  \overline{B}_u \cdot \mathbb{E}_{\mathcal{P}}
\left[\widehat{\beta}_{h+1}-\vec{\beta}_{h+1}\right]$. 
Putting things together we conclude
\begin{equation*}
\begin{aligned}
\abs{\mathbb{E}_{\mathcal{P}}[\widehat{\beta}_h-\vec{\beta}_h]}
\leq 
\abs{\text{I}}+
\abs{\text{II}}+
\abs{\text{III}}+
\abs{\text{IV}}
\leq 
4\cdot \mathbb{E}_{\mathcal{P}}[\mathsf{b}_h^k]
+
\overline{B}_{u}\cdot 
\abs{\mathbb{E}_{\mathcal{P}}[\widehat{\beta}_{h+1}-\vec{\beta}_{h+1}]}
\end{aligned}
\end{equation*}

Abbreviating the term $\abs{\mathbb{E}_{\mathcal{P}}\left[\widehat{\beta}_{h}^k-{\beta}_{h}^k\right]}$ as $\Delta^k_h$, 
we obtain a recursive equation for the beta vector errors:
\begin{equation}
\begin{aligned}
\label{recurrence_Eq}
&\left\{
\begin{aligned}
\Delta^k_{H+1}
=&0
\\
\Delta^k_{h}
\leq& \ 
\overline{B}_{u}\cdot 
\Delta_{h+1}
+
4
\mathbb{E}_{\mathcal{P}}
\left[
\mathsf{b}_h^k
\right]
\ , \ 
\forall h=H:1
\end{aligned}
\right.
\end{aligned}
\end{equation}
Recall that in Eq.~\eqref{regret_init_beta_error} 
we have shown the regret can be controlled by the ``initial beta vector errors'':
$$
\text{Regret}
(K; \mathcal{P}, \gamma)
\leq 
\mathsf{K}_\gamma \cdot 
\sum_{k=1}^K
\mathbb{E}_{\mathcal{P}}\left[\widehat{\beta}^k_{1}-{\beta}^k_{1}\right]
=\mathsf{K}_{\gamma} \sum_{k=1}^K \Delta_{1}^k
$$
Using Lemma \ref{lemma_backward_linear_Eq.}, we roll down Eq.~\eqref{recurrence_Eq} to solve $\Delta_1^k$, after which we bound the regret by the bonus functions:
\begin{equation}
\begin{aligned}
\label{regret_to_sum_product}
\text{Regret}(K; \mathcal{P}, \gamma)
\leq &
\mathsf{K}_\gamma
\sum_{k=1}^{K}
\mathbb{E}_{\mathcal{P}}[\widehat{\beta}^k_1-\beta^k_1]
\leq  
4 \mathsf{K}_\gamma
\sum_{k=1}^{K}
\sum_{h=1}^{H}
\mathbb{E}_{\mathcal{P}}[\mathsf{b}_h^k]
\prod_{t=1}^{h-1}
\overline{B}_{u}
\\ \lt &
4 
\mathsf{K}_\gamma
\sum_{h=1}^{H}
\overline{B}_{u}^{h-1}
\underbrace{
\left[\sum_{k=1}^{K}
\left(
\mathbb{E}_{\mathcal{P}}
[\mathsf{b}_h^k
(\cdot, \cdot;\gamma)]
-
\mathsf{b}_h^k
({{{\hat{s}^k_{h}}}},{{{\hat{a}^k_{h}}}};\gamma)
\right)
\right.
}_{\text{Sample bias
}}
+
\underbrace{
\left.\sum_{k=1}^K
\mathsf{b}_h^k
({{{\hat{s}^k_{h}}}},{{{\hat{a}^k_{h}}}};\gamma)
\right]
}_{\text{Bonus sample
}}
\end{aligned}
\end{equation}
We invoke Lemma \ref{Lemma_concentration_MDS} to bound the terms in the first curly bracket.
With probability at least $1-\delta$, 
\begin{equation*}
\begin{aligned}
\underbrace{
\sum_{k=1}^{K}
\left(
\mathbb{E}_{\mathcal{P}}
[\mathsf{b}_h^k
(\cdot, \cdot;\gamma)]
-
\mathsf{b}_h^k
({{{\hat{s}^k_{h}}}},{{{\hat{a}^k_{h}}}};\gamma)
\right)
}_{\text{Sample bias
}}
\leq 
\abs{e^{\gamma (H-h+1)}-1}
\sqrt{\frac{K}{2}\cdot  \frac{\ln HSA}{\delta}}
\end{aligned}
\end{equation*}

The terms in the second curly bracket will be controlled by 
the pigeon hole lemma \ref{lemma_pigeon-hole}.
Recall that for the empirical data ${(\hat{s}^k_h, \hat{a}^k_h)}$ collected during the learning process
(line \ref{alg_empirical_collect} in algorithm \ref{alg:Beta Vector Value Iteration}), the bonus function picks the value of
\begin{equation}
\begin{aligned}
\mathsf{b}^k_h
({{{\hat{s}^k_{h}}}},{{{\hat{a}^k_{h}}}};\gamma)
=
\abs{e^{\gamma(H-h+1)}-1}
\cdot 
\min
\left\{
\ 1\ , \ 
\mathsf{t}_{h}^k({{{\hat{s}^k_{h}}}},{{{\hat{a}^k_{h}}}})
+
\sum_{s_{h+1}}
\widehat{\mathbb{T}}^k_{h,a_h}(s_{h+1}|{\hat{s}^k_h})
\mathsf{o}_{h+1}^k({s_{h+1}})
\ 
\right\}
\end{aligned}
\end{equation}
In what follows we calculate the summation over the residue terms.
With probability at least $1-2\delta$:
\begin{equation}
    \begin{aligned}
    &\sum_{k=1}^K
    \mathsf{t}^k_{h}({{\hat{s}^k_h,\hat{a}^k_h}})
    +
    \sum_{s_{h+1}}
    \widehat{\mathbb{T}}^k_{h,{{\hat{a}^k_h}}}(s_{h+1}|{{{{\hat{s}^k_h}}}})
    \mathsf{o}_{h+1}^k(s_{h+1})
    \\=&
    \sum_{k=1}^K
    \left[
    \mathsf{t}^k_{h}({{\hat{s}^k_h,\hat{a}^k_h}})
    +
    \sum_{s_{h+1}}{\mathbb{T}}_{h,{{\hat{a}^k_h}}}(s_{h+1}|{{{{\hat{s}^k_h}}}})
    \mathsf{o}_{h+1}^k(s_{h+1})
    \right.
    \\&
    +
    \left.
    \left(
    \sum_{s_{h+1}}\widehat{\mathbb{T}}^k_{h,{{\hat{a}^k_h}}}(s_{h+1}|{{{{\hat{s}^k_h}}}})
    \mathsf{o}_{h+1}^k(s_{h+1})
    -
    \sum_{s_{h+1}}{\mathbb{T}}_{h,{{\hat{a}^k_h}}}(s_{h+1}|{{{{\hat{s}^k_h}}}})
    \mathsf{o}_{h+1}^k(s_{h+1})
    \right)
    \right]
    \\
    \leq &
    \sum_{k=1}^K
    \left[
    \mathsf{t}^k_{h}({{\hat{s}^k_h,\hat{a}^k_h}})
    \right.
    +
    \underbrace{
    1
    \cdot
    \mathsf{t}_h^k({{\hat{s}^k_h}},{{\hat{a}^k_h}})
    }_{\text{transition error
    }}
    +
    \\
    &
    \underbrace{\left(
    \mathbb{E}_{s_{h+1}\sim
    {\mathbb{T}}_{h,{{\hat{a}^k_h}}}(\cdot|{{{{\hat{s}^k_h}}}})}
    \mathsf{o}_{h+1}^k(s_{h+1})
    -
    \mathsf{o}_{h+1}^k({{\hat{s}^k_{h+1}}})
    \right)
    }_{\text{Sample bias
    }}
    +
    \left.
    \mathsf{o}_{h+1}^k({{{{\hat{s}^k_{h+1}}}}})
    \right]
    \ //{\text{Lemma \ref{lemma_concentration_Hoeffding_function}}}
    \\
    \leq 
    &
    \ 2\ \left(
    \sum_{k=1}^K
    \ 
    \mathsf{t}_h^k({{\hat{s}^k_h}},{{\hat{a}^k_h}})
    \ 
    \right)
    +
    \underbrace{
    2
    \cdot
    \sqrt{\frac{K}{2} \cdot \ln  HSA/{\delta}}
    }_{\text{Sample bias
    }}
    +
    \sum_{k=1}^K
    \mathsf{o}_{h+1}^k({{\hat{s}^k_{h+1}}})
    \quad {\text{//Lemma \ref{Lemma_concentration_MDS}}} 
    \end{aligned}
\end{equation}
According to the definition of the residue terms in Eq.~\eqref{bonus_design}, 
the pigeon-hole Lemma \ref{lemma_pigeon-hole} suggests that
\begin{equation}
\begin{aligned}
\sum_{k=1}^k {\mathsf{t}_h^k({{\hat{s}^k_{h}, \hat{a}^k_h}})}
\equiv &
\sum_{k=1}^k
\min
\left\{
\ 1, \ 
3\sqrt{\frac{S\cdot 
H
{\ln\frac{KHSOA}{\delta}}}{\widehat{N}_h^k(
{\hat{s}^k_{h}, \hat{a}^k_h}
)
\vee 1}}
\ \right\}
\leq 
\left(3\sqrt{SH} \cdot \iota \right)
\cdot 
2\sqrt{K\cdot SA}
\\
\sum_{k=1}^K
\mathsf{o}_{h+1}^k({{\hat{s}^k_{h+1}}})
\equiv
&
\sum_{k=1}^K
\min\left\{
\ 1, \ 
3\sqrt{\frac{O\cdot 
H
{\ln\frac{KHSOA}{\delta}}}{{\widehat{N}_{h+1}^{k+1}({{\hat{s}^k_{h+1}}})}\vee 1}}
\ \right\}
\leq 
\left(3\sqrt{OH} \cdot \iota \right)
\cdot 2\sqrt{K\cdot S}
\end{aligned}
\end{equation}
where $\ln KHSOA/\delta$ is abbreviated as $\iota$. Then with probability at least $1-2\delta$,
\begin{equation}
\begin{aligned}\label{sum_of_bonus_samples}
&
\underbrace{\sum_{k=1}^K
\mathsf{b}_h^k
({{{\hat{s}^k_{h}}}},{{{\hat{a}^k_{h}}}};\gamma)
}_{\text{Bonus sample
}}\leq 
\abs{e^{\gamma(H-h+1)}-1}
\cdot 
\left(
12\sqrt{K \cdot SA \cdot S \cdot H}+6\sqrt{K\cdot S\cdot O}
+\sqrt{K/2}
\right)
\cdot 
\sqrt{\ln \left(\frac{KHSOA}{\delta}\right)}
\\
\lt &
12
\underbrace{\abs{e^{\gamma(H-h+1)}-1}
}_{\text{Risk measure
}}
\cdot
\left(
{{{\sqrt{K}}}}
\underbrace{
{{{\sqrt{SA}}}}
\cdot
{{{\sqrt{S}}}}
}_{\text{Hidden state error
}}
+
{{{\sqrt{K}}}}
\underbrace{{{{\sqrt{S}}}}
\cdot
{{{\sqrt{O}}}}
}_{\text{Observation error
}}
+
\underbrace{{{{\sqrt{K}}}}
}_{\text{MDS
}}
\ 
\right)
\cdot	
\underbrace{{{{\sqrt{H}}}\cdot
\sqrt{\ln \left(\frac{KHSOA}{\delta}\right)}}
}_{\text{History-dependency of POMDP
}}
\end{aligned}
\end{equation}
\vspace{-10pt}
Putting things together we can safely state that with probability at least $1-3\delta$, 
\begin{equation}
\begin{aligned}
\label{last_step_to_regret}
&\underbrace{
\sum_{k\in[K]}
\mathbb{E}_{\mathcal{P}}[\mathsf{b}_{h}^k(\cdot,\cdot;\gamma)]-\mathsf{b}_h^k({\hat{s}^k_h,\hat{a}^k_h};\gamma)
}_{\text{Sample bias}} 
+
\underbrace{
\sum_{k=1}^K
\mathsf{b}_h^k({\hat{s}^k_h,\hat{a}^k_h};\gamma)
}_{\text{Bonus samples
}}
\\ \leq &
\ 
12 \cdot
\underbrace{\abs{e^{\gamma(H-h+1)}-1}
}_{\text{Risk measure
}}
\cdot
\left(
\underbrace{
{{\sqrt{{KS^{{{^2}}}A}}}}
}_{\text{Hidden state error
}}
+
\underbrace{{{\sqrt{{KS}{{{{O}}}}}}
}
}_{\text{Observation error
}}
+
\underbrace{\sqrt{K}
}_{\text{MDS
}}
\ 
\right)
\cdot	
\underbrace{{{{\sqrt{H}}}\cdot
\sqrt{\ln \left(\frac{KHSOA}{\delta}\right)}}
}_{\text{History-dependency of POMDP
}}
\end{aligned}
\end{equation}

Bringing Eq.~\eqref{last_step_to_regret}
back to Eq.~\eqref{regret_to_sum_product}, we conclude that with probability at least $1-3\delta$,
\begin{equation}
\begin{aligned}
\label{bound_evolvution_error}
&\underbrace{\sum_{k=1}^{K}
U^{-1}
\mathbb{E}_{\mathcal{P}}
\left[
\widehat{\beta}_{1}^{k,\widehat{\pi}^k}(\mathbf{S}_1)
\right]
-
U^{-1}
\mathbb{E}_{\mathcal{P}}
\left[
\vec{\beta}_{1}^{\widehat{\pi}^k}(\mathbf{S}_1)
\right]
}_{\text{Evolution error
}}
\\&\leq 
48
\cdot 
\underbrace{\frac{e^{\abs{\gamma } H}-1}{\abs{\gamma}H}
\cdot 
H
}_{\text{Risk and bonus
}}
\cdot
\left(
\underbrace{
{H{\sqrt{{KS^{{{^2}}}A}}}}
}_{\text{Hidden state error
}}
+
\underbrace{H{{\sqrt{{KS}{{{{O}}}}}}
}
}_{\text{Observation error
}}
+
\underbrace{H\sqrt{K}
}_{\text{Sample bias
}}
\ 
\right)
\cdot
\underbrace{{\sqrt{H}
\sqrt{\ln \left(\frac{KHSOA}{\delta}\right)}}
}_{\text{History-dependency of POMDP
}}
\end{aligned}
\end{equation}
\end{proof}

\subsection{Result and Discussion}
We summarize our previous analysis in Section \ref{section_control the evolution error} and \ref{section_Bound the prior error} 
into the following theorem, which characterizes the upper bound of the regret given by the algorithm of Beta Vector Value Iteration. 

\subsubsection{The Main Theorem}
\begin{theorem}{(Regret of Beta Vector Value Iteration)}
\label{theorem_regret_of_algo_1}
Given a POMDP model $\mathcal{P}$,
risk-sensitive parameter $\gamma \in \R \setminus \{0\}$ and the number of episodes $K\in \Z_{+}$, the regret after 
running algorithm \ref{alg:Beta Vector Value Iteration} 
can be controlled by the following upper bound with probability at least $1-4\delta$:
\begin{equation}
\begin{aligned}\label{full_regret_ours}
&\text{Regret}(K;\mathcal{{P}},{\gamma}) 
\\\leq &
48
\underbrace{{\frac{e^{\abs{\gamma} H}-1}{\abs{\gamma}}}
}_{\text{Risk and bonus
}}
\cdot
\left(
\underbrace{\sqrt{KS}
}_{\text{Prior error
}}
+
\underbrace{
{{
H\sqrt{{KS^{{{^2}}}A}}}}
}_{\text{Transition error
}}
+
\underbrace{{{
H\sqrt{{KS}{{{{O}}}}}}
}
}_{\text{Emission error
}}
+
\underbrace{
H\sqrt{K}
}_{\text{Sample bias
}}
\ 
\right)
\cdot
\underbrace{{{{
\sqrt{H
\cdot
\ln \left(\frac{KHSOA}{\delta}\right)
}}}}
}_{
\text{
History-dependency
of POMDP
}
}
\\
\leq 
&
\mathcal{O}
\bigg(
\underbrace{\frac{e^{\abs{\gamma } H}-1}{\abs{\gamma}H}
}_{\text{Risk-awareness
}}
\cdot 
\underbrace{H^2\sqrt{KS^2AO}
}_{\text{Statistical error
}}
\cdot \ 
\underbrace{
\sqrt{H}\sqrt{\ln\frac{KHSOA}{\delta}}
}_{\text{History-dependency
}}
\bigg)
\end{aligned}
\end{equation}
\end{theorem}
In what follows, we will provide a technical analysis of the composition of our regret, which is of particular interest to the possible improvements to our algorithm. We will also compare our results with other classical bounds in the related fields of reinforcement learning.

\subsubsection{Composition of the Regret}
The factor $\sqrt{K}$ is brought by the pigeon-hole lemma when we sum up several terms across the episodes. Similarly, one factor of $H$ is brought by the summation over the horizon $\sum_{h=1}^H$. Another $H$ is brought by the bonus on the beta vector, as is demonstrated in remark \ref{remark_degeneracy_relation}. 
The other ${\sqrt{H}}$ is incurred by the {history-dependency} of the POMDP model, as is shown in Eq.~\eqref{bonus_bound_2}.

The factor $\sqrt{O}$ and one of the $\sqrt{S}$
is brought by the coverage of an Epsilon net when we try to 
bound the beta vector function using Lemma~\ref{lemma_concentration_Hoeffding_function}. 
The other $\sqrt{SA}$ and $\sqrt{S}$ come from the 
pigeon-hole lemma, when we sum the residue terms across $k\in [K]$. 

\subsubsection{Sample Complexity}
Based on Theorem \ref{theorem_regret_of_algo_1} we can use the online-to-PAC conversion argument (cf.Appendix~\ref{online to PAC}) to obtain the sample complexity of algorithm \ref{alg:Beta Vector Value Iteration}.

\begin{corollary}({Sample complexity of Beta Vector Value Iteration})\label{corollary_sample_complexity} 
For any$K \gtrsim \frac{1}{\epsilon^2 \delta^2}\left(\frac{e^{\abs{\gamma }H}-1}{\abs{\gamma} H}\right)^2 H^5 S^2 O A \cdot \ln \left(\frac{KHSOA}{\delta}\right)$, 
with probability at least $1-\delta$, the uniform mixture of the output policies of algorithm \ref{alg:BVVI_short} is $\epsilon$ optimal:
$$
\frac{1}{K}\sum_{k=1}^K V_1^\star - V_1^{\widehat{\pi}^k} \lt \epsilon
$$
\end{corollary}

\subsubsection{Comparison with Other Studies}
\label{section_compare_other_setting}
In this section, we will try to study what the bound of our regret will be in
the risk-neutral case and/or the completely observable scenario, so that we can test whether our algorithm is still provably efficient after 
the degeneration into simpler settings.

\paragraph{Comparison with risk-neutral HOMDP}
When we take the limit $\gamma \to 0$ to Eq.~\eqref{full_regret_ours}, 
our regret bound degenerates into
\begin{equation}
\begin{aligned}
\lim_{\gamma \to 0}
\text{Regret}(K;\mathcal{{{P}}},\gamma)
=
\tilde{O}(H\sqrt{KSAO}\cdot \sqrt{O^2+H^3S})
\end{aligned}
\end{equation}
which characterizes the performance of our algorithm in the risk-neutral setting. 

\cite{Lee2023hindsight} studied risk-neutral POMDP under hindsight observability, whose setting differs from ours only in the risk-sensitivity of the agent.
The regret of their algorithm ``HOP-B'' is provided in theorem C.1 of \cite{Lee2023hindsight}, which is
\begin{equation}
\begin{aligned}
\sum_{k \in[K]} v\left(\pi^{\star}\right)-v\left(\hat{\pi}_k\right) \lesssim & \underbrace{\sqrt{{H^5} K \log (2 / \delta)}}_{\text {Azuma-Hoeffding }}+\underbrace{\sqrt{O S {H^5} K \iota}}_{\text {Emission error }}+\underbrace{\sqrt{S A H^4 K \iota}+{H^4} {S^2 A}  \iota(1+\log (K))}_{\text {Transition error }} \\
& +\underbrace{{H^3} S \sqrt{O \iota}+H S A \sqrt{H^3 \iota}}_{\text {Residual pigeonhole error }},
\end{aligned}
\end{equation}
We see that our regret improves their result in the order of $H,S$ and $A$.

The sample complexity the BVVI algorithm also nearly reaches the information-theoretic lower bound for the tabular HOMDPs, which is provided in theorem 5.1 of \cite{Lee2023hindsight}:
$$
K=\Omega\left({SO}/\epsilon^2 \right)
$$

\paragraph{Comparison with risk-sensitive MDP}
Due to the significant difference between the formulation of POMDP and MDP, several adjustments to our algorithm are necessary to adapt to a fully observable environment. We will no longer approximate the emission process and drop the terms in the bonus that is relevant with the emission residue. The confidence level will also be increased for $O$-times. In the end our algorithm will degenerate into the RSVI2 algorithm proposed by \cite{Fei21improve}. After some revision in the proofs, our regret will drop the terms relevant with $O$, as well as the additional ${\sqrt{H}}$ brought by the history-dependency of POMDP. The regret will take the form of
\begin{equation}
\begin{aligned}
\label{risk_sensitive_MDP}
\text{Regret}
(K;{\boldsymbol{M}},\gamma)=
\frac{e^{\abs{\gamma}H}-1}{\abs{\gamma}H}
\cdot 
\tilde{O}(H^2\sqrt{KS^2A})
\end{aligned}
\end{equation}
which matches the result of \cite{Fei21improve}:
\begin{equation}
\begin{aligned}
\operatorname{Regret}(K) \lesssim \frac{e^{|\gamma| H}-1}{|\gamma| H} \sqrt{H^4 S^2 A K \log ^2(H S A K / \delta)} .
\end{aligned}
\end{equation}
Consequently, in the MDP case, our regret reaches the lower bound for risk-sensitive RL using the entropic risk\cite{Fei20} in terms of $K$ and $\gamma H$:
\begin{equation}
\begin{aligned}
\operatorname{Regret}(K) \gtrsim \frac{e^{\abs{\gamma} H/2}-1}{\abs{\gamma} H}
\cdot H^{3/2} \sqrt{K \ln  KH} 
\end{aligned}
\end{equation}

\section{More Related Work}
\label{section_relate}
\paragraph{Risk-sensitive RL}
The analytical properties of the risk measures adopted in risk-sensitive RL has been extensively studied in \cite{boda2006time,DistortionIntro10,Righi2018theory,hau2023dynamic}. Combined with statistical learning theory, previous works have developed provably efficient algorithms for risk-sensitive RL in both the tabular case \cite{Fei20} and the function-approximation setting \cite{Fei_risk_fa}. 
Entropic risk \cite{Fei21improve} and CVaR \cite{IterCVaR22} are among the most popular risk measures adopted in RL.


\paragraph{Intractability of general POMDP}
Studies \cite{Tsitsiklis1987,Jin2020} have shown that seeking the exact solution to the planning or learning problem of general POMDPs is intractable. For this reason, the planning process in algorithm \ref{alg:Beta Vector Value Iteration}, as well as other exact algorithms for POMDP are inefficient in terms of computation complexity. To obtain polynomial sample complexity, recent research follows two directions: they either assume the structure of the POMDP model could leak certain information about the hidden states, or suggest that the training process will offer more knowledge to the agent. We will introduce the two lines of work in what follows. 

\emph{Learning a POMDP with structural assumptions}
The first line of research considers sub-classes of POMDP with additional structural properties, such as \cite{Golowich2022-a}($\gamma$-observable POMDPs) and \cite{Liu2022}($\alpha$-weak-revealing POMDPs). In the tabular case, \cite{Jin2020} assumed that the emission matrix posses a full column rank. In the continuous setting, \cite{Yang2022POMDPfa} asks the emission kernel to have a left inverse. Once the assumptions fails to be satisfied, their regret bounds could become vacuous \cite{Liu2022}. It also remains unclear whether these assumptions are acceptable in the application scenarios. 

\emph{Learning a POMDP with hindsight observation}
\label{condition_hind_sight}
Another line of research concerning with partially observable RL does not pose structural assumptions on the POMDP model. They consider a friendlier training setting in which the agent could review the sample path of the hidden states at the end of each episode. The new formulation for the POMDP, also referred to as the ``Hindsight Observation Markov Decision Process(HOMDP)'', is proposed by \cite{Lee2023hindsight} and echoed by \cite{sinclair2023hindsight,shi2023theoretical,guo2023sample}. The concept of ``hindsight observation'' is reasonable both theoretically and empirically. Supported by at least six examples in \cite{Lee2023hindsight}, reinforcement learning in a partially observable environment  offers hindsight information in various application scenarios. 
Furthermore, they also showed that the lower bound of sample complexity of learning an HOMDP is polynomial for the sizes of the spaces. For these reasons, we follow the second direction of research. 

We remind the reader that there is a significant difference between our work and that of \cite{Lee2023hindsight}. We propose the analytical tool of beta vector, which is different from the risk-neutral counterpart alpha vector. We also deploy a change-of-measure technique. We also adopt different analysis of the statistical errors and improve their regret bound in terms of $H, S $ and $A$. 

\paragraph{Risk-sensitive planning with partial information}
Risk-aware decision-making in partially observable environments has been studied theoretically:  \cite{Whittle91} derived an approximate solution to the continuous-time partially-observed risk-sensitive optimal control problem. \cite{Elliott96} developed a risk-sensitive Viterbi algorithm for the hidden Markov models. \cite{Baras97} also studied similar control-theoretic problems for the finite-state machines.  \cite{Baras&Elliott94,Cavazos2005,baauerle2017partially} 
considered risk-sensitive planning of a POMDP with the entropy or the utility risk measures.
However, these studies did not consider the learning problem of POMDP, not to mention sample complexity.

Our framework is built upon the study of \cite{Baras&Elliott94}. However, we should notice that there fundamental differences between the two works. 
The study of \cite{Baras&Elliott94} posed strong assumptions on the POMDP model: they set the initial state as a Gaussian random variable and required the transition law of the states and observations to be i.i.d. Gaussian distributions.  We generalize their result to accommodate transition matrices beyond Gaussian distributions. The study of \cite{Baras&Elliott94} did not consider the learning problem as they assumed the transition matrices are fully known, neither have they carried out a regret analysis which thoroughly discusseed in this work. 
We also propose new concepts such as the beta vectors and the partially observable risk-sensitive $Q$ functions not considered by the work of \cite{Baras&Elliott94}. We also devise a novel bonus function.  

\newpage
\section{Supplementary Materials} 

\subsection{Technical Lemmas}
In this section we provide the technical lemmas adopted in this work. We will give the references for existing results and provide a proof for the lemmas developed in this work. 

\subsubsection{Results from Real and Functional Analysis}
\begin{theorem}(Lebesgue-Radon-Nikodyn theorem, theorem 6.10 in \cite{rudin1987real_analysis})
\label{rn_derivative}
Let $\mathscr{M}$ be a $\sigma-$ algebra of a set $X$. 
let $\lambda$ be a measure on $\mathscr{M}$ and $\mu$ be a positive $\sigma-$ finite measure on $\mathscr{M}$.
There is a unique pair of measures $\lambda_a$ and $\lambda_s$ on $\mathscr{M}$ such that $\lambda=\lambda_a+\lambda_s$, where $\lambda$ is absolutely continuous with respect to $\mu$ ($\lambda \mu$)
while $\lambda_s$ and $\lambda_a$ are concentrated on disjoint sets ($\lambda_s \perp \mu$).
Moreover, there is a unique $h\in L^{1}(\mu)$ such that $\lambda_a(E)=\int_{E}h d\mu$ for every $E\in \mathscr{M}$. 
{We call $h$ the Radon-Nikodyn derivative of the 
measure $\lambda_a$ with respect to $\mu$ and we may express the derivative as $h=\frac{d\lambda_a}{d\mu}$.}
\end{theorem}

\begin{theorem}(Hilbert-adjoint operator, theorem 3.9-2 and 3.10-2 in \cite{kreyszig1991functional_analysis})
\label{Adjoint operator}
Let $H_1$ and $H_2$ be two Hilbert spaces and $T:H_2 \to H_1$ be a bounded linear operator.
There exists a unique linear bounded operator $T^\star$ 
with the same norm of $T$ such that for all $\vec{x} \in H_1 $ and $\vec{y}\in H_2$, 
$ \langle \vec{x}, T^\star \vec{y} \rangle = \langle T\vec{x}, \vec{y} \rangle$.
{If $H_1$ and $H_2$ have finite dimensions so that $T$ could be represented by some matrix, then $T^\star$ will be represented by the complex conjugate transpose of that matrix}.
\end{theorem}

\begin{lemma}(Dirac function and the expectation)
    \label{Lemma_sum_dirac_expect_condition}
    \begin{equation*}
        \begin{aligned}
        \int_{\mathcal{X}}dx    
        \mathbb{E}
        [f(\mathbf{X,Y})\delta(\mathbf{X}-x)|\mathbf{Z}]
        =\mathbb{E}[f(\mathbf{X,Y})|\mathbf{Z}]
        \end{aligned}
    \end{equation*}
\end{lemma}
\begin{proof}
    \begin{equation}
\begin{aligned}
        LHS=&\int_{\mathcal{X}}dx 
        \int_{\mathcal{X}}d\zeta 
        \int_{\mathcal{Y}}d\eta
        f(\zeta,\eta)\delta(\zeta-x)
        p_{\mathbf{X,Y|Z}}(\zeta,\eta|z)
        =
        \int_{\mathcal{X}}d\zeta
        \int_{\mathcal{X}}dx\delta(x-\zeta)
        \int_{\mathcal{Y}}d\eta f(\zeta,\eta)
        p_{\mathbf{X,Y|Z}}(\zeta,\eta|z)
        \\=&
        \int_{\mathcal{X}}d\zeta
        \int_{\mathcal{Y}}d\eta
        f(\zeta,\eta)
        p_{\mathbf{X,Y|Z}}(\zeta,\eta|z)
        =RHS
\end{aligned}
\end{equation}
\end{proof}

\vspace{-5mm}
\begin{lemma}(Dirac function, inner product and the expectation)
    \label{Lemma_inner_product_dirac_expect}
    \begin{equation}
        \begin{aligned}
        \left\langle \ 
        \mathbb{E}[\delta(\mathbf{X}=\ \cdot \ )F(\mathbf{X,Y})]
        \ ,\  
        f(\cdot)
        \ \right\rangle
        =
        \mathbf{E}[f(\mathbf{X})\cdot F(\mathbf{X,Y}))]
        \end{aligned}
    \end{equation}
\end{lemma}

\begin{proof}
    \begin{equation*}
        \begin{aligned}
        LHS=&\int_{\mathcal{X}} dx f(x)\mathbf{E}[\delta(\mathbf{X}-x)F(\mathbf{X,Y})]
        p_{\mathbf{X,Y}}(\zeta,\eta)
        =
        \int_{\mathcal{X}}dx f(x)
        \int_{\mathcal{Y}}d\eta
        \int_{\mathcal{X}}d\zeta \delta(\zeta-x)\ F(\zeta,\eta)
        p_{\mathbf{X,Y}}(\zeta,\eta)
        \\=&
        \int_{\mathcal{X}} d\zeta
        \left[
        \int_{\mathcal{X}} dx f(x) \delta(x-\zeta)
        \int_{\mathcal{Y}} d\eta F(\zeta,\eta)
        \right]
        p_{\mathbf{X,Y}}(\zeta,\eta)
        =
        \int_{\mathcal{X}}d\zeta f(\zeta)
        \int_{\mathcal{Y}}d\eta F(\zeta,\eta)
        p_{\mathbf{X,Y}}(\zeta,\eta)
        \\=&
        \mathbf{E}[f(\mathbf{X})\cdot F(\mathbf{X,Y})]
        =RHS
        \end{aligned}
    \end{equation*}
\end{proof}

\subsubsection{Concentration Inequalities}

\begin{lemma}(Concentration in the $\ell_1$ norm, fact 4 of \cite{Lipschitz23}, adapted from 
\cite{Weissman2003Inequalities})
\label{lemma_concentration_in_one_norm}
Let P be a probability distribution over a finite discrete measurable space $(\mathcal{X},\Sigma)$. 
Let $\widehat{P}_n$ 
be the empirical distribution of P estimated from $n$ samples.
Then with probability at least $1-\delta$, 
\begin{equation}
\begin{aligned}
\norm{\widehat{P}_n-P}_1 \leq \sqrt{\frac{2\abs{\mathcal{X}}}{n}\ln \frac{1}{\delta}}
\end{aligned}
\end{equation}
\end{lemma}

\begin{fact}
\label{lemma_naive_concentration}(Naive upper bound)
For any bounded function $f(\cdot): \mathcal{X} \to [a,b]$ and two probability measures $\mathbb{P}^\prime, \mathbb{P}\in \Delta(\mathcal{X})$, the difference in the expectation can be controlled by the range of the function and the total variance distance between the probability measures.
\begin{equation}
\begin{aligned}
\abs{
\mathbb{E}_{X\sim \mathbb{P}^\prime} [f(X)] -\mathbb{E}_{X\sim \mathbb{P}} [f(X)]}
\leq \frac{(b-a)}{2}\cdot \norm{\mathbb{P}^\prime(\cdot)- \mathbb{P}(\cdot)}_{1}
\end{aligned}
\end{equation}
\end{fact}

\begin{proof}
By the fact that all probability measures normalize to $1$, 
\begin{equation*}
\begin{aligned}
LHS=&\abs{
\mathbb{E}_{X\sim \mathbb{P}^\prime}\left(f(X)-\frac{b-a}{2}\right)-
\mathbb{E}_{X\sim \mathbb{P}}\left(f(X)-\frac{b-a}{2}\right)
}
\\=&
\abs{\sum_{x\in \mathcal{X}}\left(f(X)-\frac{b-a}{2}\right)\cdot (\mathbb{P}^\prime(x)-\mathbb{P}(x))}
\leq 
\sup_{x\in \mathcal{X}}\abs{f(x)-\frac{b-a}{2}}\cdot 
\norm{\mathbb{P}^\prime(x)-\mathbb{P}(x))}_1=RHS
\end{aligned}
\end{equation*}
\end{proof}
\begin{remark}
The upper bound provided in this lemma is tight for deterministic variable $\boldsymbol{X}$, which will be particularly useful when we study how the regret behaves when the risk-sensitivity parameters tends to zero.
\end{remark}

\begin{lemma}(Hoeffding inequality for random variables, adapted from theorem 2.3 in \cite{2003concentration_book_Lugosi})
\label{Lemma_concentration_azuma_hoeffding}
Let
$\{\mathbf{Y}_t\}_{t=1}^{n}$ 
be a finite set of independent random variables. 
Suppose that there exists two constant real numbers 
$\underline{Y}\lt \overline{Y}$
such that 
$\underline{Y} 
\leq \mathbf{Y}_t \leq \overline{Y}$
holds 
almost surely for any $\mathbf{Y}_t$, 
then with probability at least $1-\delta$, 
\begin{equation}
\begin{aligned}
\abs{\frac{1}{n}\sum_{i=1}^n Y_i -\frac{1}{n}\sum_{i=1}^n \mathbb{E}[Y_i]}
\leq (\overline{Y}-\underline{Y})\sqrt{\frac{1}{2n}\ln \frac{2}{\delta}}
\end{aligned}
\end{equation}
\end{lemma}

\begin{lemma}(Azuma-Hoeffding inequality for martingale difference sequences, theorem 2.16 in \cite{concentration_martingales})
\label{Lemma_concentration_MDS}
Let $\{Y_{t}\}_{t=1}^{\infty}$
be a martingale difference sequence 
with respect to some other
stochastic process $\{Xt\}_{t=1}^{\infty}$. 
Suppose that 
there exists two constants $a\lt b$ such that 
$a\leq Y_t \leq b$
almost sure for any $t\in \Z_+$, then for any $n\in \Z_+$ the following relation 
holds with probability at least $1-\delta$:
\begin{equation}
\begin{aligned}
\sum_{t=1}^{n} Y_t \lt (b-a)\sqrt{\frac{n}{2}\ln \frac{1}{\delta}}
\end{aligned}
\end{equation}
\end{lemma}

\begin{lemma}(Hoeffding inequality for the function of random variables,extended from lemma 12 of \cite{bai2020provable})
\label{lemma_concentration_Hoeffding_function}
Let $\mathbf{X}^{\prime}$ be a random variable supported on $\mathcal{X}$ that 
follows an unknown distribution $\mathbb{P}$. Let $f(\cdot)$ be any bounded function that maps $\mathcal{X}$ to $[a,b]$.
We draw $N$ i.i.d. samples from $\mathbb{P}$ to construct the empirical 
distribution $\widehat{\mathbb{P}}:=\frac{1}{N}\sum_{i=1}^N \mathds{1}\{ \widehat{x}_i^{\prime}=x^\prime \}$.
Denote $\Theta$ as the set of all the parameters that may distinguish the samples. Then with probability at least $1-\delta$,
\begin{equation}
\begin{aligned}
\label{hoeffding_concentration_random_function}
\abs{
\mathbb{E}_{\mathbf{X}^\prime \sim \mathbb{\widehat{P}}(\cdot)}[f(\mathbf{X}^\prime)]
-
\mathbb{E}_{\mathbf{X}^\prime \sim \mathbb{{P}}(\cdot)}[f(\mathbf{X}^\prime)]
}
\leq 
(b-a)\cdot 
\min \left\{
\ 1, \ 
3\cdot \sqrt{\abs{\mathcal{X}}} \cdot 
\sqrt{\frac{1}{N}\ln \frac{\abs{\Theta}}{\delta}}
\right\}
\end{aligned}
\end{equation}
\end{lemma}

\begin{remark}
The factor of $\sqrt{\abs{\mathcal{X}}}$ comes from the epsilon coverage of the range of $f(\cdot)$, since $\abs{\mathcal{X}}=\ln \frac{1}{\epsilon}^{\abs{\mathcal{X}}}$.
\end{remark}

\begin{lemma}(The pigeon-hole lemma, extended from \cite{azar2017minimax})
\label{lemma_pigeon-hole}
Fix constant $h$. Suppose that $\{\widehat{z}_{h}^t\}_{t=1}^K$ are i.i.d. samples drawn from a distribution $\mathbb{P}$ over the finite set $\mathcal{Z}$. For any $k=0,1,\cdots,K$, let $N_h^{k+1}(\cdot): \mathcal{Z}\to [K]$ be defined as the counter function $N_h^{k+1}(z):= \sum_{t=1}^{k} \mathds{1}\{\widehat{z}_h^t=z\}$ that records number of occurrences of $z$ within the first k samples. Let $f(\cdot): \mathcal{Z} \to \R$ be any function that receives integer input. The following relations always hold:
\begin{equation}\begin{aligned}
&(1)\ \sum_{z\in \mathcal{Z}} N_h^{k+1}(z) = k
\quad (2)\ \sum_{k=1}^K f(N_h^{k+1}(\widehat{z}_h^k)) = \sum_{z\in\mathcal{Z}}\sum_{i=1}^{N_h^{K+1}(z)}f(i) \quad 
(3)\ \sum_{k=1}^K \frac{1}{\sqrt{\max \{1,N_h^{k+1}(\widehat{z}_h^k)\}}} \lt
2\sqrt{K \cdot \abs{\mathcal{Z}}}
\end{aligned}\end{equation}
\end{lemma}
\begin{remark}
    The third equation in Lemma \ref{lemma_pigeon-hole} shows that the pigeon-hole upper bound depends on the size of the space $\mathbf{Z}$.
    For MDPs, $\mathcal{Z}$ will be replaced by $\mathcal{S\times A}$, which is polynomial in the relevant parameters. 
    However in a POMDP, since decision-making depends on the entire history, $\mathcal{Z}$ will be replaced with $\mathscr{O}^h\mathscr{A}^{h-1}$, which 
    causes the regret to be at least of order $O^H A^H$.
\end{remark}

\begin{lemma}(Linearization of Utility Function, Fact 1(a) in Appendix A of \cite{Fei20})
\label{lemma_linearize}
\begin{equation}
\begin{aligned}
\label{lipschitz}
&\text{When }\gamma \gt 0, \quad \text{for all } 1 \lt y \lt x \lt e^{\gamma H}, \quad \text{we have }
0 \lt \frac{1}{\gamma}\ln x - \frac{1}{\gamma}\ln y \lt \frac{1}{\gamma} (x-y) \\
&\text{When } \gamma \lt 0, \quad \text{for all }e^{\gamma H} \lt x \lt y \lt 1,\quad \text{we have }
0 \lt \frac{1}{\gamma}\ln x - \frac{1}{\gamma}\ln y \lt \frac{e^{\abs{\gamma }H}}{\abs{\gamma}} (y-x)
\end{aligned}
\end{equation}
Lemma \ref{lemma_linearize} implies that the differences between the entropic risk measures can be bounded by linear functions of the differences between their variables.
\end{lemma}

\subsubsection{Value functions using utility risk}
\label{section_Generalization to other risk-measures}
Exponential risks belong to a special class of risk criteria, named the utility function, 
\footnote{Readers may refer to \cite{VonNeumann1947UtilityTheory,MDPFinbook11,baauerle2017partially} 
for details.} 
which fits into the definition of 
{both static and dynamic risk measures.}
For simplicity, we demonstrate this property
in the MDP case when {actions are deterministic}
and {the initial state $s_1$ is fixed}.
We use $U\circ$ 
to represent a utility function. 
For entropic risk, 
$U=\frac{1}{\gamma} e^{\gamma (\cdot)}$. 

\text{(Static risk-measure narration)}
Assume that 
the initial state is fixed, 
\begin{equation*}
\begin{aligned}
J(\pi;M)=&U^{-1}\circ \mathbb{E}_{M}^\pi U
\left(
\sum_{h=1}^H\ r_h(\S_h,\A_h)
\right)
\quad
\\
V_{H+1,\text{static}}^\pi=&0
\quad 
V_{h,\text{static}}^\pi(s_h)=
U^{-1}\mathbb{E}_{M}^\pi U
\left[ 
\sum_{t=h}^H\ r_t(\S_t,\A_t)
\| s_{h}\right]
\end{aligned}
\end{equation*}
\text{(Dynamic risk-measure narration)}
Assume that 
the initial state is fixed, 
\begin{equation*}
\begin{aligned}
V_{H+1}^\pi=&0\\
Q_{h}^\pi(s_h,a_h)=&
r_h(s_h,a_h)+
\left(U^{-1}\circ\mathbb{E}_{s_{h+1}\sim P_{h+1}(\cdot|s_h,a_h)}
U\right)
\left(V_{h+1}^\pi(s_{h+1})\right)\\
V_h^\pi(s_h)=&\mathbb{E}_{a_h\sim \pi_h(\cdot|s_h)}Q_h^\pi(s_h,a_h)\\
J(\pi;M)=&V_1^\pi(s_1)
\end{aligned}
\end{equation*}

Next, we show that 
static and dynamic narrations are equivalent when the policy is deterministic and the initial state is fixed.
Precisely speaking, 
$\forall h\in[H+1], s_h\in \sS:$
\begin{equation*}
\begin{aligned}
V_h^\pi(s_h)=V_{h,\text{static}}^\pi(s_h)
\end{aligned}
\end{equation*}
\begin{proof}We prove by an induction on h.
First, the statement holds obviously at h=H+1.
Then 
\begin{equation*}
\begin{aligned}
&V_h^\pi(s_h)
\\=&
\mathbb{E}_{a_h}^\pi
\left[Q_h^\pi(s_h,a_h)\big|s_h
\right]
=
\mathbb{E}_{a_h}^\pi
\left[
U^{-1}
\mathbb{E}_{s_{h+1}}
\left[
U
\left[
r_h(s_h,a_h)+V_{h+1}^\pi(s_{h+1})
\right]
\|s_h,a_h\right]
s_h
\right]
\\=&
\mathbb{E}_{a_h}^\pi
\left[
U^{-1}
\mathbb{E}_{s_{h+1}}
\left[
U
\left[
r_h(s_h,a_h)+
U^{-1}
\mathbb{E}_{M}^\pi
\left[
U
\sum_{t=h+1}^H\ r_t(\S_t,\A_t)
\mid s_{h+1}
\right]
\right]
\|s_h,a_h\right]
s_h
\right]
\text{//Induction hypothesis}
\\=&
\mathbb{E}_{a_h}^\pi
\left[
U^{-1}
\mathbb{E}_{s_{h+1}}
\left[
{U}
\left[
{U^{-1}}
\mathbb{E}_{M}^\pi
\left[
U
\sum_{t=h}^H\ r_t(\S_t,\A_t)
\| 
s_{h+1},s_h,a_h
\right]
\right]
\| s_h,a_h
\right]
s_h
\right]
\text{//Markov property}
\\=&
\mathbb{E}_{a_h}^\pi
\left[
{U^{-1}}
\mathbb{E}_{s_{h+1}}
\left[ 
\mathbb{E}_{M}^\pi
\left[
{U}
\sum_{t=h}^H\ r_t(\S_t,\A_t)
\mid 
s_{h+1},s_h,a_h
\right]
\|s_h,a_h\right]
s_h
\right]
\\=&
\mathbb{E}_{a_h}^\pi
\left[
U^{-1}
\mathbb{E}_{M}^\pi
\left[
U
\sum_{t=h}^H r_t(\S_t,\A_t)
\mid s_h,a_h
\right]
s_h
\right]
\\
=&
U^{-1}
\mathbb{E}_{M}^\pi
\left[
U
\sum_{t=h}^H r_t(\S_t,\pi_t(\S_t))
\big| s_h\right]
=
V_{h,\text{static}}^\pi(s_h)
\quad \text{//Deterministic policy}
\end{aligned}
\end{equation*}
\end{proof}
We stress that the coincidental equivalence 
between the two narrations arises from 
the fact that
$
r_h(s_h,a_h)
=U^{-1} \mathbb{E}_{M}^{\pi}[U \circ r_h(s_h,a_h)|s_h,a_h]
$ and $\mathbb{E}_{M}^{\pi}[U\sum_{t\geq h+1}r_t(\mathbf{S}_t,\mathbf{A}_t)|s_{h+1}]
=
\mathbb{E}_{M}^{\pi}[U\sum_{t\geq h+1}r_t(\mathbf{S}_t,\mathbf{A}_t)|s_{h+1}s_h,a_h]$.

\subsubsection{The Auto-regressive Equation}

\begin{lemma}\label{lemma_backward_linear_Eq.}
    The solution to the initial value problem 
    of the equations
    $    x_N=0
        \quad
        x_n= A_n x_{n+1} + C_n    $
    is
    $x_1 = \sum_{\tau=1}^{N-1}A_{1:\tau-1}C_{\tau}$
\end{lemma}
We can obtain this result by an induction argument.

\begin{remark}\label{remark_inequality_version}
    If we restrict $\x_t, A_t\in\R_{\geq 0}$, similarly we can prove that
    that if $
            x_N=0,\ 
            x_n\leq  A_n x_{n+1} + C_n$
    we have $x_1 \leq \sum_{\tau=1}^{N-1}A_{1:\tau-1}C_{\tau}$
\end{remark}

\subsubsection{Online-to-PAC Conversion}\label{online to PAC}
The relationship between the regret of an online learning algorithm and its sample complexity is studied by \cite{Cesa-Bianchi2004} and \cite{jin2018qlearning}. 
In section 3.1 of \cite{jin2018qlearning} the authors used Markov's inequality to show that if we choose the output policies $\{\widehat{\pi}^k\}$ of an online learning algorithm uniformly at random, 
then to ensure these policies are provably approximately correct, i.e.
$$
\mathbb{P}\left(\sum_{k=1}^K V^\star - V^{\widehat{\pi}^k} \leq \epsilon \right) \geq 1-\delta 
$$
one only needs to ensure that the number of episodes $K$ 
will make the average expected regret lower than $\epsilon \delta$
$$
\overline{\operatorname{Regret}}(K):=\frac{1}{K}\sum_{k=1}^K V^\star -V^{\widehat{\pi}^k} \leq \epsilon \delta
$$
This technique is frequently used in reinforcement learning, such as in the derivation of corollary 5 in \cite{Liu2022} 
and theorem 6.3 of \cite{Lee2023hindsight} and we have invoked this relation in the derivation of Corollary \ref{corollary_sample_complexity}. 
For an elementary introduction to the conversion argument please refer to \cite{tirinzoni2022PAC-conversion}.



\end{document}